\documentclass[final]{siamart0516}

\usepackage{lipsum}
\usepackage{amsfonts}
\usepackage{graphicx}
\usepackage{epstopdf}
\usepackage{comment}

\ifpdf
  \DeclareGraphicsExtensions{.eps,.pdf,.png,.jpg}
\else
  \DeclareGraphicsExtensions{.eps}
\fi

\newcommand{\TheTitle}{A Tale of Two Bases: Local-Nonlocal Regularization on Image Patches with Convolution Framelets} 
\newcommand{\TheRunningTitle}{Local-Nonlocal Regularization with Convolution Framelets} 
\newcommand{\TheAuthors}{R. Yin, T. Gao, Y. M. Lu, and I. Daubechies}

\headers{\TheRunningTitle}{\TheAuthors}

\title{{\TheTitle}\thanks{Submitted to the editors \today.
}}

\author{
  Rujie Yin\thanks{Department of Mathematics, Duke University, Durham, NC 27708 (\email{rujie.yin@duke.edu}.)}
  \and
  Tingran Gao\thanks{Department of Mathematics, Duke University, Durham, NC 27708 (\email{trgao10@math.duke.edu}.)}
  \and
  Yue M. Lu\thanks{John A. Paulson School of Engineering and Applied Sciences, Harvard University, Cambridge, MA 02138 (\email{yuelu@seas.harvard.edu}.)}
  \and
  Ingrid Daubechies\thanks{Department of Mathematics, Duke University, Durham, NC 27708 (\email{ingrid@math.duke.edu}.)}
}

\usepackage{amsopn}
\usepackage{amssymb,amsmath,turnstile,txfonts,amsfonts}
\usepackage{bbm}
\usepackage{placeins}
\usepackage{booktabs}
\usepackage{enumerate}
\usepackage{caption}
\usepackage{mathrsfs}
\usepackage{multirow}
\usepackage{algorithm}
\usepackage{algpseudocode}

\DeclareMathOperator{\rank}{rank}
\DeclareMathOperator{\range}{range}
\DeclareMathOperator*{\argmin}{arg\,min}
\DeclareMathOperator*{\argmax}{arg\,max}
\newcommand{\embed}{\mathcal{E}}
\newcommand{\tx}{\tilde{x}}
\newcommand{\tX}{\widetilde{X}}
\newcommand{\tv}{\tilde{v}}
\newcommand{\tV}{\widetilde{V}}

\newcommand{\vone}{\mathbbm{1}}
\renewcommand{\t}[1]{\widetilde{#1}}
\newcommand{\new}[1]{#1}
\newsiamthm{claim}{Claim}
\newsiamremark{remark}{Remark}
\newcommand{\newsiamnonumberremark}[2]{
  \theoremstyle{plain}
  \theoremheaderfont{\normalfont\itshape}
  \theorembodyfont{\normalfont}
  \theoremseparator{.}
  \theoremsymbol{}
  \newtheorem*{#1}[theorem]{#2}
}
\newsiamnonumberremark{example}{Example}
\newsiamremark{hypothesis}{Hypothesis}
\crefname{hypothesis}{Hypothesis}{Hypotheses}
\graphicspath{{./fig/}}

\ifpdf
\hypersetup{
  pdftitle={\TheTitle},
  pdfauthor={\TheAuthors}
}
\fi

\begin{document}

\maketitle

\begin{abstract}
We propose an image representation scheme combining the local and nonlocal characterization of patches in an image. Our representation scheme can be shown to be equivalent to a tight frame constructed from convolving local bases (e.g. wavelet frames, discrete cosine transforms, etc.) with nonlocal bases (e.g. spectral basis induced by nonlinear dimension reduction on patches), and we call the resulting frame elements {\it convolution framelets}. Insight gained from analyzing the proposed representation leads to a novel interpretation of a recent high-performance patch-based image inpainting algorithm using Point Integral Method (PIM) and Low Dimension Manifold Model (LDMM) [Osher, Shi and Zhu, 2016]. In particular, we show that LDMM is a weighted $\ell_2$-regularization on the coefficients obtained by decomposing images into linear combinations of convolution framelets; based on this understanding, we extend the original LDMM to a reweighted version that yields further improved inpainting results. In addition, we establish the energy concentration property of convolution framelet coefficients for the setting where the local basis is constructed from a given nonlocal basis via a linear reconstruction framework; a generalization of this framework to unions of local embeddings can provide a natural setting for interpreting BM3D, one of the state-of-the-art image denoising algorithms.
\end{abstract}



\begin{table}[htbp]
\caption{Notations used throughout this paper}
\begin{center}
\begin{tabular}{c c p{10cm} }
\toprule
$I_k$ & & identity matrix in $R^{k\times k}$\\
$\vone_k$ & & vector with all one entries in $R^k$\\
$\left\| \cdot \right\|_{\mathrm{F}}$ & & matrix Frobenius norm\\
$\ell$ & & dimension of \emph{ambient space}, i.e. number of pixels in a patch\\
$p$ & & dimension of \emph{embedding space}, i.e. number of coordinate functions in the nonlocal embedding\\
$N$ & & number of data points (patches)\\
$X$ & & data matrix in $\mathbb{R}^{N\times \ell}$\\
$\t{X}$ & & embedded data matrix in $\mathbb{R}^{N\times p}$, ``orthogonalized" s.t. $\tX = \Phi_\embed\, C_\embed$\\
$X^i\, (\textrm{resp. }\t{X}^i)$ & & the $i$th coordinate in ambient space (resp. embedding space)\\
$\Phi_\embed$ & & normalized graph bases in $\mathbb{R}^{N\times p}$ from $\embed$ and $\Phi_\embed^{\top}\,\Phi_\embed = I_p$ \\
$\Phi$ & & full orthonormal nonlocal bases extended from $\Phi_\embed$\\
$C_\embed$ & & diagonal matrix with entries $\Vert \tX^i\Vert$\\
$x_i$ & & the $i$th row of $X$, i.e. the $i$th data point \\
$\tx_i$ & & embedding of $x_i$\\
$\embed$ &  & embedding function from $\mathbb{R}^{\ell}$ to $\mathbb{R}^p$\\
$\embed_x$ &  & affine approximation of $\embed$ at point $x$\\
$V_0$ & & patch bases in $\mathbb{R}^{\ell\times p}$,\, s.t. $V_0^{\top}V_0 = I_p$.\\
$V$ & & full orthonormal local bases in $O(\ell)$\\
$C$ & & coefficient matrix\\
$f$ & & 1-D or 2-D signal, e.g. an image\\
$F$ & & patch matrix in $\mathbb{R}^{N\times \ell}$ generated from $f$, a special type of data matrix\\
$F_i$ & & $i$th patch\\
$F^i$ & & $i$th coordinate in patch space, e.g. $i$th pixel in all patches\\
$\t{F}$ & & embedded patch matrix \\
$\Psi $ & & bases in $\mathbb{R}^{N\times \ell}$ from $\phi_i\,v_j^{\top}$\\
$\psi $ & & bases in $\mathbb{R}^N$ from $\phi_i\,*v_j(-\cdot)$\\
$W$ & & affinity matrix of diffusion graph with Gaussian kernel\\
$D$ & & degree matrix from $W$\\
$L$ & & normalized graph diffusion Laplacian\\
$R_L$ & & graph operator in LDMM\\
\bottomrule
\end{tabular}
\end{center}
\label{tab:notation}
\end{table}

\section{Introduction}
\label{sec:introduction}
In the past decades, patch-based techniques such as Non-Local Means (NLM) and Block-Matching with 3-D Collaborative Filtering (BM3D) have been successfully applied to image denoising and other image processing tasks \cite{BCM2005CVPR,BCM2005,BM3D2007,ChatterjeeMilanfar2012,talebi2014global,kheradmand2014general}. These methods can be viewed as instances of graph-based adaptive filtering, with similarity between pixels determined not solely by their pixel values or spatial adjacency, but also by the (weighted) $\ell^{2}$-distance between their neighborhoods, or \emph{patches} containing them. The effectiveness of patch-based algorithms can be understood from several different angles. On the one hand, patches from an image often enjoy sparse representations with respect to certain redundant families of vectors, or unions of bases, which motivated several dictionary- and sparsity-based approaches \cite{EladAharon2006,ChatterjeeMilanfar2009,NLSM2009}; on the other hand, the nonlocal characteristics of patch-based methods can be used to build highly data-adaptive representations, accounting for nonlinear and self-similar structures in the space of image patches \cite{IzbickiLee2015,LeeIzbicki2016SpectralSeries}. Combined with adaptive thresholding, these constructions have connections to classical wavelet-based and total variation algorithms \cite{SMC2008,Peyre2008}. Additionally, the patch representation of signals has specific structures that can be exploited in regularization; for example, the inpainting algorithm ALOHA \cite{JY2015} utilized the low-rank block Hankel structure of certain matrix representation of image patches.

Among the many theoretical frameworks built to understand these patch-based algorithms, manifold models have recently drawn increased attention and have provided valuable insights in the design of novel image processing algorithms. Along with the development of manifold learning algorithms and topological data analysis, it is hypothesized that high-contrast patches are likely to concentrate in clusters and along low-dimensional non-linear manifolds; this phenomenon is very clear for cartoon images, see e.g. \cite{Peyre2008,Peyre2011Review}. This intuition was made precise in \cite{LPM2003} and followed-up by more specific Klein bottle models \cite{CIdSZ2008,PereaCarlsson2014} on both cartoon and texture images. Adopting a point of view from diffusion geometry, \cite{SSN2009} interprets the non-local mean filter as a diffusion process on the ``patch manifold,'' relating denoising iterations to the spectral properties of the infinitesimal generator of that diffusion process; similar diffusion-geometric intuitions can also be found in \cite{SMC2008,QiWu2015} which combined patch-based methods with manifold learning algorithms.

Recently, a new method called Low-Dimensional Manifold Model (LDMM) was proposed in \cite{LDMM}, with strong results. LDMM is a direct regularization on the dimension of the patch manifold in a variational argument for patch-based image inpainting and denoising.  The novelty of \cite{LDMM} includes 1) an identity relating the dimension of a manifold with $L^2$-integrals of ambient coordinate functions; and 2) a new graph operator (which we study below) on the nonlocal patch graph obtained via the Point Integral Method (PIM) \cite{LiShiSun2014PIM,ShiSun2014PIMNeumann,ShiSun2013PIMDirichelt}.
The current paper is motivated by our wish to better understand the embedding of image patches in general, and the LDMM construction in particular.

Typically, given an original signal $f\in\mathbb{R}^N$, patch-based methods start with explicitly building for $f$ a redundant representation consisting of \emph{patches} of $f$. The patches either start with or are centered at each pixel\footnote{Possibly with a \emph{stride} larger than $1$ in many applications. In this paper though, we assume that the stride is always equal to $1$ to demonstrate the key ideas. Periodic boundary condition is assumed throughout this paper.} in the domain of $f$, and are of constant length $\ell$, for $1 < \ell < N$. Reshaped into row vectors stacked vertically in the natural order, these patches constitute a Hankel\footnote{For 2-D images, the patch matrix is indeed block Hankel; see e.g. \cite{JY2015}.} matrix $F\in\mathbb{R}^{N\times \ell}$, which we refer to as a \emph{patch matrix} (see Fig~\ref{fig:patch-mat} in Section~\ref{sec:revisit-ldmm} below.) It is the patch matrix $F$, rather than the signal $f$ itself, that constitutes the object of main interest in nonlocal image processing and in particular LDMM; each single pixel of image $f$ is represented in $F$ exactly $\ell$ times, a redundancy that is often beneficially exploited in signal processing tasks. (As will be made clear in \cref{prop:frame}, representing $f\in\mathbb{R}^N$ as $F\in\mathbb{R}^{N\times \ell}$ incurs an ``$\ell$-fold redundancy'' in the sense of frame bounds.) For comparison, earlier image processing models based on total variation \cite{ROF1992} or nonlocal regularization \cite{GilboaOsher2007,GilboaOsher2008} focus on regularizing the signal $f$ directly, whereas more recent state-of-the-art image inpainting techniques such as LDMM and low-rank Hankel matrix completion \cite{JY2015} build upon variational frameworks for the patch matrix $F$, and do not convert $F$ back to $f$ until the optimization step terminates. To our knowledge, the mechanism of these regularization strategies on patch matrices has not been fully investigated.

From an approximation point of view, the patch matrix $F\in\mathbb{R}^{N\times \ell}$ has more flexibility than the original signal $f\in\mathbb{R}^N$ since one can search for efficient representation of the matrix $F$ either in its row space or column space. The idea of learning sparse and redundant representations for rows of $F$, or the patches of $f$, has been pursued in a sequence of works (see e.g. \cite{OlshausenField1997,EASH1999,KRB2000,KMRELS2003,LGBB2005,AEB2006,EladAharon2006} and the references therein); this amounts to learning a redundant dictionary $\mathcal{D}\in\mathbb{R}^{\ell\times m}$, $m\geq \ell$, such that $F=A\mathcal{D}^{\top}$ where the rows of $A\in \mathbb{R}^{N\times m}$ are sparse. In the meanwhile, each column of $F$ can be viewed as a ``coordinate function'' (adopting the geometric intuition in \cite{LDMM}) defined on the dataset of patches, and can thus be efficiently encoded using spectral bases adapted to this dataset: for example, let $\Phi:\mathbb{R}\rightarrow [0, +\infty)$ be a non-negative smooth \emph{kernel function} with exponential decay at infinity, and construct the following positive semi-definite kernel matrix for the dataset of patches of $f$:
\begin{equation*}
  \Phi_{\epsilon} \left( ij \right) = \Phi \left( \frac{\left\| F_i-F_j \right\|_2^2}{\epsilon} \right), \qquad 0\leq i,j \leq N-1
\end{equation*}
where $F_i, F_j$ are the $i$th and $j$th row of the patch matrix $F$, respectively, and $\epsilon>0$ is a bandwidth parameter representing our confidence in the similarity between patches of $f$ (e.g. how small $L^2$-distances should be to reflect the geometric similarity between patches; this is influenced, for example, by the noise level in image denoising tasks). By Mercer's Theorem, $\Phi_{\epsilon}$ admits an eigen-decomposition
\begin{equation*}
  \Phi_{\epsilon} = \sum_{k=1}^N \lambda_k \phi_k\phi_k^{\top}
\end{equation*}
where for each $1\leq k\leq N$ the column vector $\phi_k\in\mathbb{R}^N$ is the eigenvector associated with non-negative real eigenvalue $\lambda_k\in\mathbb{R}$.
These eigenvectors constitute a basis for $\mathbb{R}^N$, with respect to which each column of the patch matrix $F$ can be expanded as a linear combination.
Though such expansions are not sparse in general, they are highly data-adaptive and result in efficient approximations when the eigenvalues have fast decay; see \cite{LeeIzbicki2016SpectralSeries,ABK2015} for theoretical bounds of the approximation error, \cite{Peyre2008} for empirical evidence, and \cite{SMC2008} for applications in semi-supervised learning and image denoising. By construction, the sparse representation for the rows of $F$ relies heavily on the \emph{local} properties of the signal $f$, whereas the spectral expansion for the columns of $F$ captures more \emph{nonlocal} information in $f$. We remark here that many other orthonormal or overcomplete systems can be used to produce different representations for the row and column spaces of the patch matrix $F$: for instance, wavelets or discrete cosine transform can be used in place of a dictionary $\mathcal{D}$, while any linear/nonlinear embedding methods, dimension reduction algorithms (e.g. PCA \cite{Pearson1901PCA}, MDS \cite{MDS1952,MDS1962}, Autoencoder \cite{Autoencoder2006}, t-SNE \cite{tSNE2008}), or Reproducing Kernel Hilbert Space techniques \cite{ScholkopfSmola2002} can work just as well as the kernel $\Phi$; nevertheless, the different choices for the row (resp. column) space of $F$ primarily read off local (resp. nonlocal) information of $f$. These observations motivate us to seek new representations for the patch matrix $F$ that could reflect both local and nonlocal behavior of the signal $f$. This methodology is already implicit in BM3D \cite{BM3D2007}, one of the state-of-the-art image denoising algorithms (see Section~\ref{sec:extension-union-local-embed} for details); we point out in this paper that such a paradigm is much more universal for a wide range of patch-based image processing tasks, and propose a regularization scheme for a signal $f$ based on its coefficients with respect to \emph{convolution framelets} (to be defined in Section~\ref{sec:revisit-ldmm}), a type of signal-adaptive tight frames generated from the adaptive representation of the patch matrix $F$.

As a first attempt at understanding the theoretical guarantees of convolution framelets, we consider the problem of determining ``optimal'' local basis, in the sense of minimum linear reconstruction error, with respect to a fixed nonlocal basis (interpreted as embedding coordinate functions of the patches); convolution framelets constructed from such an ``optimal'' pair of local and nonlocal bases are guaranteed to have an ``energy compaction property'' that can be exploited to design regularization techniques in image processing. In particular, we show that when the nonlocal basis comes from Multi-Dimensional Scaling (MDS), right singular vectors\footnote{Since the singular value decomposition of a patch matrix is not known \emph{a priori} in image reconstruction tasks, the algorithms we propose in this paper are all of iterative nature, with the SVD basis updated in each iteration; similar strategies have previously been utilized in nonlocal image processing algorithms, see e.g. \cite{GilboaOsher2007,GilboaOsher2008}.} of the patch matrix $F$ constitute the corresponding optimal local basis. The linear reconstruction framework itself --- of which LDMM can be viewed as an instantiation --- is general and uses variational functionals associated with nonlinear embeddings, via a linearization. This insight allows us to generalize LDMM by reformulating the manifold dimension minimization in \cite{LDMM} as an equivalent weighted $\ell_2$-minimization on coefficients of such a convolution frame and by using more adaptive weights; for some types of images this proposed scheme leads to markedly improved results. Finally, we note that our framework is widely applicable and can be adapted to different settings, including BM3D~\cite{BM3D2007} (in which case the framework needs to be extended to describe unions of local embeddings, as is done in Section~\ref{sec:extension-union-local-embed} below).

The rest of the paper is organized as follows. In Section~\ref{sec:conv-framelets} we present convolution framelets as a data-adaptive redundant representation combining local and nonlocal bases for signal patch matrices. Section~\ref{sec:approximation-conv-framelets} motivates the energy compaction property of convolution framelets and establishes a guarantee for energy concentration through a linear reconstruction procedure
 related to (nonlinear) dimension reduction \cite{lorenzoIS16}. Section~\ref{sec:revisit-ldmm} interprets LDMM as an $\ell^2$-regularization on the energy concentration of convolution framelet coefficients. This novel interpretation and insights gained from the previous section lead to improvement of LDMM by incorporating more adaptive weights in the regularization. We compare LDMM with our proposed improvement in Section~\ref{sec:numerical} by numerical experiments. Section~\ref{sec:conclusion} summarizes and suggests future work.

\section{Convolution Framelets}
\label{sec:conv-framelets}

Consider a one-dimensional\footnote{The same idea can be easily generalized to signals of higher dimensions.} real-valued signal
\begin{equation*}
  f = \left(f[0],\cdots,f[N-1]\right)^{\top}\in\mathbb{R}^N
\end{equation*}
sampled at $N$ points. We fix the \emph{patch size} $\ell$ as an integer between $1$ and $N$, and assume periodic boundary condition for $f$. For any integer $m\in \left[ 0,\cdots,N-1 \right]$, we refer to the row vector $F_m=\left(f[m],\cdots,f[m+\ell-1]\right)\in\mathbb{R}^{\ell}$ as the \emph{patch} of $f$ at $m$ of length $\ell$. Construct the \emph{patch matrix} of $f$, denoted as $F\in\mathbb{R}^{N\times \ell}$, by vertically stacking the patches according to their order of appearance in the original signal:
\begin{equation}
\label{eq:patch_matrix}
  F = [F_0^{\top},\cdots,F_{N-1}^{\top}]^{\top}.
\end{equation}
See \cref{fig:patch-mat} for an illustration. It is clear that $F$ is a Hankel matrix, and thus $f$ can be reconstructed from $F$ by averaging the entries of $F$ ``along the anti-diagonals'', i.e.\footnote{Note that in \cref{eq:average_along_antidiagonal} the row indices start at 0, but the column indices start at 1 --- for instance, the entry at the upper left corner of $F$ is denoted as $F_{01}$.}
\begin{equation}
\label{eq:average_along_antidiagonal}
  f(n) = \frac{1}{\ell}\sum_{i=1}^{\ell} F_{n-i+1,\,i}\quad\quad n=0,1,\cdots,N-1.
\end{equation}

\begin{figure}
\centering
\includegraphics[width = .6\textwidth]{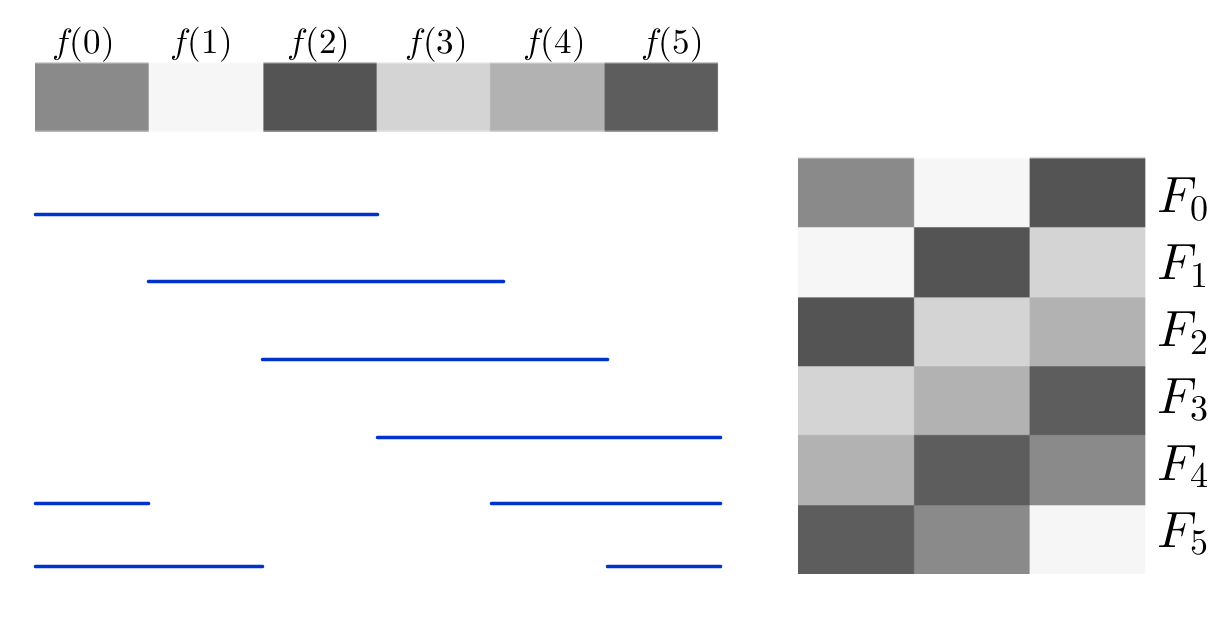}
\caption{Illustration of a patch matrix constructed from a 1-D signal (the blue lines indicate locations of the patches $F_0$ through $F_5$ in the original 1-D signal.)}
\label{fig:patch-mat}
\end{figure}

For simplicity, we introduce the following notations that are standard in signal processing:
\begin{itemize}
\item The \emph{(circular) convolution} of two vectors $v,\,w\in\mathbb{R}^N$ is defined as
\begin{equation*}
  (v * w)\,[n] = \sum_{m = 0}^{N-1}\,v[n-m]\,w[m],
\end{equation*}
where periodic boundary conditions are assumed (as is done throughout this paper);
\item For any $v\in\mathbb{R}^{N_1}$ and $w\in\mathbb{R}^{N_2}$ with $N_1,\,N_2\leq N$, define their \emph{convolution} in $\mathbb{R}^N$ as
\begin{equation*}
  v*w = v^0*w^0
\end{equation*}
where $v^0 = [v^{\top},\,\mathbf{0}^{\top}_{N-N_1}]^{\top},\,w^0 = [w^{\top},\,\mathbf{0}^{\top}_{N-N_2}]^{\top}$ denote the length-$N$ zero-padded versions of $v$ and $w$, respectively;
\item For any $v\in \mathbb{R}^{N_1}$ with $N_1\leq N$, define the \emph{flip} of $v$ as $v(-\cdot)\,[n] = v^0\,[-n]$.
\end{itemize}

Using these notations, the matrix-vector product of $F$ with any $v\in\mathbb{R}^{\ell}$ can be written in convolution form as
\begin{equation}\label{eq: patch-conv}
  Fv = f*v(-\cdot).
\end{equation}
Furthermore, it is straightforward to check for any $w\in\mathbb{R}^{\ell}$ and $v, s\in\mathbb{R}^N$ that
\begin{align}\label{eq: commute}
  s^{\top}\,(v*w) &= \sum_{m=0}^{N-1} s[m]\,\sum_{n=0}^{N-1} v[n]\,w[m-n] = \sum_{n=0}^{N-1}v[n]\,\sum_{m=0}^{N-1}s[m]\,w[m-n]\\
  &= \sum_{n=0}^{N-1}v[n]\sum_{m'=0}^{N-1}s[n+m']\,w[m']
  = v^{\top}\,(s*w(-\cdot)).\notag
\end{align}

Now let $\Phi\in O(N)$ and $V\in O(\ell)$ be orthogonal matrices of dimension $N\times N$ and $\ell\times \ell$, respectively; also denote the columns of $\Phi$, $V$ as $\phi_i$, $v_j$ correspondingly, where $1\leq i\leq N$, $1\leq j\leq \ell$. The outer products of the columns of $\Phi$ with the columns of $V$, denoted as
\begin{equation*}
  \left\{\Psi_{ij} = \phi_i\,v_j^{\top}\,\big|\, i = 1,\cdots, N,\, j = 1,\cdots,\ell\right\},
\end{equation*}
form an orthonormal basis for the space $\mathbb{R}^{N\times \ell}$ of all $N\times \ell$ matrices equipped with inner product $\left\langle A,B \right\rangle = tr \left( AB^{\top} \right)$. The patch matrix $F$ can thus be written in this orthonormal basis as
\begin{equation*}
  F = \sum_{i=1}^N\sum_{j=1}^{\ell} tr \left( F \Psi_{ij}^{\top} \right)\Psi_{ij} = \sum_{i=1}^N\sum_{j=1}^{\ell} C_{ij}\Psi_{ij} = \sum_{i=1}^N\sum_{j=1}^{\ell} C_{ij}\phi_iv_j^{\top},
\end{equation*}
where
\begin{equation*}
    C_{ij}:=tr \left( F \Psi_{ij}^{\top} \right)=tr \left( F v_j\phi_i^{\top} \right) = \phi_i^{\top}\,F\,v_j = \phi_i^{\top}\,(f*v_j(-\cdot)) = f^{\top}\,(\phi_i * v_j),
\end{equation*}
and the last two equalities are due to the identities \eqref{eq: patch-conv} and \eqref{eq: commute} given above.
In other words, we have the following linear decomposition for $F$:
\begin{equation}
  \label{eq:patch_matrix_decom}
  F = \sum_{i=1}^N\sum_{j=1}^{\ell} \left\langle f, \phi_i * v_j \right\rangle \phi_iv_j^{\top}.
\end{equation}
Combining \cref{eq:average_along_antidiagonal} and \cref{eq:patch_matrix_decom} leads to a decomposition of the original signal $f$ as
\begin{equation}
  \label{eq:original_signal_convframelet_decom}
  f = \frac{1}{\ell}\sum_{i=1}^N\sum_{j=1}^{\ell} \left\langle f, \phi_i * v_j \right\rangle \phi_i * v_j,
\end{equation}
where the convolution $\phi_i * v_j$ stems from averaging the entries of $\phi_iv_j^{\top}$ along the anti-diagonals [c.f. \cref{eq:average_along_antidiagonal}].
Define {\it convolution framelets}
\begin{equation}
  \label{eq:defn_convframelet}
  \psi_{ij} = \frac{1}{\sqrt{\ell}}\phi_i * v_j, \qquad i = 1,\cdots,N,\quad j = 1,\cdots,\ell,
\end{equation}
then \cref{eq:original_signal_convframelet_decom} indicates that $\{\psi_{ij}\mid 1\leq i\leq N, 1\leq j\leq \ell\}$ constitutes a tight frame for functions defined on $\mathbb{R}^N$. In fact, we have the following more general observation which can be derived directly from standard frame theory:
\begin{proposition}\label{prop:frame}
Let $V^L\in\mathbb{R}^{n\times n'},\,V^S\in\mathbb{R}^{m\times m'}$ be such that $V^L\,(V^L)^{\top} = I_n,\,V^S\,(V^S)^{\top} = I_m$ with $m \leq n$. Then $v_i^L *\, v_j^S, \, i = 1,\cdots,n',\, j = 1,\cdots,m'$ form a tight frame for $\mathbb{R}^n$ with frame constant $m$.
\end{proposition}
The proof of \cref{prop:frame} can be found in \cref{sec:appendix}.

\section{Approximation of functions with convolution framelets}
\label{sec:approximation-conv-framelets}

The construction in Section~\ref{sec:conv-framelets} may seem unintuitive at a first glance. Our motivation for introducing two different bases, $\Phi\in O \left( N \right)$ and $V\in O \left( \ell \right)$, was simply to take advantage of the representability of patch matrices jointly in its row and column spaces.

\subsection{Local and nonlocal approximations of a signal}
\label{sec:local-nonlocal-approximation}

 The columns of $V$ form an orthonormal basis for $\mathbb{R}^{\ell}$, with respect to which the rows of $F$, or equivalently the length-$\ell$ patches of $f$, can be expanded; the role of $V$ is thus similar to transforms on a localized time window, such as the \emph{Short-Time Fourier Transform} (STFT) or \emph{Windowed Wigner Distribution Function} (WWDF). For this reason, we refer to the strategy of approximating the rows of $F$ using the columns of $V$ as \emph{local approximation}, and call $V$ a \emph{local basis} in the construction of convolution framelets. The local basis $V$ can be chosen as either fixed functions, e.g. Fourier or wavelet basis, or data-dependent functions, such as the right singular vectors of $F$. See \cref{fig:patch-svd} for an example.

\begin{figure}[h]
\includegraphics[width = \textwidth]{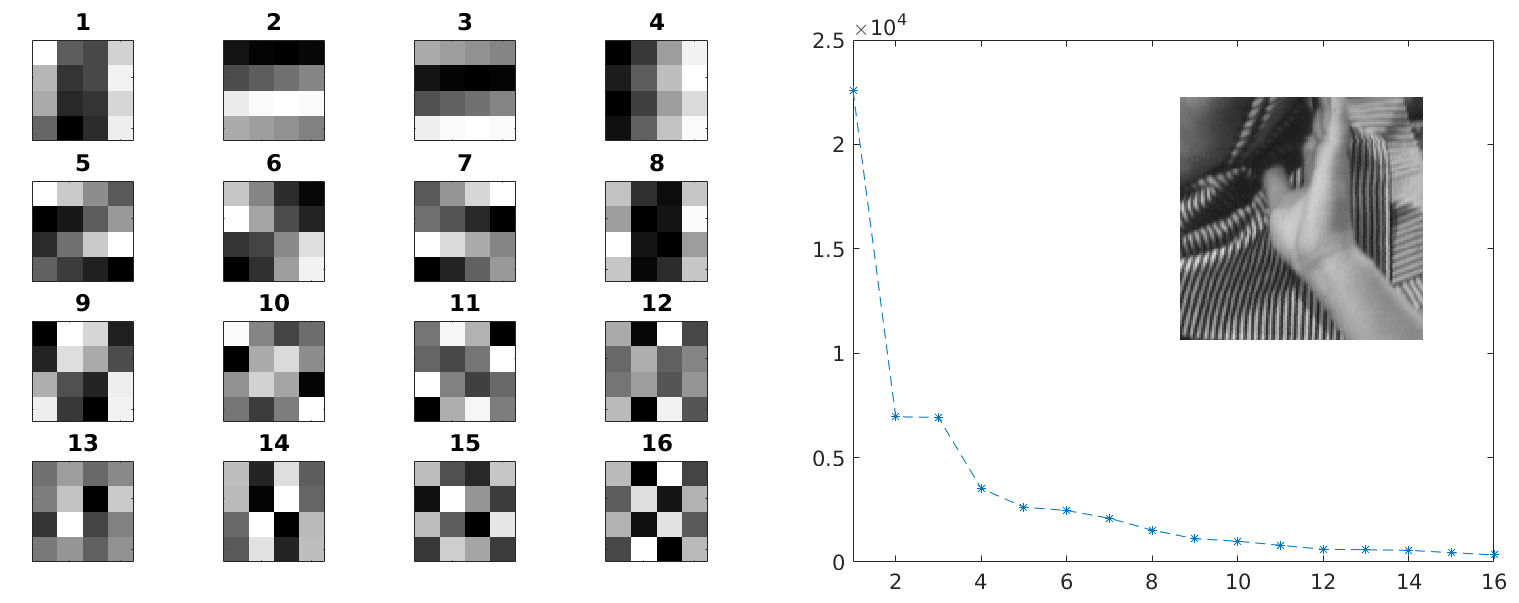}
\caption{\emph{Left:} Right singular vectors (ordered in decreasing singular values) of the patch matrix of a cropped {\sc barbara} image of size $128\times 128$, with patch size $4\times 4$; \emph{Right:} The cropped {\sc barbara} image and the  singular values corresponding to the right singular vectors shown on the left. Notice the fast decay of the singular values.}
\label{fig:patch-svd}
\end{figure}

The columns of $\Phi$, on the other hand, are treated as a basis for the columns of $F$. When the patch stride is set to $1$, each column $F$ is just a shifted copy of the original signal $f$ (see \cref{fig:patch-mat}); more generally (including arbitrary patch strides), columns of $F$ can be seen as functions defined on the set of patches $\mathscr{F}=\left\{ F_0,\cdots,F_{N-1} \right\}$. When $\mathscr{F}$ is viewed as a discrete point cloud in $\mathbb{R}^{\ell}$, efficient representations of functions on $\mathscr{F}$ depend more on the Euclidean proximity between patches as points in $\mathbb{R}^{\ell}$, rather than spatial adjacency in the original signal domain, as detailed in previous work on \emph{spectral basis} \cite{Peyre2008,LeeIzbicki2016SpectralSeries,IzbickiLee2015}. Therefore, it is natural to refer to the paradigm of approximating the columns of $F$ using $\Phi$ as \emph{nonlocal approximation}, and call $\Phi$ a \emph{nonlocal basis} in the construction of convolution framelets. 

Viewing the patch matrix $F$ as a collection $\mathscr{F}\subset\mathbb{R}^{\ell}$ brings in a large class of nonlinear approximation techniques from \emph{dimension reduction}, a field of statistics and data science dedicated to efficient data representations. Given a data matrix $X=[x_1,\cdots,x_N]^{\top}\in\mathbb{R}^{N\times \ell}$ consisting of $N$ data points in an ambient space $\mathbb{R}^{\ell}$ (we adopt the convention that $x_i$'s are column vectors and the $i$th row of $X$ is $x_i^{\top}$), dimension reduction algorithms map the full data matrix $X$ to $\t{X}=[\tx_1,\cdots,\tx_N]^{\top}\in\mathbb{R}^{N\times p}$, where each row $\tx_i\in\mathbb{R}^p$ ($p\leq \ell$) is the image of $x_i$. The dissimilarity between two original data points is assumed to be given by a metric (distance) function $d(\cdot,\cdot)$ on the ambient space $\mathbb{R}^{\ell}$, in many applications different from the canonical Euclidean distance; one hopes that the embedding is ``almost isometric'' between metric spaces $\left( \mathbb{R}^{\ell},d \right)$ and $\mathbb{R}^p$ equipped with the standard Euclidean distance.
More precisely, let $\embed = (\embed_1,\cdots,\embed_p):\mathbb{R}^{\ell}\rightarrow\mathbb{R}^p$ be the embedding given by $p$ coordinate functions, and denote $\embed(x) 
= [\tx^1,\cdots,\tx^p]^{\top}\in\mathbb{R}^p$ for any $x\in\mathbb{R}^{\ell}$. The embedding $\embed$ is said to be {\it near isometric} if in an appropriate sense
\begin{itemize}
\item[{\it (P1).}]
$d(x, x') \approx \Vert \embed(x) - \embed(x') \Vert_{\ell_2}$,\quad $\forall\,x,\,x'\in\mathbb{R}^{\ell}$.
\end{itemize}
Without loss of generality, we can assume that the coordinate functions of the embedding $\embed$ are orthogonal on the data set $\left\{ x_1,\cdots,x_N \right\}$, i.e.
\begin{itemize}
\item[{\it (P2).}] \new{$(\t{X}^s)^{\top}\t{X}^t = 0,\quad \forall 1\leq s\neq t\leq p$},
\end{itemize}
where $\t{X}^i$ is the $i$th column of $\tX$ (and corresponds to the $i$th coordinate in the embedding space); for general $\tX$ with coordinate functions non-orthogonal on the data set, we define its {\it orthogonal normalization} by $\tX^O = \tX\,O_{\tX}$, where $O_{\tX}$ comes from the \emph{Singular Value Decomposition} (SVD) of $\tX = U_{\tX}\Sigma_{\tX} O_{\tX}^{\top}$. Note that classical linear and nonlinear dimension reduction techniques, such as \emph{Principal Component Analysis} (PCA), \emph{Multi-Dimensional Scaling} (MDS), \emph{Laplacian Eigenmaps} \cite{BelkinNiyogi2003}, and \emph{Diffusion Maps} \cite{CoifmanLafon2006}, all produce embedding coordinate functions satisfying {\it (P1)} and {\it (P2)}.

A standard approach in manifold learning and spectral graph theory for building basis functions on $\mathscr{F}$ is through the eigen-decomposition of \emph{graph Laplacians} for a weighted graph constructed from $\mathscr{F}$. For instance, in \emph{diffusion geometry} \cite{CoifmanLafon2006,CoifmanLafonLMNWZ2005PNAS1,CoifmanLafonLMNWZ2005PNAS2}, one considers the \emph{graph random walk Laplacian} $I-D^{-1}W$, where $W\in\mathbb{R}^{N\times N}$ is the weighted adjacency matrix defined by
\begin{equation}
  \label{eq:weighted_adjacency_matrix}
  W_{ij} = \exp\left(-\Vert F_i-F_j\Vert^2/\epsilon\right)
\end{equation}
with the bandwidth parameter $\epsilon>0$, and $D\in\mathbb{R}^{N\times N}$ is the diagonal degree matrix with entries $D_{ii} = \sum_{j}W_{ij}$ for all $i=1,\cdots,N$. If the points in $\mathscr{F}$ are sampled uniformly from a submanifold of $\mathbb{R}^{\ell}$, eigenvectors of $I-D^{-1}W$ converge to eigenfunctions of the Laplace-Beltrami operator on the smooth submanifold as $\epsilon\rightarrow0$ and the number of samples tends to infinity \cite{BelkinNiyogi2007,SingerWu2013}. Up to a similarity transform, the random walk graph Laplacian is equivalent to the symmetric \emph{normalized graph diffusion Laplacian}
\footnote{Note that $L$ is different from the \emph{normalized graph Laplacian}, which in standard spectral graph theory is constructed from an adjacency matrix with $0$ or $1$ in its entries, instead of the weighted adjacency matrix $W$ in \cref{eq:weighted_adjacency_matrix}. The crucial difference is in the range of eigenvalues: normalized graph Laplacian has eigenvalues in $\left[ 0,2 \right]$, whereas $L$ has eigenvalues in $\left[ 0,1 \right]$ (see \cite{SSN2009} or \cite[\S 2.2.2]{LafonThesis2004}.)}
\begin{equation}
  \label{eq:normalized_laplacian}
  L= D^{1/2}(I-D^{-1}W)D^{-1/2} = I - D^{-1/2}WD^{-1/2}.
\end{equation}
Let $L=\Phi\Lambda\Phi^{\top}$ be the eigen-decomposition of $L$, where $\Phi\in O \left( N \right)$ and $\Lambda$ is a diagonal matrix with all diagonal entries between $0$ and $1$. As in Diffusion Maps \cite{CoifmanLafon2006}, the columns of $\Phi\Lambda^{1/2}$ can be used as coordinate functions for a spectral embedding of the patch collection into $\mathbb{R}^N$. This embedding introduces the \emph{diffusion distance} $d \left( \cdot,\cdot \right)$ between patches $F_i, F_j$ ($0\leq i,j\leq N-1$) by setting $d \left( F_i, F_j \right)$ as the Euclidean distance between their embedded images in $\mathbb{R}^N$, i.e. the $i$th and $j$th row of $\Phi\Lambda^{1/2}$. If an $p$-dimensional embedding (with $p < \ell$) is desired, we can choose the $p$ columns of $\Phi\Lambda^{1/2}$ corresponding to the $p$ smallest eigenvalues of $L$ to minimize the error of approximation in {\it (P1)}. Since the columns of $\Phi\Lambda^{1/2}$ are already orthogonal, {\it (P2)} is automatically satisfied. \cref{fig:barbara128} is an example that illustrates nonlocal basis obtained from eigen-decomposition of a normalized graph diffusion Laplacian.

\begin{figure}[h]
\includegraphics[width = \textwidth]{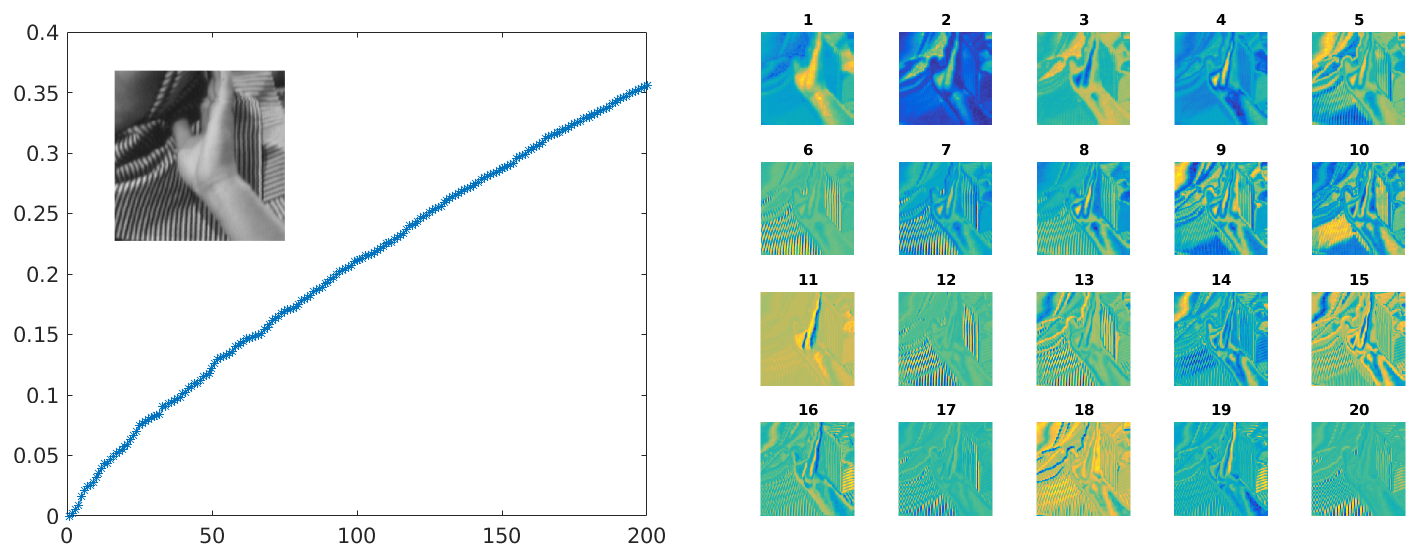}
\caption{\emph{Left:} The same cropped {\sc barbara} image of size $128\times 128$ as in \cref{fig:patch-svd}, and the smallest $200$ eigenvalues of the symmetric normalized graph diffusion Laplacian on the collection of all patches of size $\ell=4\times 4$; \emph{Right:} Eigenvectors associated with the smallest $20$ non-zero eigenvalues on the left.}
\label{fig:barbara128}
\end{figure}

\subsection{Energy concentration of convolution framelets}
\label{sec:energy-conc-conv}

Convolution framelets \cref{eq:defn_convframelet} is a signal representation scheme combining both local and nonlocal bases. Advantages of local and nonlocal bases, on their own, are known for specific signal processing tasks, under a general guiding principle seeking signal representations with certain \emph{energy concentration} patterns. Local basis such as wavelets or Discrete Cosine Transforms (DCT) are known to have ``energy compaction'' properties, meaning that real-world signals or images often exhibit a pattern of concentration of their energies in a few low-frequency components \cite{DCT1974,TenLectures1992,Mallat2008WaveletTour}; this phenomenon is fundamental for many image compression \cite{JPEG1992,JPEG20002001} and denoising \cite{DonohoJohnstone1994,DJKP1995} algorithms. On the other hand, nonlocal basis obtained from nonlinear dimension reduction or kernel PCA --- viewed as coordinate functions defining an embedding of the data set --- strives to capture, with only a relatively few number of basis functions, as much ``variance'' within the data set as possible; large portions of the variability of the data set is thus encoded primarily in the leading basis functions \cite{LeeVerleysen2007}. In the context of manifold learning, where the data points are assumed to be sampled from a smooth manifold, the number of eigenvectors corresponding to ``relatively large'' eigenvalues of a covariance matrix is treated as an estimate for the dimension of the underlying smooth manifold \cite{ISOMAP2000,LLE2000,BelkinNiyogi2003,LittleMaggioniRosasco2016}.

In practice, energy concentration patterns of signal representation in specific domains have been widely exploited to design powerful regularization schemes for reconstructing signals from noisy measurements. Since convolution framelets combine local and nonlocal basis, it is reasonable to expect that convolution framelet coefficients of typical signals tend to have energy concentration properties as well. To give a motivating example, consider the case in which both local and nonlocal bases concentrate energy on their low-frequency components, and basis functions are sorted in the order of increasing frequencies: typically the coefficient matrix $C=\Phi^{\top}FV$ will then concentrates its energy on the upper left block storing coefficients for convolution framelets corresponding to both local and nonlocal low-frequency basis functions. As an extreme example, if $\Phi$, $V$ in \cref{eq:patch_matrix_decom} come from the full-size singular value decomposition of $F$, i.e.
\begin{equation*}
  F = \Phi\Sigma V^{\top},\quad F,\Sigma\in\mathbb{R}^{N\times\ell}, \Phi\in\mathbb{R}^{N\times N}, V\in\mathbb{R}^{\ell\times\ell},
\end{equation*}
then the only non-zero entries in the coefficient matrix $C=\Phi^{\top}FV=\Sigma$ lie along the diagonal of its upper $\ell\times\ell$ block. We illustrate in \cref{fig:energy-concentration} the energy concentration of several different types of convolution framelets on a $1$-D random signal. \cref{fig:energy-concentration-barbara128} demonstrates the energy concentration of a $2$-D example using the same cropped {\sc barbara} image as in \cref{fig:patch-svd} and \cref{fig:barbara128}, in which we explore $4$ different types of local bases $V$ with the nonlocal basis $\Phi$ fixed as the graph Laplacian eigenvectors shown in \cref{fig:barbara128}; notice that in this example the energy concentrates more compactly in SVD and Haar bases than in DCT and random bases.

\begin{figure}
\centering
\includegraphics[width = 1.0\textwidth]{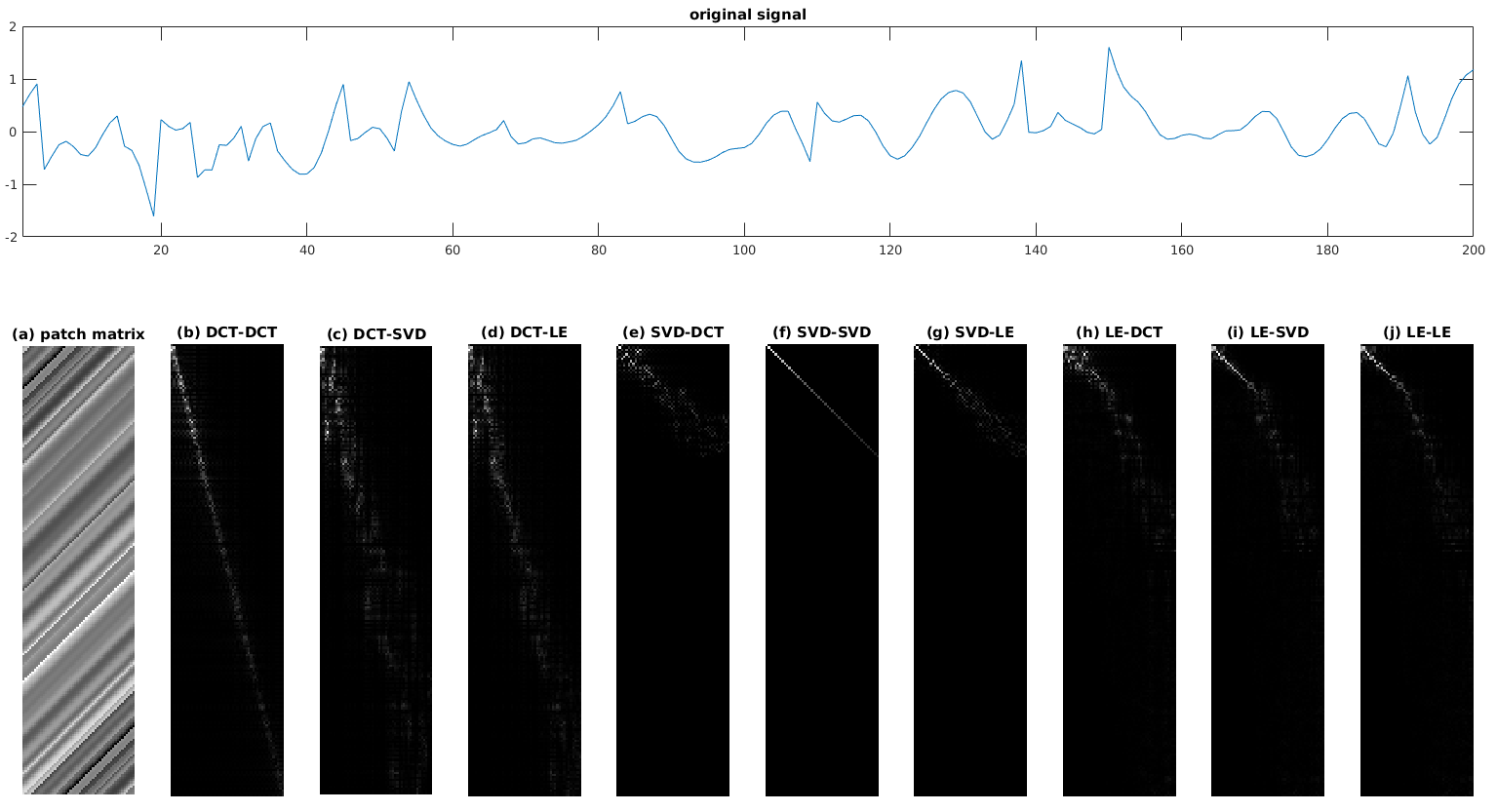}
\caption{A $1$-D signal of length $N=200$ and several convolution framelet coefficient matrices with fixed patch size $\ell=50$. \emph{Top:} A piecewise smooth $1$-D signal $f$ randomly generated from the stochastic model proposed in \cite{CohenDAles1997}. \emph{Bottom:} (a) The patch matrix $F$ of the signal $f$ on the top panel; (b)-(j) Energy concentration patterns of the coefficients of $f$ in several convolution framelets. The titles of subplots (b) to (j) indicate the different choices [\emph{Discrete Cosine Transform} (DCT), \emph{Singular Value Decomposition} (SVD), \emph{Laplacian Eigenmaps} (LE)] for the nonlocal basis $\Phi$ (appearing before the dash) and the local basis $V$ (appearing after the dash). These plots suggest that data-adaptive nonlocal bases (SVD or LE) tend to concentrate more energy on the upper-left part of the coefficient matrix $C$ than DCT does.}
\label{fig:energy-concentration}
\end{figure}

\begin{figure}[h]
\centering
\includegraphics[width = \textwidth]{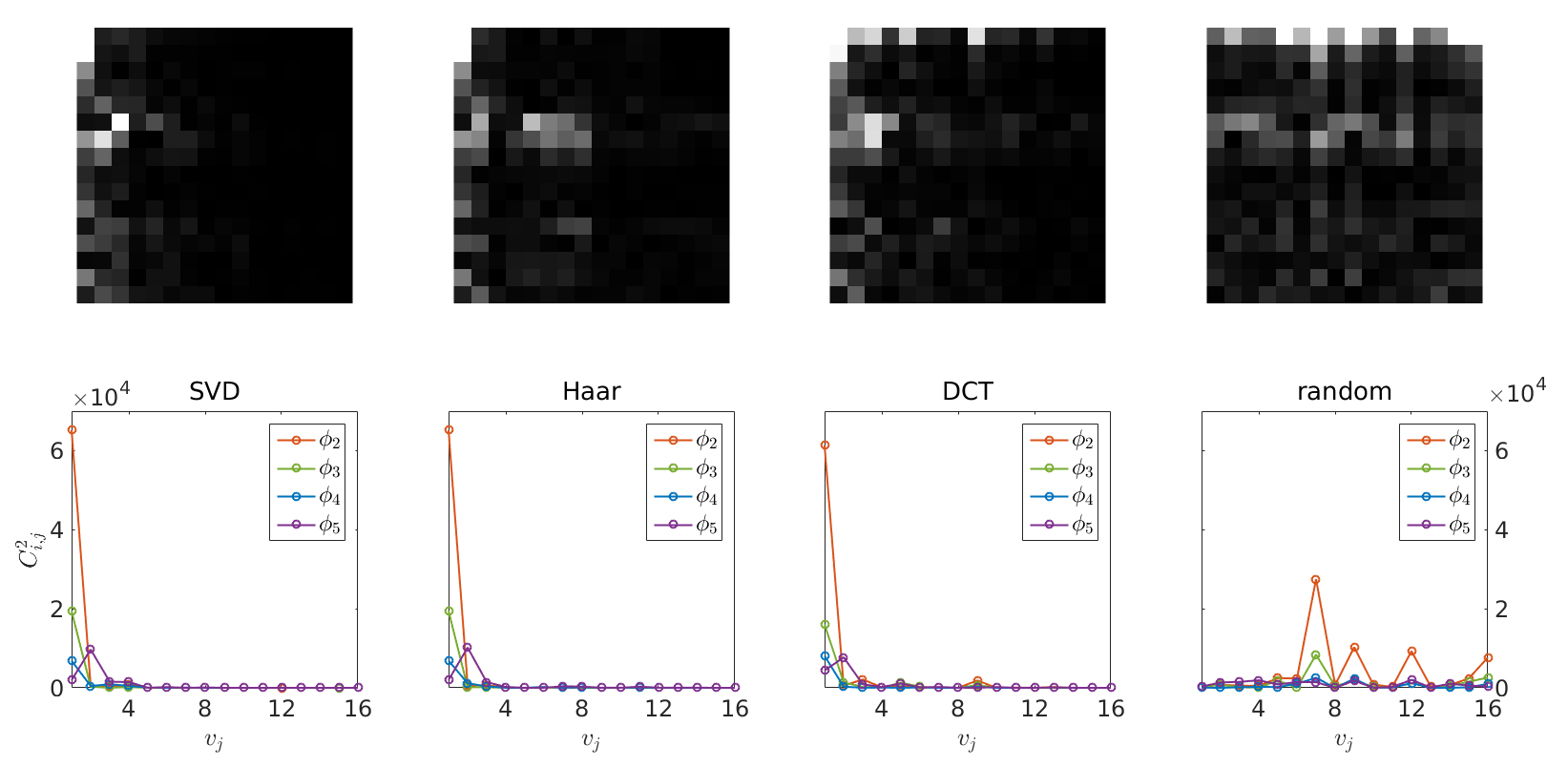}
\caption{Energy concentration of convolution framelet coefficient matrices of the same cropped $128\times 128$ {\sc barbara} image shown in \cref{fig:patch-svd} and \cref{fig:barbara128},  with fixed nonlocal basis $\Phi$ (Laplacian eigenfunctions) and different local bases $V$. The patch size is fixed as $\ell = 4\times 4$. \emph{Top:} The top $16\times 16$ blocks (corresponding to convolution framelets $\phi_i*v_j$ with $1\leq i,j\leq 16$) of the convolution framelet coefficient matrices using (from left to right) SVD basis, Haar basis, DCT basis, and random orthonormal basis. \emph{Bottom:} Squared coefficients $C^2_{ij}$ for the second to the fifth row in each coefficient block on the top panel. For each type of local basis, the line corresponding to $\phi_i$ ($i=2,3,4,5$) depicts $C^2_{ij}$ with $i$ fixed and $j$ ranging from $1$ to $16$.}
\label{fig:energy-concentration-barbara128}
\end{figure}

An interesting fact to notice is the following: in order for convolution framelets to have a structured energy concentration, it is \emph{not} strictly required that \emph{both} local and nonlocal bases have energy concentration properties. In a sense, regularization schemes based on convolution framelets are more flexible since the energy compaction effects of a local (resp. nonlocal) basis can be amplified through coupling with a nonlocal (resp. local) basis. More specifically, given $\Phi\in\mathbb{R}^{N\times N}$ satisfying mild assumptions\footnote{E.g. the leading columns of $\Phi$ give near isometric embeddings of the rows of $F$ satisfying {\it (P1)} and {\it (P2)} in Section~\ref{sec:local-nonlocal-approximation}; when columns of $\Phi$ are not orthogonal, a QR decomposition can be applied, see Section~\ref{sec:local-nonlocal-approximation} as well.}, we can systematically construct a local basis $V$ via minimizing a ``linear reconstruction loss'' such that the coefficient matrix $\Phi^{\top}FV$ concentrates its energy on the upper left block; this is the focus of Section~\ref{sec:energy-conc-guar}.

\subsection{Energy concentration guarantee via linear reconstruction}
\label{sec:energy-conc-guar}

Throughout this subsection, we will adopt the nonlocal point of view described in Section~\ref{sec:local-nonlocal-approximation}, and treat the patch matrix $F\in\mathbb{R}^{N\times\ell}$ of a signal $f\in\mathbb{R}^N$ as a point cloud $\mathscr{F}=\left\{ F_0,\cdots,F_{N-1} \right\}$ consisting of $N$ points in $\mathbb{R}^{\ell}$. Let $\embed:\mathbb{R}^{\ell}\supset\mathscr{F}\rightarrow\mathbb{R}^p$ be an embedding satisfying {\it (P1)} and {\it (P2)}, with $1\leq p\leq \ell$. Our goal is to ensure that the dimension reduction $\embed$ does not lose information in the original data set $\mathscr{F}$, by requiring the \emph{approximate invertibility}\footnote{We remark that the ``invertibility'' or ``reconstruction'' assumptions have been widely exploited in dimension reduction techniques, see e.g. \cite{Pearson1901PCA,Autoencoder2006}.} of $\embed$ on its image; as will be seen in \cref{prop:energy-concentration-guarantee},  the optimal $L^2$-reconstruction of $\mathscr{F}$ from its image $\embed \left( \mathscr{F} \right)$ leads to a local basis $V\in\mathbb{R}^{\ell}$. This particular local basis, paired with the nonlocal orthogonal system read off from the embedding $\embed$, renders convolution framelets that concentrate energy on the upper left block.

Let us motivate the linear reconstruction framework by considering a linear embedding $\embed:\mathbb{R}^{\ell}\rightarrow\mathbb{R}^p$ with $1\leq p\leq \ell$. Assume $\widetilde{A}\in\mathbb{R}^{\ell\times p}$ is full-rank, and $\mathscr{X}=\left\{ x_1,\cdots,x_N \right\}\subset\range \left( \widetilde{A} \right)\subset\mathbb{R}^{\ell}$, i.e., points in $\mathscr{X}$ are sampled from the $p$-dimensional linear subspace of $\mathbb{R}^{\ell}$ spanned by the columns of $\widetilde{A}$. Denote $X\in\mathbb{R}^{N\times\ell}$ for the data matrix storing the coordinates of $x_j$ in its $j$th row, and $X\widetilde{A} = \t\Phi \Sigma V_0^{\top}$ for the reduced singular value decomposition of $X\widetilde{A}$ (thus $\t\Phi \in\mathbb{R}^{N\times p}$, $V_0 \in\mathbb{R}^{p\times p}$, and $\Sigma\in\mathbb{R}^{p\times p}$ contains the singular values of $X\widetilde{A}$ along the diagonal and zeros elsewhere). Define $A:=\widetilde{A}V_0\in\mathbb{R}^{\ell\times p}$ and consider the linear embedding $\embed:\mathbb{R}^{\ell}\rightarrow\mathbb{R}^p$ given by
\begin{equation*}
  \embed \left( x \right)=x^{\top}A,\quad\forall x\in\mathbb{R}^{\ell}.
\end{equation*}
In matrix notation, the image of $\mathscr{X}$ under $\embed$ is $XA$. Note that {\it (P2)} is automatically satisfied because the columns of $XA=\t\Phi\Sigma$ are orthogonal.

Now that $\mathscr{X}$ is in $\range \left( \widetilde{A} \right)$ and $V_0\in\mathbb{R}^{p\times p}$ is orthonormal, we also have $\mathscr{X}\subset\range \left( A \right)$ and thus can write $X^{\top}=AB$ for some $B\in\mathbb{R}^{p\times N}$. It follows that $\embed$ is invertible on $\embed \left( \mathscr{X} \right)$ since
\begin{equation*}
  X^{\top} = AB = AA^{\dagger}AB = AA^{\dagger}X^{\top}=\left( AA^{\dagger} \right)^{\top}X^{\top}=\left( A^{\dagger} \right)^{\top}A^{\top}X^{\top},
\end{equation*}
where $A^{\dagger}=\left( A^{\top}A \right)^{-1}A^{\top}$ is the Moore-Penrose pseudoinverse of $A$, or equivalently
\begin{equation}
\label{eq:linear_projection_invertible}
  X = XAA^{\dagger}.
\end{equation}
In other words, in this case the dimension reduction $\embed$ is ``lossless'' in the sense that we can perfectly reconstruct $\mathscr{X}\subset\mathbb{R}$ from its embedded image $\embed\left(\mathscr{X}\right)$ in a space of lower dimension. Using a Gram-Schmidt process, we can write $A^{\dagger}=R\tV^{\top}$, where $R\in\mathbb{R}^{p\times p}$ is upper-triangular and $\tV\in\mathbb{R}^{p\times\ell}$ is orthogonal. This transforms \cref{eq:linear_projection_invertible} into
\begin{equation}
\label{eq:global-inverse}
  X = XAR\tV^{\top}=\t\Phi\Sigma R\tV^{\top} = \sum_{1\leq\, i\leq\, j\,\leq\, p}\Sigma_{ii}\,R_{ij}\;\phi_i\,v_j^{\top}
\end{equation}
where we invoked $XA=\t\Phi \Sigma$ and denoted $\left\{ \phi_i\mid 1\leq i\leq p \right\}$, $\left\{ v_j\mid 1\leq j\leq p \right\}$ for the columns of $\t\Phi$, $\tV$ respectively; note that the coefficient matrix $\Sigma R$ is upper-triangular. Let $\Phi\in O \left( N \right)$ and $V \in O\left( \ell \right)$ be orthonormal matrices, the columns of which extend $\left\{ \phi_i\mid 1\leq i\leq p \right\}$ and $\left\{ v_j\mid 1\leq j\leq p \right\}$ to complete bases on $\mathbb{R}^N$, $\mathbb{R}^{\ell}$ respectively. Following Section~\ref{sec:conv-framelets}, denote the outer products of columns of $\Phi$ with columns of $V$ as
\begin{equation*}
  \Psi_{ij} = \phi_i\,v_j^{\top}\quad\textrm{for }1\leq i\leq N, 1\leq j\leq \ell.
\end{equation*}
The expression \cref{eq:global-inverse}, now understood as an expansion of $X$ in orthogonal system $\left\{\Psi_{ij}\right\}$, uses only $p \left( p+1 \right)/2$ out of a total number of $N\times \ell$ basis functions. It is clear that in this case the energy of $X$ concentrates on (the upper triangular part of) the upper left block of the coefficient matrix $\Phi^{\top}XV$, or equivalently on components corresponding to $\left\{ \Psi_{ij}\,\mid\,1\leq i \leq j\leq p \right\}$. This establishes \cref{prop:energy-concentration-guarantee} below for all linear embeddings satisfying {\it (P2)}.

A similar argument can be applied to general nonlinear embeddings satisfying {\it (P2)}; all nonlinear dimension reduction methods based on kernel spectral embedding, such as Multi-Dimensional Scaling, Laplacian eigenmaps, and diffusion maps, belong to this category. In these cases we generally can not expect a perfect reconstruction of type \cref{eq:linear_projection_invertible}, but we can still seek a linear reconstruction in the form of $R\tV^{\top}$, with upper triangular $R$ and orthogonal $\tV$, that reduces the reconstruction error between $\widetilde{X}R\tV^{\top}$ and the original $X$ as much as possible.

\begin{proposition}
\label{prop:energy-concentration-guarantee}
  Let $\mathscr{X}=\left\{ x_1,\cdots,x_N \right\}$ be a point cloud in $\mathbb{R}^{\ell}$, $1\leq p\leq \ell$, and $\embed:\mathbb{R}^{\ell}\supset\mathscr{X}\rightarrow\mathbb{R}^p$ an embedding satisfying {\it (P2)}. Let $X\in\mathbb{R}^{N\times\ell}$ be the matrix storing the coordinates of $x_j$ in its $j$th row, and $\widetilde{X}\in\mathbb{R}^{N\times p}$ be the matrix storing the coordinates of $\embed \left( x_j \right)$ in its $j$th row ($1\leq j\leq N$). 
For $V_{\embed}$ given by
\begin{align}\label{eq:min-V}
\left( V_{\embed}, R_{\embed} \right)=\argmin_{\substack{\tV^{\top}\tV = I_p,\,\tV\in\mathbb{R}^{\ell\times p} \\ 
\t{R}_{ij} =\, 0,\, 1\leq j<i\leq p
}}\, 
\Vert\tX\,\t{R}\,\tV^{\top} - X\Vert_{\mathrm{F}}^2,
\end{align}
construct $V$ that extends $V_\embed$ to a complete orthonormal basis in $\mathbb{R}^l$; 
for $\Phi_\embed$ derived from the decomposition
\begin{equation}\label{eq:min-Phi}
  \tX = \Phi_{\embed}C_{\embed},\quad \Phi_{\embed}\in\mathbb{R}^{N\times p}\textrm{ orthonormal}, C_{\embed}\in\mathbb{R}^{p\times p}\textrm{ diagonal},
\end{equation}
also construct $\Phi$ that extends $\Phi_\embed$ to a complete orthonormal basis in $\mathbb{R}^N$.
Then $C=\Phi^{\top}XV\in\mathbb{R}^{N\times\ell}$ concentrates its energy on the upper triangle part of its upper left $p\times p$ block.
\end{proposition}
\begin{proof}
Let $\Phi_{\embed}$, $\Phi$ be defined as in the statement of \cref{prop:energy-concentration-guarantee}, $\tV_0\in\mathbb{R}^{\ell\times p}$ an arbitrary matrix with orthonormal columns, and $V_0$ an arbitrary extension of $\tV_0$ to an orthonormal basis on $\mathbb{R}^{\ell}$. The first term within the Frobenius norm of \cref{eq:min-V} can be re-written as
\begin{align}
\tX\,\t{R}\,\tV_{0}^{\top} &= \Phi_\embed\,C_\embed\,[\,\t{R},\,\mathbf{0}_{\ell,\,\ell-p}]\,V_0^{\top}\notag\\
&= \Phi\,\begin{bmatrix}
C_\embed\\ \,\mathbf{0}_{\,N-p,\,p}
\end{bmatrix}
[\,\t{R},\,\mathbf{0}_{\,\ell,\,\ell-p}]\,V_0^{\top} = \Phi
\begin{bmatrix}
\,C_\embed\,\t{R} & \mathbf{0}_{\,\ell,\,\ell-p}\\
\,\mathbf{0}_{\, N-p,\, p} & \mathbf{0}_{\, N-p,\,\ell-p}
\end{bmatrix}V_0^{\top}.
\end{align}
The minimization problem in \eqref{eq:min-V} can thus be reformulated as
\begin{align}\label{eq:min-C}
\min_{\substack{\tV\in O(\ell),\\
\t{R}_{ij} = 0,\,1\leq j<i\leq p}}
\left\Vert 
\begin{bmatrix}
\,C_\embed\,\t{R} & \mathbf{0}\\
\,\mathbf{0} & \mathbf{0}
\end{bmatrix} 
- C\,\right\Vert_{\mathrm{F}}^2,
\quad\textrm{where } C = \Phi^{\top}XV_0.
\end{align}
For any fixed orthonormal $V_0\in\mathbb{R}^{\ell\times \ell}$ (which also fixes $C$ since $\Phi$ and $X$ are already given), the optimal upper triangular matrix $\t{R}^{*}$ is clearly characterized by $\t{R}_{ij}^* = \left(C_{\embed}^{-1}C\right)_{ij}$ for all $1\leq i\leq j\leq p$. In fact, if we partition the matrix $C$ into blocks compatible with the block structure in \cref{eq:min-C}, denoted as
\begin{equation*}
  C = \begin{bmatrix}C_{LT}&C_{RT}\\C_{LB}&C_{RB} \end{bmatrix},
\end{equation*}
then $C_{\embed}\t{R}$ must cancel out with the upper triangle part of $C_{LT}$ in order to achieve the minimum of the minimization problem in \eqref{eq:min-C}. The optimization problem in \cref{eq:min-C} is thus equivalent to minimizing the $L^2$ energy of the remaining strictly lower triangular part of $C_{LT}$ together with the $L^2$ energy of the other three blocks $C_{RT},\,C_{LB}$, and $C_{RB}$. In addition, since $\Vert C\Vert_{\mathrm{F}}^2 = \Vert X\Vert_{\mathrm{F}}^2$ is constant, this is further equivalent to maximizing the $L^2$ energy of the upper triangular part of $C_{LT}$ (which gets canceled out with $C_{\embed}\t{R}^{*}$ anyway). Simply put, we have
\begin{align}\label{eq:min-tri}
\argmin_{\tV\in O(\ell)}\; \sum_{\substack{1\leq\, j<\,i\leq p,\\ \text{or }\, i,\,j\, > \,p}} C_{ij}^2
\quad =\quad
\argmax_{\tV\in O(\ell)}\;\sum_{1\leq\, i\leq\,j\leq p} C_{ij}^2.
\end{align}
This indicates that the optimal local basis $V_{\embed}$, and consequently its extension $V$ to a complete orthonormal basis on $\mathbb{R}^{\ell}$, must concentrate as much energy of the coefficient matrix $C$ as possible on the upper triangular part\footnote{One could also require that the energy concentrates on the lower triangle. Yet this is equivalent to changing $V$ to $PVP$, where $P =\begin{bmatrix}
J_p& 0\\
0& I_{\ell-p}
\end{bmatrix} $, and $J_p$ is anti-diagonal with non-zero entries all equal to one.} of the upper left $p\times p$ block.
\end{proof}
\begin{remark}\label{rmk:centering}
  The core idea behind \cref{prop:energy-concentration-guarantee} is to approximate the inverse of an arbitrary (possibly nonlinear) dimension reduction embedding $\embed$ using a global linear function
\begin{align}\label{eq:V-approx}
\embed^{-1}(\tX) \approx 
\new{\tX\,R\,\tV^{\top},}
\end{align}
where the upper triangular matrix $R\in \mathbb{R}^{p\times p}$ and the orthonormal matrix $\tV\in\mathbb{R}^{\ell\times p}$ together play the role of $A^{\dagger}$ in \cref{eq:linear_projection_invertible} for linear embeddings. Note that it is straightforward to incorporate a bias correction in the linear reconstruction \cref{eq:V-approx} by considering $\embed^{-1}(\tX) \approx \tX\,R\,\tV^{\top}-B$, where $B\in\mathbb{R}^{N\times \ell}$ is a ``centering matrix''; we assume $B=0$ in \cref{prop:energy-concentration-guarantee} for simplicity but the argument can be easily extended to $B\neq 0$. 
\end{remark}
\begin{remark}
  As will be seen in Section~\ref{sec:revisit-ldmm}, LDMM \cite{LDMM} implicitly exploits the energy concentration pattern characterized in \cref{prop:energy-concentration-guarantee}. More systematic exploitation of the energy concentration pattern lead to our improved design of \emph{reweighted LDMM}; see Section~\ref{sec:reweighted-ldmm}.
\end{remark}

\begin{example}[Example: Optimal Local Basis for Multi-Dimensional Scaling (MDS)]
  When $\embed$ is given by Multi-Dimensional Scaling (MDS), the optimal local basis $V$ in the sense of \cref{eq:min-V} consists of the right singular vectors of the centered data matrix $X$. To see this, first recall that in MDS the eigen-decomposition is performed on the doubly centered distance matrix $K = \frac{1}{2}HD^2H$, where $(D^2)_{ij} = d^2( X_i , X_j)$ and $H = I_N-\frac{1}{N}\vone_N\,\vone_N^{\top}$; coordinate functions for the low-dimensional embedding are then chosen as the eigenvectors of $K$ corresponding to the largest eigenvalues, weighted by the square roots of their corresponding eigenvalues. In particular, when $d(\cdot,\cdot)$ is the Euclidean distance on $\mathbb{R}^{\ell}$, one has $K = -HXX^{\top}H$, and the eigenvectors of $K$ correspond to the left singular vectors of the centered data matrix $HX$. (Here the centering matrix is $B = HX - X$; see \cref{rmk:centering}.) Let \(HX = U_X\Sigma_XV_X^{\top}\) be the reduced singular value decomposition of $HX$ as computed in the standard $MDS$ procedure. Then the optimal $V$ for \cref{eq:min-V} is exactly $V_X$, and the corresponding matrix basis has the sparsest representation of $X$. The proof of this statement can be found in \cref{sec: appendix-c}.
\end{example}

\subsection{Connection with nonlocal transform-domain image processing techniques}
\label{sec:extension-union-local-embed}

In some circumstances, the framework of convolution framelets can be interpreted as a nonlocal method applied to signal representation in a transform domain. For instance, if we use wavelets for the local basis $V$, and eigenvectors of the normalized graph diffusion Laplacian $L$ (see \cref{eq:normalized_laplacian}) for the nonlocal basis $\Phi$, then $L$ can be seen as defined on the wavelet coefficients since
\begin{equation*}
  W_{ij}=\exp\left(-\Vert F_i-F_j\Vert^2/\epsilon\right) = \exp\left(-\Vert F_iV-F_jV\Vert^2/\epsilon\right)\quad\forall 1\leq i,j\leq N.
\end{equation*}
Thus convolution framelet has the potential to serve as a natural framework for other nonlocal transform-domain image processing techniques. As an example, we show in what follows that BM3D \cite{BM3D2007,dabov2009bm3d}, a widely accepted state-of-the-art image denoising algorithms based on nonlocal filtering in transform domain, may also be interpreted through our convolution framelet framework, with a slightly extended notion of ``nonlocal basis''.

The basic algorithmic paradigm of BM3D can be roughly summarized in three steps. First, for a given image decomposed into $N$ patches of size $\ell$, denoted as $F_1,\cdots,F_N$, a block-matching process groups all patches similar to $F_i$ in a set $S_i$, and form matrix $F_{S_i}\in\mathbb{R}^{\left| S_i \right|\times\ell}$ consisting of patches in $S_i$; denote $\sigma=\sum_{i=1}^N \left| S_i \right|$. Second, let $V_{S_i}\in\mathbb{R}^{\ell\times\ell}$ be a local basis\footnote{In the original BM3D \cite{BM3D2007}, $V_{S_i}$ is set as DCT, DFT or wavelet and $\Phi_{S_i}$ is set as the 1-D Haar transform; in BM3D-SAPCA \cite{dabov2009bm3d}, $V_{S_i}$ is set to the principal components of $F_{S_i}$ when $|S_i|$ is large enough.}, $\Phi_{S_i}\in\mathbb{R}^{\left| S_i \right|\times \left| S_i \right|}$ be a nonlocal basis for $S_i$, and calculate coefficient matrix $C_{S_i}=\Phi_{S_i}^{\top}F_{S_i}V_{S_i}$ for group $S_i$; the matrix $F_{S_i}$ is then denoised by hard-thresholding (or Wiener filtering) $C_{S_i}$ and estimate $\widehat{F}_{S_i}=\Phi_{S_i}\widehat{C}_{S_i}V_{S_i}^{\top}$ from the resulting coefficient matrix $\widehat{C}_{S_i}$. In matrix form, this can be written as
\begin{equation}
\label{eq:BM3D_step2}
\begin{bmatrix}
  \widehat{F}_{S_1}\\
  \vdots\\
  \widehat{F}_{S_N}
\end{bmatrix}
=
\begin{bmatrix}
  \Phi_{S_1} & & \\
  & \ddots & \\
  & & \Phi_{S_N}
\end{bmatrix}
\begin{bmatrix}
  \widehat{C}_{S_1} & & \\
  & \ddots & \\
  & & \widehat{C}_{S_N}
\end{bmatrix}
\begin{bmatrix}
  V_{S_1} & & \\
  & \ddots & \\
  & & V_{S_N}
\end{bmatrix}^{\top}.
\end{equation}
In the third and last step, pixel values at each location of the image are reconstructed using a weighted average of all patches covering that location in the union of all estimated $\widehat{F}_{S_i}$'s; the contribution of an estimated patch contained in $\widehat{F}_{S_i}$ is proportional to $w_i:=\Vert C_{S_i}\Vert_{\ell_0}^{-1}$ i.e. inversely proportional to the sparsity of $C_{S_i}$. If we set $A_F\in\mathbb{R}^{N\times \sigma}$ to be an weighted incidence matrix defined by
\begin{equation*}
  \left(A_F\right)_{kq} =
  \begin{cases}
   w_q & \textrm{if patch $F_k$ is contained in $S_q$,}\\
   0 & \textrm{otherwise}
  \end{cases}
\end{equation*}
and let $D\in\mathbb{R}^{N\times N}$ be a diagonal matrix with
\begin{equation*}
  D_{kk} = \sum_{q=1}^{\sigma}\left( A_F \right)_{kq},
\end{equation*}
then the patch matrix of the original noise-free image is estimated via
\begin{equation}
  \label{eq:BM3D_step3}
  \widehat{F}
  =D^{-1}A_F
  \begin{bmatrix}
  \widehat{F}_{S_1}\\
  \vdots\\
  \widehat{F}_{S_N}
  \end{bmatrix}
  =D^{-1}A_F
\begin{bmatrix}
  \Phi_{S_1} & & \\
  & \ddots & \\
  & & \Phi_{S_N}
\end{bmatrix}
\begin{bmatrix}
  \widehat{C}_{S_1} & & \\
  & \ddots & \\
  & & \widehat{C}_{S_N}
\end{bmatrix}
\begin{bmatrix}
  V_{S_1} & & \\
  & \ddots & \\
  & & V_{S_N}
\end{bmatrix}^{\top}.
\end{equation}
The denoised image $\hat{f}$ is finally constructed from $\widehat{F}$ by taking a weighted average along anti-diagonals of $\widehat{F}$, with adaptive weights depending on the pixels.

In this three-step procedure, if we define
\begin{align}
\Phi &= D^{-1}A_F
\begin{bmatrix}
  \Phi_{S_1} & & \\
  & \ddots & \\
  & & \Phi_{S_N}
\end{bmatrix},\label{eq:frame-phi}\\
V &= \begin{bmatrix}
  V_{S_1} & & \\
  & \ddots & \\
  & & V_{S_N}
\end{bmatrix},\label{eq:frame-v}
\end{align}
then $\Phi$, $V$ together defines a tight frame similar to our construction of convolution framelets in Section~\ref{sec:conv-framelets}. The main difference here is that our energy concentration intuition described in Section~\ref{sec:energy-conc-conv} would not carry through to this setup, because in general every patch appears in multiple $F_{S_i}$'s and it is difficult to conceive that $\Phi$ consistently defines an embedding $\embed$ for the patches of the image. This technicality, however, can be easily remedied if we extend our framework from a global embedding over the entire data set $\mathscr{X}$ to a union of ``local embeddings'' on ``local charts'' of $\mathscr{X}$, i.e.
\begin{equation*}
  \embed_{S_i}:\,\mathbb{R}^{\ell}\supset S_i\rightarrow\mathbb{R}^{p_i},\, i = 1,\cdots,N
\end{equation*}
where $\mathscr{X}$ is covered by the unions of all $S_i$'s; note that the target spaces $\mathbb{R}^{p_i}$ do not even have to be of the same dimension (assuming $p_i\leq \ell$ for simplicity). For each embedding $\embed_{S_i}$, $\Phi_{S_i}\in\mathbb{R}^{|S_i|\times |S_i|}$ and $V_{S_i}\in\mathbb{R}^{\ell\times \ell}$ define nonlocal and local orthonormal bases, respectively. It can be expected that the energy concentration of convolution framelet coefficients in this setup will be more involved since both concentration patterns within and across local embedding spaces will be intertwined. We will further explore these interactions in a future work.

\section{LDMM as a regularization on convolution framelet coefficients}
\label{sec:revisit-ldmm}

In this section, we connect the discussion on convolution framelets with the recent development of Low Dimensional Manifold Model (LDMM) \cite{LDMM} for image processing. The basic assumption in LDMM is that the collection of all patches of a fixed size from an image live on a low-dimensional smooth manifold isometrically embedded in a Euclidean space. If we denote $f$ for the image and write $\mathcal{M}\left( f \right)=\mathcal{M}_{\ell}\left( f \right)$ for the manifold of all patches of size $\ell$ from $f$, then the image $f$ can be reconstructed from its (noisy) partial measurements $y$ by solving the optimization problem
\begin{equation}
  \label{eq:LDMM}
  \argmin_{f}\quad\text{dim}(\mathcal{M}(f)) + \lambda\Vert y - S f\Vert_2^2
\end{equation}
where $\lambda$ is a parameter and $S$ is the measurement (sampling) matrix. In other words, LDMM utilizes the dimension of the ``patch manifold'' $\mathcal{M}\left( f \right)$ as a regularization term in a variational framework. It is shown\footnote{We give a simplified proof of identify \cref{eq:dimension-L2-integral} in \cref{sec:appendix-short}.} in \cite{LDMM} that
  \begin{equation}
  \label{eq:dimension-L2-integral}
    \dim(\mathcal{M}(f)) = \sum_{j=1}^{\ell}|\nabla_\mathcal{M}\alpha_j(x)|^2,
  \end{equation}
where $\alpha_j$ is the $j$th coordinate function on $\mathcal{M}(f)$, i.e.
\begin{equation*}
  x = \left(\alpha_1(x),\cdots,\alpha_{\ell}(x)\right)\quad \forall x \in \mathcal{M}(f)\subset\mathbb{R}^{\ell},
\end{equation*}
and $\nabla_{\mathcal{M}}$ is the gradient operator on the Riemannian manifold $\mathcal{M}$. Note that $\alpha_j$ corresponds exactly to the $j$th column of the patch matrix $F$ of $f$, see \cref{eq:patch_matrix} and \cref{fig:patch-mat}.

With $\dim(\mathcal{M}(f))$ substituted by the right hand side of \cref{eq:dimension-L2-integral}, a split Bregman iterative scheme can be applied to the optimization problem \eqref{eq:LDMM}, casting the latter into sub-problems that optimize the dimension regularization with respect to each coordinate function $\alpha_j$ and the measurement fidelity term iteratively. In the $n$th iteration, the sub-problem of dimension regularization decouples into the following optimization problems on each coordinate function, 
\begin{equation}\label{eq:LDMM2}
\min_{\alpha_j\in H^1(\mathcal{M}_{(n-1)})}\, \Vert\nabla \alpha_j\Vert_{L^2(\mathcal{M}_{(n-1)})}^2 + \mu\sum_{x\in\mathcal{M}_{(n-1)}}|\alpha_j(x)-e_j(x)|^2,\quad j = 1,\cdots,\ell,
\end{equation}
where $\mathcal{M}_{(n-1)} = \mathcal{M}(f_{(n-1)})$ is the patch manifold associated with the reconstruction $f_{(n-1)}$ from the $\left( n-1 \right)$th iteration, $\mu$ is a penalization parameter, and $e_j$ is a function on this manifold originated from the split Bregman scheme. The Euler-Lagrange equations of the minimization problems in \eqref{eq:LDMM2} are cast into integral equations by the Point Integral Method (PIM), and then discretized as
\begin{align}\label{eq:linear-eq}
 \left[D_{(n-1)}-W_{(n-1)} + \mu W_{(n-1)}\right]\, F^j = \mu W_{(n-1)}\, E_{(n-1)}^j,\quad j = 1,\cdots, \ell,
\end{align}
where $F^j$, $E_{(n-1)}^j$ are the $j$th columns of the patch matrix $F$ and the matrix $E_{(n-1)}$, corresponding to $\alpha_j$ and $e_j$ in \eqref{eq:LDMM2} respectively; the weighted adjacency matrix $W_{(n-1)}$ and the diagonal degree matrix $D_{(n-1)}$, both introduced by PIM, are updated in each iteration after building the patch matrix $F_{(n-1)}$ from $f_{(n-1)}$. We refer interested readers to \cite{LDMM} for more details.

The rest of this section presents a connection we discovered between solving equation \eqref{eq:linear-eq} and an $\ell_2$-regularization problem on the convolution framelet coefficients of $f$.

\subsection{Dimension regularization in convolution framelets}
\label{sec:low-dim-as-reg}

The low-dimension assumption in LDMM is reflected in the minimization of a quadratic form derived from \cref{eq:dimension-L2-integral} for the column vectors of the patch matrix $F$ associated with image $f$. From a manifold learning point of view, \cref{eq:dimension-L2-integral} is not the only approach to impose dimension regularization. Since the columns of $F$ are understood as coordinate functions in $\mathbb{R}^{\ell}$, $F$ is indeed a data matrix representing a point cloud in $\mathbb{R}^{\ell}$ (see Section~\ref{sec:energy-conc-guar}). If this point cloud is sampled from a low-dimensional submanifold of $\mathbb{R}^{\ell}$, then one can attempt to embed the point cloud into a Euclidean space of lower dimension without significantly distorting pairwise distances between points. As we have seen in \cref{prop:energy-concentration-guarantee}, if there exists a good low-dimension embedding $\Phi$ for the data matrix $F$, the energy of convolution framelet coefficients will concentrate on a small triangular block on the upper left part of the coefficient matrix, provided that an appropriate local basis $V$ is chosen to pair with $\Phi$; a lower intrinsic dimension corresponds to a smaller upper left block and thus more compact energy concentration. Therefore, alternative to \cref{eq:dimension-L2-integral}, one can impose regularization on convolution framelet coefficients to push more energy into the upper left block of the coefficient matrix; see details below.

We start by reformulating the optimization problem \cref{eq:LDMM} proposed in \cite{LDMM} as an $\ell_2$-regularization problem for convolution framelet coefficients, where the convolution framelets themselves will be estimated along the way since they are adaptive to the data set. For simplicity of notation, we drop the sub-index $(n-1)$ in \cref{eq:linear-eq} as $W,\,D$ and $E$ are fixed when updating the patch matrix within each iteration. To distinguish from the notation $F_0,\cdots,F_{N-1}$ which stand for the rows of matrix $F$, we use super-indices $F^1,\cdots,F^{\ell}$ to denote the columns of $F$. Let $F_D^j = D^{1/2}F^j$ and $E_D^j = D^{1/2}E^j$, then the linear systems  \cref{eq:linear-eq} can be rewritten as
\begin{equation}
\label{eq:euler_lagrange_original_ldmm}
  (D-W)\,D^{-1/2}\,F_D^j + \mu W\,D^{-1/2}(F_D^j-E_D^j) = 0,\quad j = 1,\cdots,\ell.
\end{equation}
This system can be instantiated as the Euler-Lagrange equations of a different variational problem. Multiplying both sides of \cref{eq:euler_lagrange_original_ldmm} by $(W\,D^{-1/2})^{-1} = D^{1/2}\,W^{-1}$ from the left\footnote{The random walk matrix $D^{-1}W$ is invertible since all of its eigenvalues are positive, thus $W$ is also invertible.}, we have the equivalent linear system
\begin{equation}
\label{eq:euler_lagrange_new}
  D^{1/2}W^{-1}(D-W)D^{-1/2}F_D^j + \mu (F_D^j-E_D^j )= 0,\quad j = 1,\cdots,\ell.
\end{equation}
Notice that
\begin{equation*}
  D^{1/2}W^{-1}(D-W)D^{-1/2} = D^{1/2}W^{-1}D^{1/2} -I = (I - L)^{-1} - I,
\end{equation*}
where $L$ is the normalized graph diffusion Laplacian defined in \cref{eq:normalized_laplacian}. Therefore, solving \cref{eq:euler_lagrange_new} is equivalent to minimizing the following objective function
$$\sum_{j=1}^\ell\,\left[(F_D^j)^{\top} \,((I-L)^{-1}-I)\,F_D^j + \mu\Vert F_D^j-E_D^j\Vert^2\,\right]. $$
This is also equivalent to determining
\begin{align}\label{eq:LDMM3}
\argmin_{F\in\mathbb{R}^{N\times \ell}}\, \sum_{j=1}^\ell (F^j)^{\top}\;R_L\;F^j + \mu\Vert F-E\Vert_{\mathrm{F},\,D^{1/2}}^2,
\end{align}
where $\Vert\cdot\Vert_{\mathrm{F},\,A} = \Vert A\cdot\Vert_{\mathrm{F}} $ is the $A$-weighted Frobenius norm, and
\begin{equation}
\label{eq:graph_operator_ldmm}
  R_L = D\,W^{-1}(D-W) = D^{1/2}((I-L)^{-1}-I)D^{1/2}.
\end{equation}
The first term in \cref{eq:LDMM3} corresponds to the manifold dimension regularization term proposed in \cite{LDMM} whereas the second term promotes data fidelity. By the equivalence between \cref{eq:euler_lagrange_original_ldmm} and \cref{eq:euler_lagrange_new}, it suffices to focus on \cref{eq:LDMM3} hereafter.

To motivate our approach to analyze \cref{eq:LDMM3}, let us briefly investigate a similar but simpler regularization term based on nonlocal graph Laplacian, $\sum_j (F^j)^\top L F^j$, which differs from the dimension regularization term in \cref{eq:LDMM3} only in that the graph Laplacian $L$ replaces $R_L$. If we let $L = \Phi\Lambda\Phi^{\top}$ be the eigen-decomposition of $L$ with eigenvalues $\lambda_1,\cdots,\lambda_N$ on the diagonal of $\Lambda$ in ascending order, and pick any matrix $\tV\in\mathbb{R}^{\ell\times \ell'}$ satisfying $\tV\,\tV^{\top} = I_{\ell}$, then
\begin{align}\label{eq:reform-objective}
\sum_{j=1}^{\ell} (F^j)^{\top} LF^j
 =  tr\,\left(F^{\top}\Phi\Lambda\Phi^{\top}F\right)
 = tr\,\left( (\Phi^{\top}F\tV)^{\top}\Lambda(\Phi^{\top} F\tV)\right)
= \sum_{i = 1}^N\sum_{j=1}^{\ell'} \lambda_i\, C_{ij}^2,
\end{align}
where $C_{ij}$ is the $\left( i,j \right)$-entry of $C=\Phi^{\top}F\tV$. Minimizing this quadratic form will thus automatically regularize the energy concentration pattern by pushing more energy to the left part of $C$ where the columns correspond to smaller eigenvalues $\lambda_i$. Note that the only assumption we put on $\tV$ is that its columns constitutes a frame; by \cref{prop:frame}, $\tV$ being a frame in the patch space already suffices for constructing a convolution framelet system with $\Phi$.

Now we consider the minimization problem \cref{eq:LDMM3} with $R_L=D^{1/2}((I-L)^{-1}-I)D^{1/2}$ in the manifold dimension regularization term. 
Using $L=\Phi\Lambda\Phi^{\top}$, the operator $R_L$ can be written as $D^{1/2}\Phi\,\t{\Lambda}\,\Phi^{\top}D^{1/2}$, where $\t{\Lambda} = (I - \Lambda)^{-1}-I$ is a diagonal matrix. 
Similar to \cref{eq:reform-objective}, we have
\begin{align}\label{eq:reform-RL}
\sum_{j=1}^{\ell} (F^j)^{\top}\;R_L\; F^j
& = tr\,(F^{\top} D^{1/2}\Phi\t{\Lambda}\Phi^{\top}D^{1/2}F)
= tr\,(F^{\top}\t{\Phi}\t{\Lambda}\t{\Phi}^{\top}F) = \sum_{i=1}^N\sum_{j=1}^{\ell} \t{\lambda}_i\t{C}_{ij}^{\,2},
\end{align}
where $\t{\Phi} = D^{1/2}\Phi$, $\t{\lambda}_i = \lambda_i/\left(1-\lambda_i\right)$ for $i=1,\cdots,N$, and $\t{C} = \t\Phi^{\top}F\tV$ is the convolution framelet coefficient matrix. The optimization problem \cref{eq:LDMM3} can thus be recast as
\begin{equation}\label{eq:LDMM-eqv}
\begin{aligned}
\min_{F\in\mathbb{R}^{N\times \ell}}\quad & \sum_{i=1}^N\sum_{j=1}^{\ell} \,\t{\lambda}_i\t{C}_{ij}^{\,2} + \mu\Vert F-E\Vert_{\mathrm{F},\,D^{1/2}}^2\\
\textrm{s.t.}\quad & \t{C} = \t\Phi^{\top}F\tV.
\end{aligned}
\end{equation}
Ideally, if the first $p$ columns of $\t{\Phi}$ provide a low-dimensional embedding of the patch manifold $\mathcal{M}$ with small isometric distortion,
then $\lambda_i \approx 1$ for all $i >p$, which correspond to large $\t{\lambda}_i$ and forces $\t{C}_{ij}$ for the optimal $\t{C}$ to be close to $0$ for all $j\geq 1$ and all $i>p$. Intuitively, since $0\leq \widetilde{\lambda_1}\leq \widetilde{\lambda_2}\leq\cdots\leq\t{\lambda_N}$, the coefficient matrix $\t{C}$ of the minimizer of \cref{eq:LDMM-eqv} will likely concentrate most of its $\left\| \t{C} \right\|_{\mathrm{F}}^2$ energy on its top few rows corresponding to the smallest eigenvalues. Note that \cref{eq:reform-RL} imposes a much stronger regularization on the lower part of $\t{C}$ than \cref{eq:reform-objective} does, since $\lambda_j$'s are bounded from above by $1$ but $\t{\lambda}_j$ can grow to $+\infty$ as sub-index $j$ increases.

\begin{algorithm}
\caption{Inpainting using Reweighted LDMM with Local SVD Basis}
\label{alg:rw-LDMM}
  \begin{algorithmic}[1]
    \Procedure{rw-LDMM-SVD}{$f_{\left( 0 \right)},\ell$}\Comment{subsampled image $f_{\left( 0 \right)}\in\mathbb{R}^N$, patch size $\ell\in\mathbb{Z}^{+}$}
      \State $f_{\left( 0 \right)}\gets$ randomly assign values to missing pixels in $f_{\left( 0 \right)}$
      \State $n\gets 0$, $r\gets \lceil 0.2\ell\rceil$\Comment{reweight only the first $r$ columns}
      \State $d_{\left( 0 \right)}\gets 0\in\mathbb{R}^{N\times \ell}$
      \State $F_{\left( 0 \right)}\gets\textrm{patch matrix of }f^{\left( 0 \right)}$\Comment{$F_{\left( n \right)}\in\mathbb{R}^{N\times \ell}$}
      \While{not converge}
        \State $s_1,\cdots, s_r, V^1_{\left( n \right)},\cdots, V^r_{\left( n \right)}\gets$ partial SVD of $F_{\left( n \right)}$ \Comment{$s_i\in\mathbb{R}^+, V^i_{\left( n \right)}\in\mathbb{R}^{\ell\times 1}$}
        \State $V_{\left( n \right)}\gets\left[ V_{\left( n \right)}^1,\cdots,V_{\left( n \right)}^{r} \right]$\Comment{$V_{\left( n \right)}\in\mathbb{R}^{\ell\times r}$}
        \State $W_{\left( n \right)}\leftarrow$ weighted adjacency matrix constructed from $F_{\left( n \right)}$\Comment{$W_{\left( n \right)}\in\mathbb{R}^{N\times N}$}
               \begin{equation*}
                 W_{\left( n \right)} \left( i,j \right)=\exp \left( -\frac{\left\| F_i-F_j \right\|^2}{\epsilon} \right)\quad 1\leq i,j\leq N
                \end{equation*}
        \State $D_{\left( n \right)}\leftarrow$ diagonal matrix containing row sums of $W_{\left( n \right)}$\Comment{$D_{\left( n \right)}\in\mathbb{R}^{N\times N}$}
               \begin{equation*}
                  D_{\left( n \right)} \left( i,i \right)=\sum_{j=1}^NW_{\left( n \right)} \left( i,j \right)\quad 1\leq i\leq N
               \end{equation*}
        \State $E_{\left( n \right)}\gets F_{\left( n \right)}-d_{\left( n \right)}$\Comment{$E_{\left( n \right)}\in\mathbb{R}^{N\times \ell}$}
        \State $H_{\left( n \right)}\gets 0\in\mathbb{R}^{N\times r}$, $U_{\left( n \right)}\gets 0\in\mathbb{R}^{N\times \left( \ell-r \right)}$\Comment{$H_{\left( n \right)}\in\mathbb{R}^{N\times r}$, $U_{\left( n \right)}\in\mathbb{R}^{N\times \left( \ell-r \right)}$}
        \For{$i\gets 1,r$}
          \State $\gamma_i\gets 1-s_1^{-1}s_i$
          \State $H_{\left( n \right)}^i\gets$ solution of the linear system\Comment{$H_{\left( n \right)}^i\in\mathbb{R}^{N\times 1}$}
          \begin{equation*}
            \left( \gamma_i \left( D_{\left( n \right)}-W_{\left( n \right)} \right)+\mu W_{\left( n \right)} \right)H_{\left( n \right)}^i=\mu W_{\left( n \right)} E_{\left( n \right)} V_{\left( n \right)}^i
          \end{equation*}
        \EndFor
        \State $U_{\left( n \right)}\gets$ solution of the linear systems
        \begin{equation*}
          \left( D_{\left( n \right)}-W_{\left( n \right)}+\mu W_{\left( n \right)} \right)U_{\left( n \right)} = \mu W_{\left( n \right)}E_{\left( n \right)}\left(I_N-V_{\left( n \right)}V_{\left( n \right)}^{\top}\right)
        \end{equation*}
        \State $\widetilde{F}_{\left( n+1 \right)}\gets H_{\left( n \right)}V_{\left( n \right)}^{\top}+U_{\left( n \right)}+d_{\left( n \right)}$
        \State $\t{f}_{\left( n+1 \right)}\gets$ average out entries of $\widetilde{F}_{\left( n+1 \right)}$ according to \cref{eq:average_along_antidiagonal}
        \State $f_{\left( n+1 \right)}\gets$ reset subsampled pixels to their known values
        \State $F_{\left( n+1 \right)}\gets$ patch matrix of $f_{\left( n+1 \right)}$
        \State $d_{\left( n+1 \right)}\gets \t{F}_{\left( n+1 \right)}-F_{\left( n+1 \right)}$
        \State $n\gets n+1$
      \EndWhile\label{euclidendwhile}
      \State \textbf{return} $f_{\left( n \right)}$
    \EndProcedure
  \end{algorithmic}
\end{algorithm}

\subsection{Reweighted LDMM}
\label{sec:reweighted-ldmm}
As explained in Section~\ref{sec:low-dim-as-reg}, LDMM regularizes the energy concentration of convolution framelet coefficients by pushing the energy to the upper part of the coefficient matrix. This is clearly suboptimal from the point of view of \cref{prop:energy-concentration-guarantee}: the energy $\left\| \t{C} \right\|_{\mathrm{F}}^2$ should actually concentrate on the \emph{upper left} part as opposed to merely on the upper part of $C$, at least when an appropriate local basis is chosen. This observation motivates us to modify the objective function in \cref{eq:LDMM-eqv} to reflect the stronger patter of energy concentration pointed out in \cref{prop:energy-concentration-guarantee}. We refer to the modified optimization problem as \emph{reweighted LDMM}, or \emph{rw-LDMM} for short, since it differs from the original LDMM mainly in the weights in front of each $\t{C}_{ij}^2$ in \cref{eq:LDMM-eqv}.

Note that the objective function in the optimization problem \cref{eq:LDMM-eqv} is invariant to the choices of $\tV$ --- this is consistent with the interpretation of the regularization term as an estimate for the manifold dimension (the dimension of a manifold is basis-independent); but we can modify the objective function by incorporating patch bases as well. Consider a matrix $V = [v_1,\cdots,v_l]$ consisting of basis vectors for the ambient space $\mathbb{R}^{\ell}$ where the patches live, and define $s_j$, the \emph{energy filtered by $v_j$} of signal, as
\begin{equation}
\label{eq:energy-filtered-local}
  s_j = \Vert Fv_j\Vert^2 = \Vert f* v_j(-\cdot)\Vert^2=\sum_{i=1}^N\,C_{ij}^{2}.
\end{equation}
Note that $s_j$ is precisely the $j$th singular value of the patch matrix $F$ when $v_j$ is chosen as the $j$th right singular vector of $F$. If $s_j$ decays fast enough as $j$ increases, the patches on average will be approximated efficiently using a few $v_j$'s with large $s_j$ values. As discussed in Section~\ref{sec:energy-conc-conv}, natural candidates of $V$ include DCT bases, wavelet bases, or even SVD basis of $F$ (which are optimal low-rank approximations of $F$ in the $L^2$-sense; when the true $F$ is unknown, as in the case of signal reconstruction, we can also consider using right singular vectors obtained from an estimated patch matrix). After choosing such a basis $V$, the energy of the optimal coefficients matrix $C$ with respect to convolution framelets $\{\,\psi_{ij}\,\}$ concentrate mostly within the upper left $p\times r$ block, where $r$ depends on the decay rate of $s_j$. For this purpose, instead of using weights $\t{\lambda}_i$ in \eqref{eq:LDMM-eqv} alone, we propose to use weights $\t{\lambda}_i\,\gamma_j$, where $\gamma_j$ is a weight associated to $v_j$ such that $\gamma_j$ increases as $s_j$ decreases; one such example\footnote{We have also experimented with other forms of $\gamma_j$, for instance $\gamma_j=s_1s_j^{-1}-1$, which sends $\gamma_j$ to $+\infty$ when $s_j$ is close to $0$ and is thus a stronger regularization than the one used in rw-LDMM (which only sends $\gamma_j$ to $1$ as $s_j\rightarrow 0$). We do not use such stronger regularization weights since in practice they tend to produce over-smoothed results for reconstruction. This is not surprising, as natural images may contain intricate details that are encoded in convolution framelet components corresponding to small $s_j$'s; these details are likely smoothed out if $\gamma_j$ over-regularizes the convolution framelet coefficients.} is to set $\gamma_j = 1-s_1^{-1}s_j\in\left[0,1\right]$. In other words, we \emph{reweight} the penalties $\t{\lambda}_i$  to fine-tune the regularization. With this modification, the quadratic form \cref{eq:reform-RL} becomes\footnote{The reweighted quadratic form \cref{eq:add-weight}, as well as \cref{eq:rw-LDMM-topr} below, depends on $V$ only through $\Gamma^{1/2}$. In fact, as long as $V\,V^{\top} = I_{\ell}$, there holds $\Vert\, x-y\,\Vert_{\ell_2} = \Vert\, V^{\top}x-V^{\top}y\,\Vert_{\ell_2}$, and thus $W$ --- the weighted adjacency matrix constructed using a Gaussian RBF --- is $V$-invariant; consequently $R_L$ is $V$-invariant as well.
}
\begin{align}\label{eq:add-weight}
tr\Big(\,(FV\,\Gamma^{1/2})^{\top}\;R_L\;(FV\,\Gamma^{1/2})\,\Big) = \,\sum_{i=1}^{N}\sum_{j=1}^\ell\t{\lambda}_i\gamma_j\,\t{C}_{ij}^{\,2}.
\end{align}
Substituting this new quadratic energy for the original quadratic energy in \cref{eq:LDMM-eqv} and \cref{eq:LDMM3} yields the following optimization problem:
\begin{align}\label{eq:rw-LDMM-opt}
\argmin_{F\in\mathbb{R}^{N\times \ell}}\,& \sum_{j=1}^\ell \gamma_j(Fv_j)^{\top}R_L\left(Fv_j\right) + \mu\Vert F-E\Vert_{\mathrm{F},\,D^{1/2}}^2\\
\Leftrightarrow &\qquad\argmin_{F\in\mathbb{R}^{N\times \ell}}\, tr\Big(\,(FV\,\Gamma^{1/2})^{\top}\;R_L\;(FV\,\Gamma^{1/2})\,\Big) + \mu\Vert F-E\Vert_{\mathrm{F},\,D^{1/2}}^2.\notag
\end{align}
Using PIM, the Euler-Lagrange equations of \cref{eq:rw-LDMM-opt} turn into the corresponding linear systems:
\begin{align}\label{eq:rw-linear-eq}
 (\gamma_j(D-W) + \mu W) Fv_j = \mu W Ev_j,\quad j = 1,\cdots, \ell.
\end{align}
We shall refer to the optimization problem \cref{eq:rw-LDMM-opt} (sometimes also the linear system \cref{eq:rw-linear-eq} when the context is clear) \emph{reweighted LDMM}, or \emph{rw-LDMM} for short.

In practice, we observed that it often suffices to reweight the penalties only for the coefficients in the leading columns, i.e., keep the $\gamma_j$'s in \cref{eq:add-weight} only for $1\leq j\leq r$, where $r$ is a relatively small number compared with $\ell$. This can be done by first noting that the quadratic energy in \cref{eq:reform-objective} equals
\begin{equation}
\label{eq:split-quadratic-energy}
  \begin{aligned}
    tr\left( V^{\top} F^{\top}\Phi\Lambda\Phi^{\top} FV\right) = tr\left(V_r^{\top}F^{\top}\Phi\Lambda\Phi^{\top} FV_r\right)+tr\left(\left(FV^c_r\right)^{\top}\Phi\Lambda\Phi^{\top} FV^c_r\right),
  \end{aligned}
\end{equation}
where $V_r\in\mathbb{R}^{\ell\times r}$ consists of the left $r$ columns of $V$, and $V_r^c$ consists of the remaining columns. We can then reweight only the first term in the summation on the right hand side of \cref{eq:split-quadratic-energy}, i.e. replace \cref{eq:add-weight} with
\begin{align}\label{eq:rw-LDMM-topr}
tr\left((FV_r\Gamma_r^{1/2})^{\top}R_L(FV_r\Gamma_r^{1/2})\right) + tr\left(\left(FV^c_r\right)^{\top}\Phi\Lambda\Phi^{\top} FV^c_r\right)=\sum_{i=1}^N\t{\lambda}_i\left(\sum_{j=1}^r\gamma_jC_{ij}^2 +\!\!\! \sum_{j=r+1}^\ell C_{ij}^2\right).
\end{align}
The linear systems \cref{eq:rw-linear-eq} change accordingly to
\begin{align}\label{eq:rw-linear-eq-topr}
\begin{split}
 (\gamma_j(D-W) + \mu W) Fv_j &= \mu W Ev_j,\quad j= 1,\cdots, r,\\
  ((D-W) + \mu W) Fv_j &= \mu W Ev_j,\quad j= r+1,\cdots, \ell.
\end{split}
\end{align}
In all numerical experiments presented in Section~\ref{sec:numerical}, we set $r\approx 0.2\ell$, i.e. only coefficients in the left $20\%$ columns are reweighted in the regularization. We did not observe serious changes in performance when this economic reweighting strategy is adopted, but the improvement in computational efficiency is significant: for example, when right singular vectors of $F$ are used as local basis, solving \cref{eq:rw-linear-eq-topr} with partial SVD in each iteration is much faster than the full SVD required in \cref{eq:rw-linear-eq}. One can avoid explicitly computing $v_{r+1},\cdots,v_{\ell}$ by converting \cref{eq:rw-linear-eq-topr} into
\begin{equation}
  \label{eq:rw-linear-eq-topr2}
\begin{split}
 (\gamma_j(D-W) + \mu W) Fv_j &= \mu W Ev_j,\quad j= 1,\cdots, r,\\
  ((D-W) + \mu W) F \left( I_N-V_rV_r^{\top} \right) &= \mu W E \left( I_N-V_rV_r^{\top} \right),
\end{split}
\end{equation}
see \cref{alg:rw-LDMM} for more details\footnote{The linear systems in \cref{alg:rw-LDMM} actually produce $FV_r$ and $F \left( I_N-V_rV_r^{\top} \right)$ separately; the two matrices are combined together to reconstruct $F$ through $F=FV_rV_r^{\top}+F \left( I_N-V_rV_r^{\top} \right)$.}. Variants of \cref{alg:rw-LDMM} with other choices of $V$, such as DCT or wavelet basis, are just simplified versions of \cref{alg:rw-LDMM} where $V$ is a fixed input. Regardless of the choice for local basis, rw-LDMM yields consistently better inpainting results than LDMM in all of our numerical experiments; see details in Section~\ref{sec:comparison_inpainting}.

\section{Numerical results}
\label{sec:numerical}


\subsection{Linear and nonlinear approximation with convolution framelets}
\label{sec:comparison_approx}

For an orthogonal system $\left\{ e_n \right\}_{n\geq 0}$, the $N$-term \emph{linear} approximation of a signal $f$ is
\begin{equation*}
  f_N=\sum_{j=0}^{N-1}\left\langle f, e_j \right\rangle e_j,
\end{equation*}
whereas the $N$-term \emph{nonlinear} approximation of $f$ uses the $N$ terms with largest coefficients in magnitude, i.e.
\begin{equation*}
  \t{f}_N=\sum_{j\in \mathcal{I}_N}\left\langle f,e_j \right\rangle e_j,
\end{equation*}
where
\begin{equation*}
  \mathcal{I}_N\subset \mathbb{N}, \left| \mathcal{I}_N \right|=N,\textrm{ and } \left| \left\langle f,e_i \right\rangle \right|\geq \left| \left\langle f,e_k \right\rangle \right|\;\forall i\in\mathcal{I}_N, k\notin \mathcal{I}_N.
\end{equation*}
\begin{figure}[htbp]
\centering
\includegraphics[width = .8\textwidth]{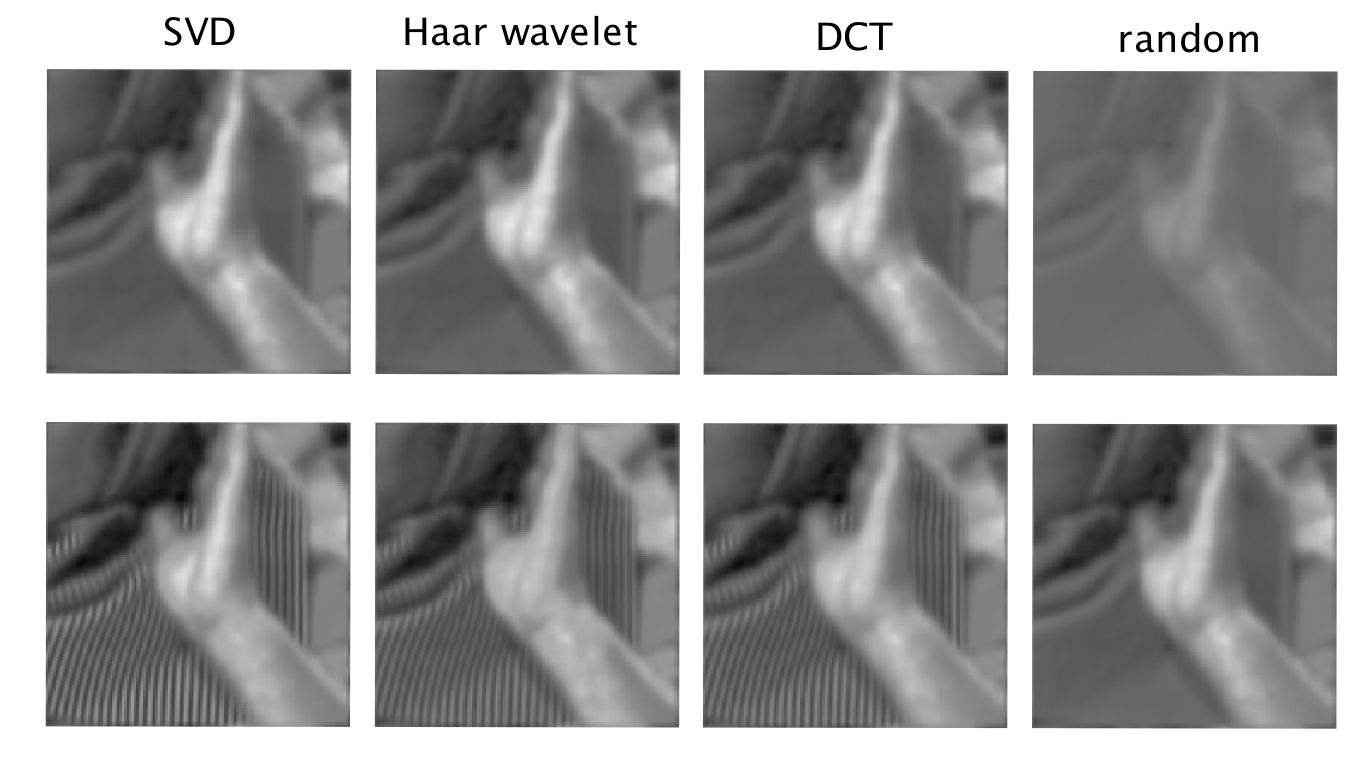}
\caption{Linear (\emph{top}) and nonlinear (\emph{bottom}) $8$-term convolution framelet approximation of the $128\times 128$ cropped {\sc barbara} image shown in \cref{fig:barbara128}. Except for the last column corresponding to random local basis, nonlinear approximation captures much more texture on the scarf than linear approximation does.}
\label{fig:approx}
\end{figure}

\begin{figure}[htbp]
\centering
\includegraphics[width = 1.0\textwidth]{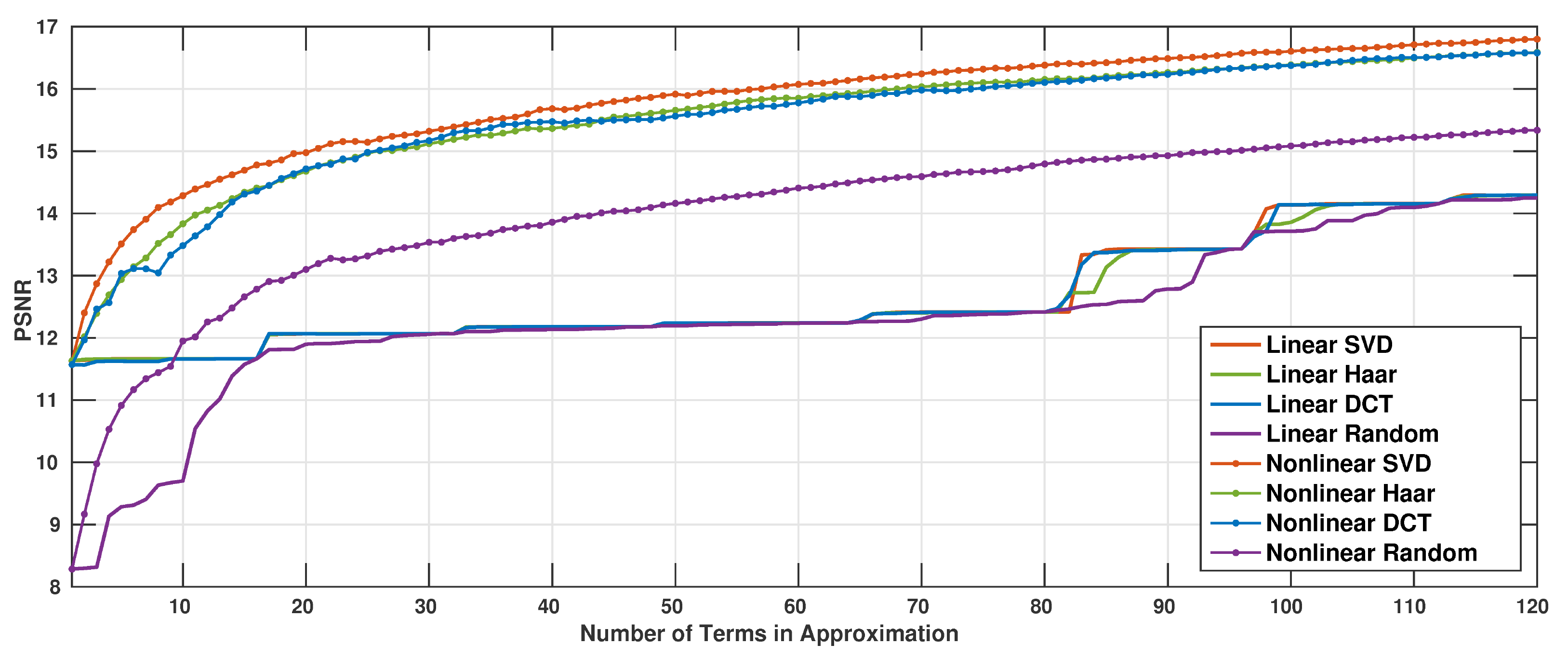}
\caption{PSNR as a function of the number of approximation terms in linear and nonlinear approximations of the $128\times 128$ cropped {\sc barbara} image in \cref{fig:barbara128}. Except for random local basis, the PSNR curves for linear approximation are almost identical.}
\label{fig:psnr}
\end{figure}

We compare in this section linear and nonlinear approximations of images using different convolution framelets $\left\{ \psi_{ij} = \ell^{-1/2}\phi_i*v_j \right\}$. To make sense of linear approximation, which requires a predetermined ordering of the basis functions, we fix the nonlocal basis $\left\{\phi_i\right\}$ to be the eigenfunctions of the normalized graph diffusion Laplacian $L$ (see \cref{eq:normalized_laplacian}); $\psi_{ij}$'s are then ordered according to descending magnitudes $\left|\left( 1-\lambda_i \right)s_j\right|$, where $\lambda_i$ is the $i$th eigenvalue of $L$ (which lies in $\left[ 0,1 \right]$) and $s_j$ is the energy of the function filtered by $v_j$ (see \cref{eq:energy-filtered-local}). We take a cropped {\sc barbara} image of size $128\times 128$, as shown in \cref{fig:barbara128}, subtract the mean pixel value from all pixels, then perform linear and nonlinear approximation for the resulting image. \cref{fig:approx} presents the $N$-term linear and nonlinear approximation results with $N=8$, patch size $\ell=16$ ($4\times 4$ patches), and local basis $V$ is chosen as patch SVD basis (right singular vectors of the patch matrix), Haar wavelets, DCT basis, and---as a baseline---randomly generated orthonormal vectors. In terms of visual quality, nonlinear approximation produces consistently better results here than linear approximation; as we also expect, SVD basis, Haar wavelets, and DCT basis all outperform the baseline using random local basis.

The superiority of nonlinear over linear approximation is also justified in terms of the Peak Signal-to-Noise Ratio (PSNR) of the reconstructed images. In \cref{fig:psnr}, we plot PSNR as a function of the number of terms used in the approximations. Except for random local basis, PSNR curves for all types of nonlinear approximation are higher than the curves for linear approximation, suggesting that sparsity-based regularization on convolution framelet coefficients may lead to stronger results than $\ell_2$-regularization. When the number of terms is large, even nonlocal approximation with random local basis outperforms linear approximation with SVD, wavelets, or DCT basis. \cref{fig:non-linear top} shows several convolution framelet components with the largest coefficients in magnitude for each choice of local basis.

\begin{figure}[htbp]
\centering
\includegraphics[width = 1.0\textwidth]{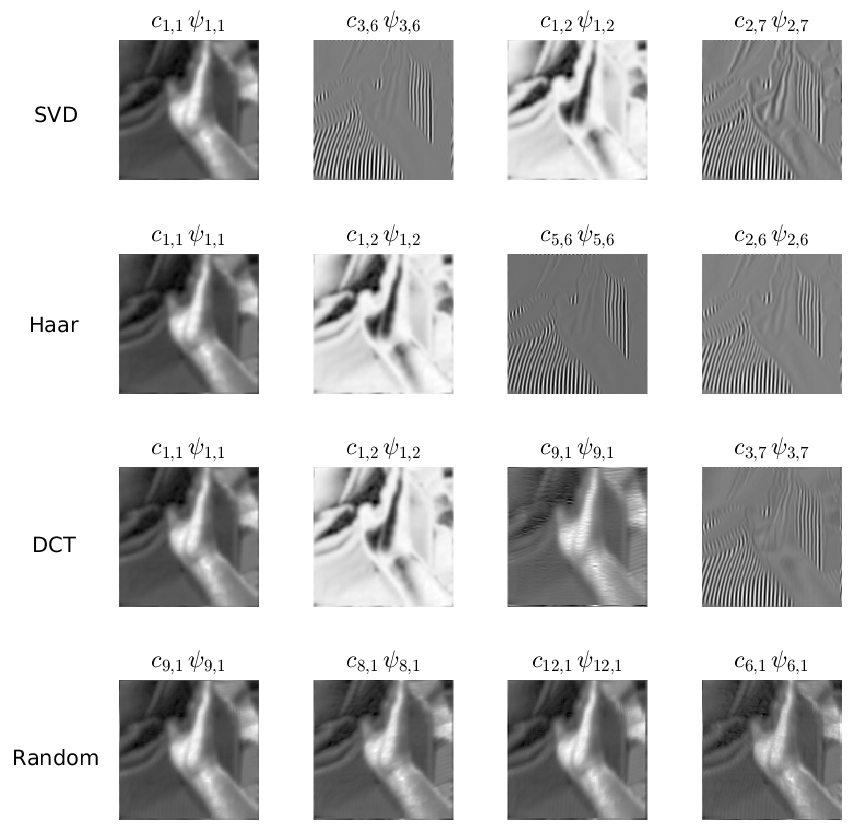}
\caption{The first four terms in each type of nonlocal convolution framelet approximation. Components in each row correspond to the four convolution framelet coefficients with largest magnitudes. The cropped $128\times 128$ {\sc barbara} image is the same as shown in \cref{fig:barbara128}.}
\label{fig:non-linear top}
\end{figure}




\subsection{Inpainting with rw-LDMM}
\label{sec:comparison_inpainting}

We first compare rw-LDMM with LDMM in the same setup as in \cite{LDMM} for image inpainting: given the randomly subsampled original image with only a small portion (e.g. $5\%$ to $20\%$) of the pixels retained, we reconstruct the image from an initial guess that fills missing pixels with Gaussian random numbers. The mean and variance of the pixel values filled in the initialization match those of the retained pixels. In our numerical experiments, rw-LDMM outperforms LDMM whenever the same initialization is provided. For LDMM, we use the \verb|MATLAB| code and hyperparamters provided by the authors of \cite{LDMM}; for rw-LDMM, we experimented with both SVD and DCT basis as local basis, and reweigh only the leading $20\%$ functions in the local basis. We run both LDMM and rw-LDMM for $100$ iterations on images of size $256\times 256$, and the patch size is always fixed as $10\times 10$. Peak Signal-to-Noise Ratio (PSNR)\footnote{$\textrm{PSNR}(f,f') \doteq 20\log_{10}(\textrm{MAX}(f)) - 10\,\log_{10}(\textrm{MSE}(f,f'))$.} of the reconstructed images obtained after the $100$th iteration\footnote{The number of iteration is also a hyperparameter to be determined. We use $100$ iterations to make fair comparisons between our results and those in \cite{LDMM}. In case the reconstruction degenerates after too many iterations due to over-regularization, one may --- for the purpose of comparison only --- also look at the highest PSNR within a fix number of iterations for each algorithm. We include those comparisons in \texttt{Supplementary Materials} as well.} are used to measure the inpainting quality. \cref{fig:rwLDMMvsLDMM} compares the three algorithms for a cropped {\sc Barbara} image of size $256\times 256$; \cref{fig:PSNRvsIter} plots PSNR as a function of the number of iterations and indicates that rw-LDMM outperforms LDMM consistently for a wide range of iteration numbers. More numerical results and comparisons can be found in \texttt{Supplementary Materials}.

\begin{figure}[htbp]
\centering
\includegraphics[width=1.0\textwidth]{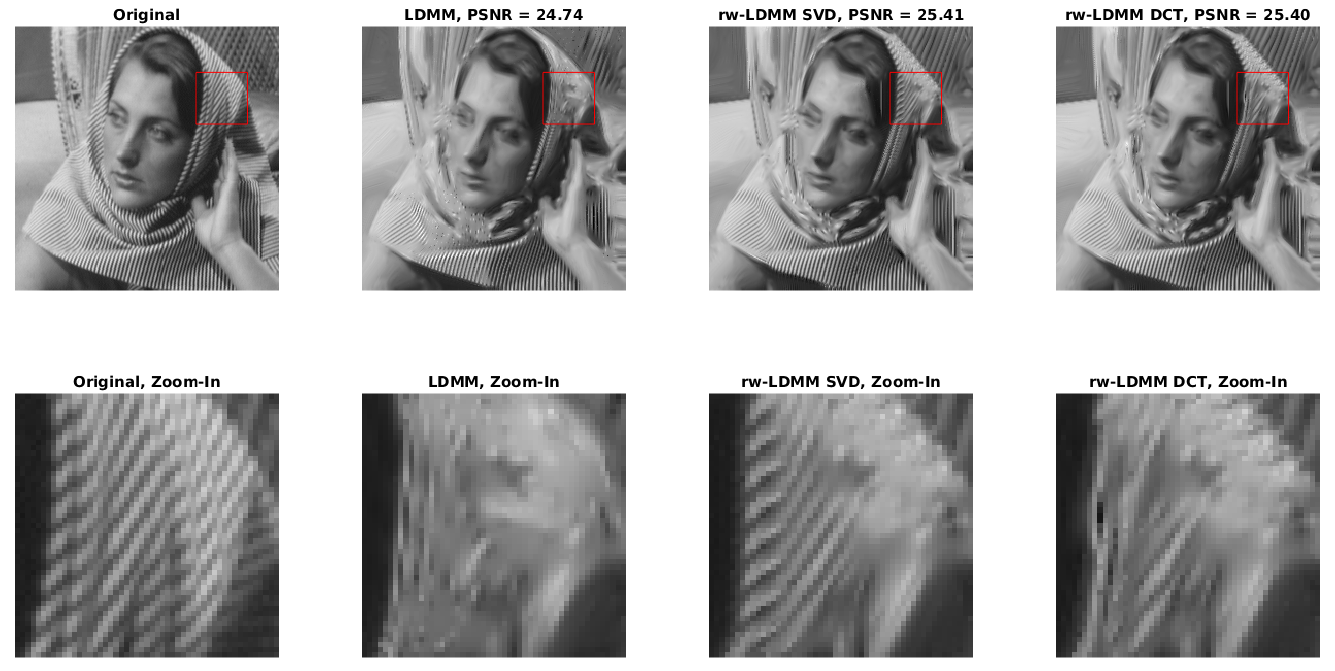}
\caption{Reconstructed $256\times 256$ {\sc Barbara} images from $10\%$ randomly subsampled pixels using LDMM and rw-LDMM. The same random initialization for missing pixels was used for LDMM and both SVD and DCT versions of rw-LDMM. \emph{Top:} Both rw-LDMM algorithms outperform LDMM in terms of PSNR. \emph{Bottom:} Zoom-in views of the $50\times 50$ blocks enclosed by red boxes on each reconstructed image illustrate better texture restoration by rw-LDMM.}
\label{fig:rwLDMMvsLDMM}
\end{figure}

\begin{figure}[htbp]
\centering
\includegraphics[width=0.75\textwidth]{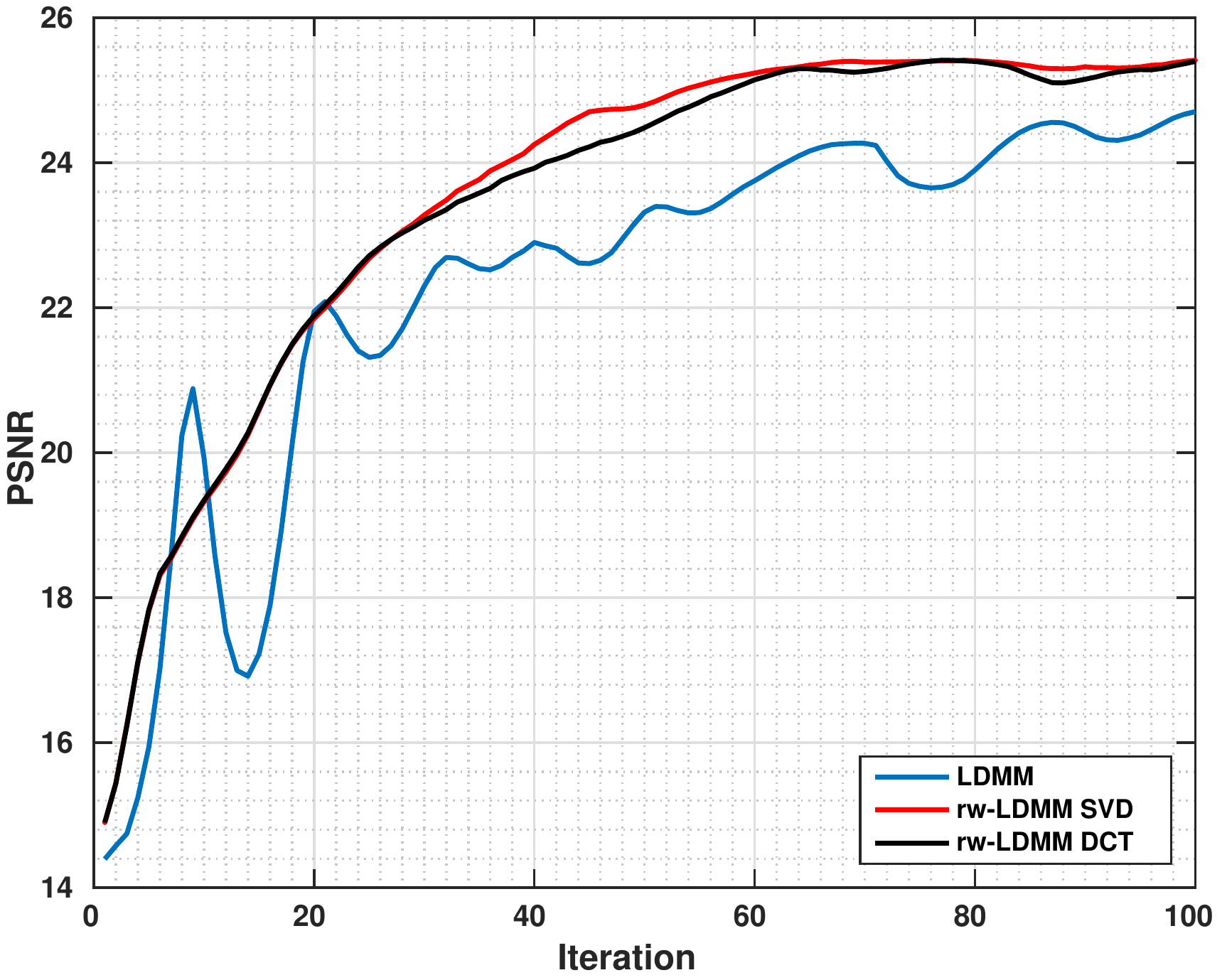}
\caption{PSNR of the reconstructed $256\times 256$ {\sc Barbara} image (see \cref{fig:rwLDMMvsLDMM}) at each iteration of LDMM, rw-LDMM with SVD, and rw-LDMM with DCT. After the $20$th iteration, both rw-LDMM algorithms always achieve higher PSNR than the original LDMM.}
\label{fig:PSNRvsIter}
\end{figure}

We also compare LDMM and rw-LDMM with ALOHA (Annihilating Filter-based Low-Rank Hankel Matrix) \cite{JY2015}, a recent patch-based inpainting algorithm using a low-rank block-Hankel structured matrix completion approach. For some test images with strong texture patterns (e.g. {\sc Barbara}, {\sc Fingerprint}, {\sc Checkerboard}, {\sc Swirl}), restoration from $10\%$ random subsamples by ALOHA reaches higher PSNR than LDMM and rw-LDMM; see \texttt{Supplementary Materials} for more details. However, we observe that the reconstruction by ALOHA sometimes contains artefacts that are not present in those obtained by rw-LDMM and LDMM, even though the ALOHA results can have higher PSNR (see e.g. \cref{fig:checkerboard_aloha} and \cref{fig:fingerprint_aloha}.)
Intuitively, this effect suggests different inpainting mechanisms underlying LDMM/rw-LDMM and ALOHA: LDMM and rw-LDMM, as indicated in \cite{LDMM}, ``spread out'' the retained subsamples to missing pixels, whereas ALOHA exploits the intrinsic (rotationally invariant) low-rank property of the block Hankel structure for each image patch. Numerical results with critically low subsample rate ($2\%$ and $5\%$) are in accordance with this intuition; see \cref{fig:stopsign_2percent}, \cref{fig:man_2percent}, as well as more examples in \texttt{Supplementary Materials}.

\begin{figure}[htbp]
\centering
\includegraphics[width=1.0\textwidth]{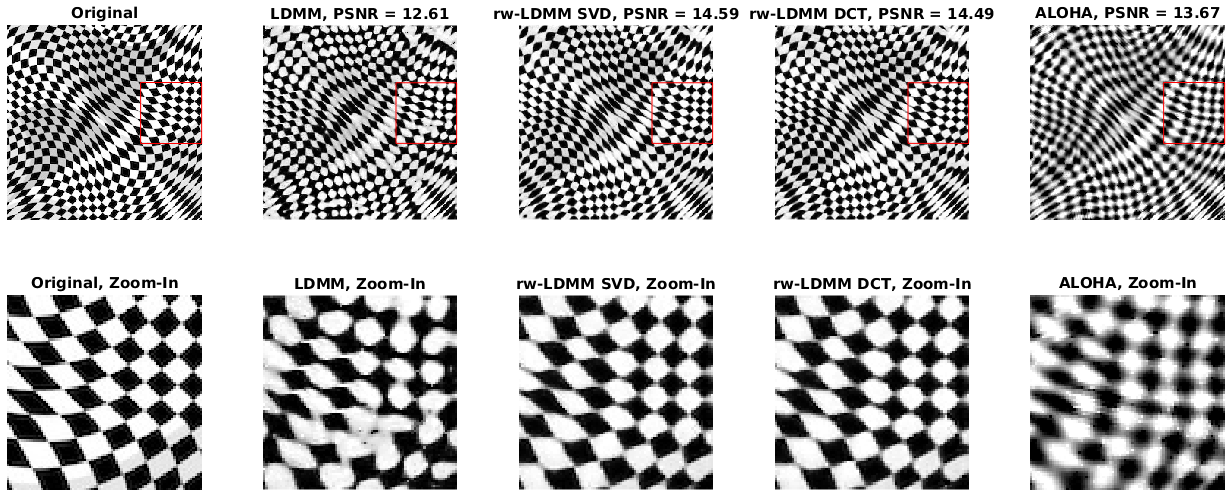}
\caption{Reconstructed $256\times 256$ {\sc Checkerboard} images from $10\%$ randomly subsampled pixels using LDMM, rw-LDMM, and ALOHA. \emph{Top:} Restored images. \emph{Bottom:} Zoom-in views of the $80\times 80$ blocks enclosed by red boxes. Compared with LDMM and ALOHA, the proposed rw-LDMM reconstructs images with higher PSNR and fewer visual artefacts.}
\label{fig:checkerboard_aloha}
\end{figure}

\begin{figure}[htbp]
\centering
\includegraphics[width=1.0\textwidth]{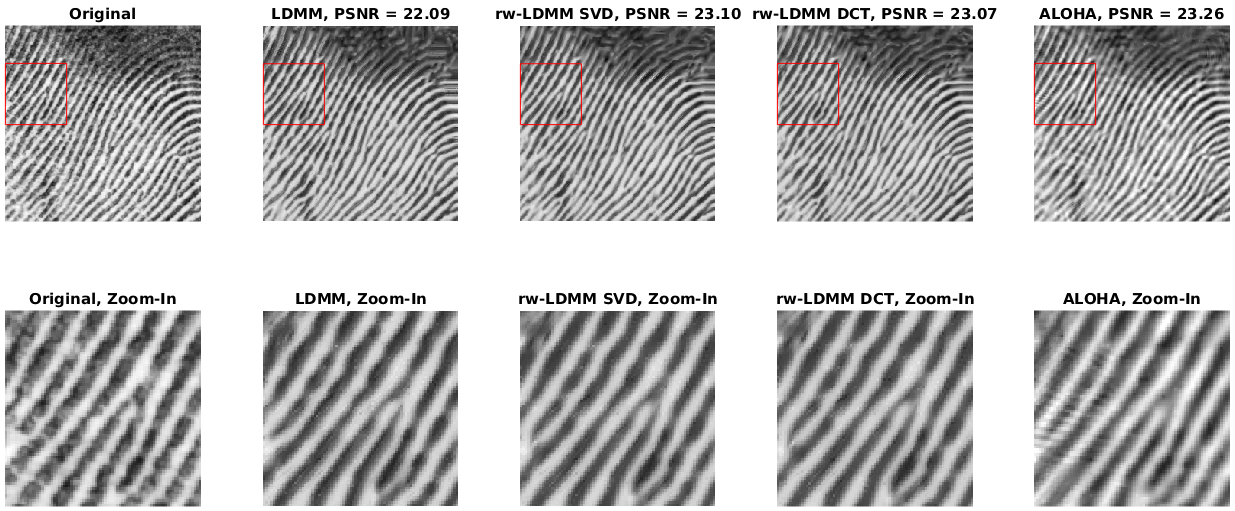}
\caption{Reconstructed $256\times 256$ {\sc Fingerprint} images from $10\%$ randomly subsampled pixels using LDMM, rw-LDMM, and ALOHA. \emph{Top:} Restored images. \emph{Bottom:} Zoom-in views of the $80\times 80$ blocks enclosed by red boxes. Compared with LDMM and ALOHA, the proposed rw-LDMM reconstructs images with comparable or higher PSNR and fewer visual artefacts.}
\label{fig:fingerprint_aloha}
\end{figure}

\begin{figure}[htbp]
\centering
\includegraphics[width=1.0\textwidth]{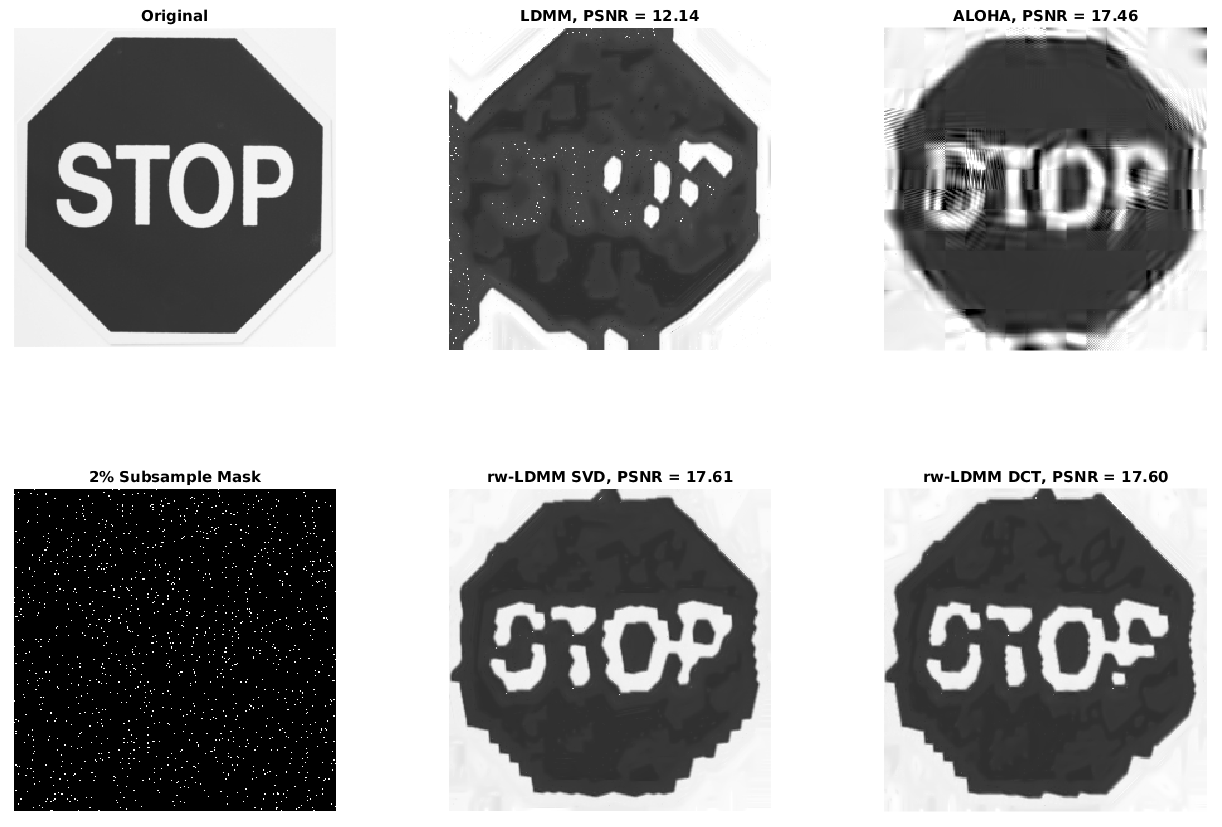}
\caption{Reconstructions of a $256\times 256$ traffic sign from $2\%$ randomly subsampled pixels using LDMM, rw-LDMM, and ALOHA. Both ALOHA and rw-LDMM restore legible letters even under such critically low subsample rate; rw-LDMM methods also achieve higher PSNR with fewer visual artefacts.}
\label{fig:stopsign_2percent}
\end{figure}

\begin{figure}[htbp]
\centering
\includegraphics[width=1.0\textwidth]{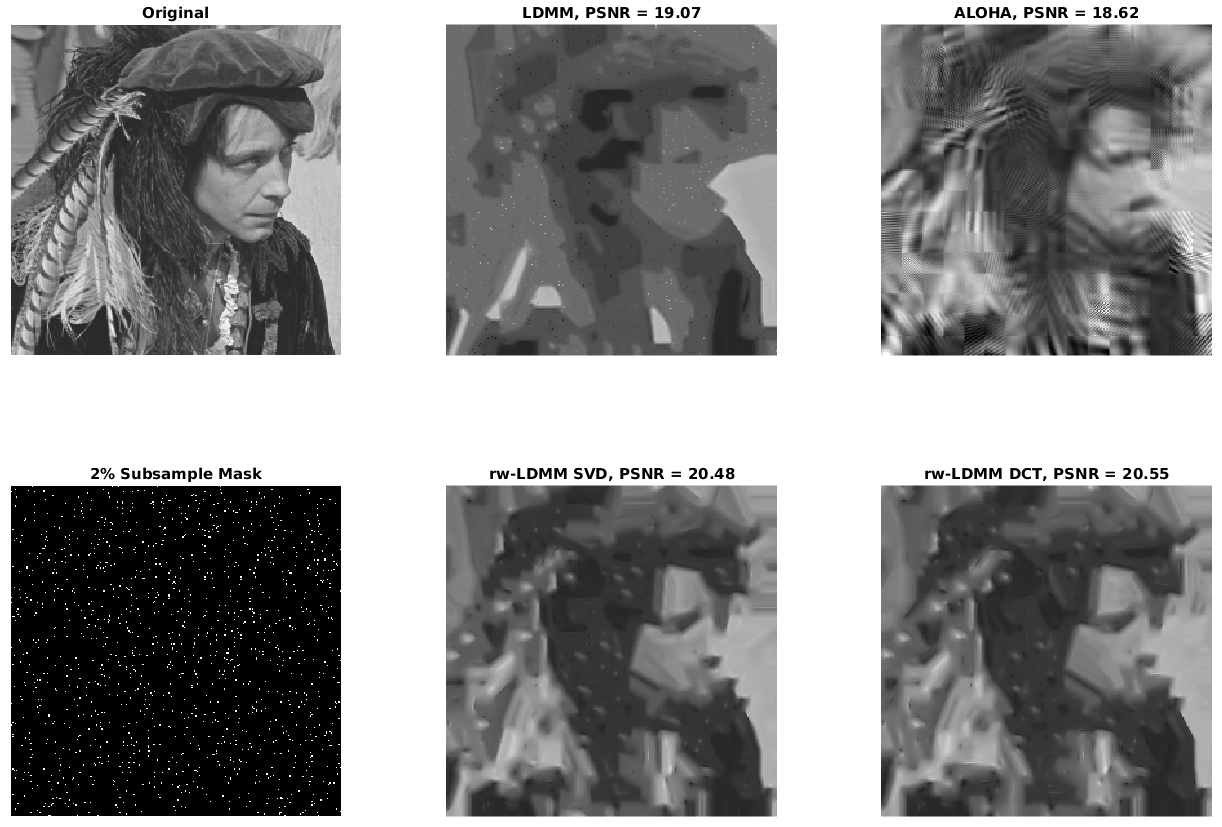}
\caption{Reconstructed $256\times 256$ {\sc Man} images from $2\%$ randomly subsampled pixels using LDMM, rw-LDMM, and ALOHA. The proposed rw-LDMM methods restore recognizable human shapes as well as color patterns on the hair and the hat decoration even under such critically low subsample rate.}
\label{fig:man_2percent}
\end{figure}



\section{Conclusion and future work}
\label{sec:conclusion}
In this paper, we present convolution framelets, a patch-based representation that combines local and nonlocal bases for image processing. We show the energy compaction property of these convolution framelets in a linear reconstruction framework motivated by nonlinear dimension reduction, i.e. the $L^2$-energy of a signal concentrates on the upper left block of the coefficient matrix with respect to convolution framelets. This energy concentration property is exploited to improve LDMM by incorporating ``near optimal'' local patch bases into the regularization mechanism, for the purpose of strengthening the energy concentration pattern. Numerical experiments suggest that the proposed reweighted LDMM algorithm performs better than the original LDMM in inpainting problems, especially for images containing high contrast non-regular textures. 

One direction we would like to explore is to compare the $\ell_2$-regularization with other regularization frameworks. In fact, our numerical experiments suggest that nonlinear approximation of signals with convolution framelets could outperform linear approximation, hence regularization techniques based on $\ell_1$- and $\ell_0$-norms have the potential to further improve the reconstruction performance. Furthermore, although we established an energy concentration guarantee in Section~\ref{sec:energy-conc-guar}, it remains unclear in concrete scenarios which local patch basis exactly attains the optimality condition in \cref{prop:energy-concentration-guarantee}. We made a first attempt in this direction for specific linear embedding in Section~\ref{sec:energy-conc-guar}, but similar results for nonlinear embeddings, as well as further extensions of the framework to unions of local embeddings (which we expect will also provide insights for other nonlocal transform-domain techniques, including BM3D), are also of great interest.

Another direction we intend to explore is the influence of the patch size $\ell$. Throughout this work, as well as in most patch-based image processing algorithms, the patch size is a hyperparamter to be chosen empirically and fixed; however, historically neuroscience experiments \cite{SPCB2007} and fractal image compression techniques \cite{BarnsleySloan1990Patent,Jacquin1992} provide evidence for the importance of perceiving patches of different sizes simultaneously in the same image. Since patch matrices corresponding to varying patch sizes of the same image are readily available, we can potentially combine convolution framelets across different scales to build multiresolution convolution framelets.


\appendix
\section{Proof of \cref{prop:frame}}
\label{sec:appendix}

\begin{lemma}\label{lem:v-conv}
Let $\tV\in\mathbb{R}^{l\times p},\, s.t.\,\tV\,\tV^{\top} = I_{\ell}$, then $\forall\, f\in\mathbb{R}^N,\, N\geq \ell$, \[f = \frac{1}{\ell}\, \sum_{i=1}^p\,f * \t{v}_i * \t{v}_i(-\cdot).\]
\end{lemma}
\begin{proof}[Proof of \cref{lem:v-conv}]
By definition \[\t{v}_i *\t{v}_i(-\cdot)[n] = \sum_{m=0}^{N-1}\t{v}_i[n-m]\t{v}_i[-m] = \sum_{m'=0}^{l-1} \t{v}_i[n+m']\t{v}_i[m'],\] since $\t{v}_i[m] = 0, \, \forall \ell\leq m\leq N,\, i = 1,\cdots,p$.
Therefore, 
\[\sum_{i=1}^p \t{v}_i*\t{v}_i(-\cdot)[n] = \sum_{i=1}^p\sum_{m=0}^{\ell-1}\t{v}_i[n+m]\t{v}_i[m],\]
and if we change the order of summation, we have $\sum_{i=1}^p \t{v}_i[m+n]\t{v}_i[m] = \delta(n)$, which follows from $\t{V}\t{V}^{\top} = I_{\ell}$. In sum, $\sum_{i=1}^p \t{v}_i*\t{v}_i(-\cdot)[n] = \ell\cdot \delta(n)$, hence $f = \frac{1}{\ell}\sum_{i=1}^p f* \t{v}_i *\t{v}_i(-\cdot)$. 
\end{proof}

\begin{proof}[Proof of \cref{prop:frame}]
By \cref{lem:v-conv},
\begin{align}\label{eq:expansion}
f & = \frac{1}{m}\sum_{i=1}^{m'}\, f*v^{\,S}_i*v^{\,S}_i(-\,\cdot)\notag= \frac{1}{m}\sum_{i=1}^{m'}\,\left(\sum_{j=1}^{n'}\, \langle f*v^{S}_i(-\cdot),\,v^L_j\rangle\, v^L_j\right) *v^{S}_i\notag\\
&= \sum_{i,j} \, \Big\langle f,\,\frac{1}{\sqrt{m}}\,v^L_j*v^S_i\Big\rangle\, \frac{1}{\sqrt{m}}\,v^L_j*v^S_i\notag \doteq \sum_{i,j} c_{ij}\, \psi_{ij},\quad \textrm{where }\psi_{ij} = \frac{1}{\sqrt{m}}\,v^L_j*v^S_i.
\end{align}
\end{proof}

\section{A Simplified Proof of the Dimension Identity \cref{eq:dimension-L2-integral}}
\label{sec:appendix-short}

\begin{proposition}
Assume a $d$-dimensional Riemannian manifold $\mathcal{M}$ is isometrically embedded into $\mathbb{R}^{\ell}$, with coordinate functions $\left\{ \alpha_j\mid 1\leq j\leq \ell \right\}$. Then at any point $x\in \mathcal{M}$
\begin{equation*}
    d = \dim(\mathcal{M}) = \sum_{j=1}^{\ell}|\nabla_\mathcal{M}\alpha_j(x)|^2,
\end{equation*}
where $\nabla_{\mathcal{M}}:C^{\infty}\left( \mathcal{M} \right)\rightarrow \mathfrak{X}\left(\mathcal{M}\right)$ is the gradient operator on $\mathcal{M}$.
\end{proposition}

\begin{proof}
  Let $\nabla: C^{\infty}\left( \mathbb{R}^{\ell} \right)\rightarrow \mathfrak{X}\left( \mathcal{M} \right)$ be the gradient operator on $\mathbb{R}^{\ell}$. For any $f\in C^{\infty}\left( M \right)$, if $f$ is the restriction to $\mathcal{M}$ of a smooth function $\bar{f}\in C^{\infty}\left( \mathbb{R}^{\ell} \right)$, then $\nabla_{\mathcal{M}}f \left( x \right)$ is the projection of $\nabla \bar{f}$ to $T_xM$, the tangent space of $\mathcal{M}$ at $x\in \mathcal{M}$. Now, fix an arbitrary point $x\in\mathcal{M}$ and let $E_1 \left( x \right),\cdots,E_d \left( x \right)$ be an orthonormal basis for $T_x\mathcal{M}$. We have for any $1\leq j\leq \ell$
\begin{equation*}
  \nabla_{\mathcal{M}}\alpha_j \left( x \right) = \sum_{k=1}^d \left\langle \nabla \alpha_j \left( x \right), E_k \left( x \right) \right\rangle E_k \left( x \right),
\end{equation*}
and thus
\begin{equation*}
  |\nabla_\mathcal{M}\alpha_j(x)|^2 = \sum_{k=1}^d \left|\left\langle \nabla \alpha_j \left( x \right), E_k \left( x \right) \right\rangle \right|^2.
\end{equation*}
Note that $\nabla \alpha_j$ is a constant vector in $\mathbb{R}^{\ell}$ with $1$ at the $j$th entry and $0$ elsewhere. Consequently, inner product $\left\langle \nabla \alpha_j \left( x \right), E_k \left( x \right) \right\rangle$ simply picks out the $j$th coordinate of $E_k \left( x \right)$. Therefore
\begin{equation*}
  \begin{aligned}
    \sum_{j=1}^{\ell}|\nabla_\mathcal{M}\alpha_j(x)|^2 &= \sum_{j=1}^{\ell}\sum_{k=1}^d \left|\left\langle \nabla \alpha_j \left( x \right), E_k \left( x \right) \right\rangle \right|^2 = \sum_{k=1}^d\left(\sum_{j=1}^{\ell} \left|\left\langle \nabla \alpha_j \left( x \right), E_k \left( x \right) \right\rangle \right|^2\right)\\
    &= \sum_{k=1}^d \left| E_k \left( x \right) \right|^2 = \sum_{k=1}^d 1 = d
  \end{aligned}
\end{equation*}
which completes the proof.
\end{proof}

\section{Proof of the optimality and the sparsity of SVD in MDS}
\label{sec: appendix-c}

\begin{proposition}\label{prop:MDS-opt}
Let \(HX = U_X\Sigma_XV_X^{\top}\) be the reduced singular value decomposition of the centered data matrix $X\in\mathbb{R}^{N\times\ell}$, where $H = I_N-\frac{1}{N}\vone_N\,\vone_N^{\top}$ is the centering matrix. The optimal $V$ for \cref{eq:min-V} is exactly $V_X$, and the corresponding matrix basis has the sparsest representation of $X$. The proof of this statement can be found in \cref{sec: appendix-c}.
\end{proposition}

\begin{proof}
Without loss of generality, assume $X = HX$.
In MDS, $\Phi_\embed = U_X$ with $p = \ell-1$. The entries of the coefficient matrix $C=\Phi^{\top}X\t{V}$ can be explicitly computed as
\[C_{ij} =\phi_i^{\top} X \,\tv_j = u_{X,\,i}^{\top} X \,\tv_j = u_{X,\,i}^{\top} U_X\,\Sigma_X \,V_X^{\top}\,\,\tv_j = \sigma_{X,i}\, v_{X,\,i}^{\top}\,\tv_j,\]
where $u_{X,\,i},\,\tv_j$ are the columns of $U_X$ and $\tV$, respectively, and $\sigma_{X,i}$ is the $i$th diagonal entry of $\Sigma_X$. According to \cref{eq:min-tri},  the optimal $\t{V}$ should satisfy $v_{X,\,i}^{\top}\,\tv_j = 0$ for all $i > j$, which is achieved by setting $\tv_i = v_{X,\,i}$. Moreover, since $\rank(C) = \ell-1$, $C$ has at least $\left(\ell-1\right)$ non-zero entries; it follows from $v_{X,\,i}^{\top}v_{X,\,j} = \delta_{i,j}$ that $C=U_X^{\top}XV_X$ has exactly $\left( \ell-1 \right)$ non-zero entries and is thus the sparsest representation.
\end{proof}

\section*{Acknowledgments}
We thank the authors of LDMM \cite{LDMM} for providing us their code. The work of Yue M. Lu was supported in part by the NSF under grant CCF-1319140 and by ARO under grant W911NF-16-1-0265. The work of Rujie Yin was supported in part by the NSF under grant 1516988.

\bibliographystyle{siamplain}
\bibliography{main}

\newpage

\setcounter{section}{0}
\renewcommand{\thesection}{SM\arabic{section}}
\renewcommand{\appendixname}{}

\begin{center}
\large \textbf{SUPPLEMENTARY MATERIALS}
\end{center}
\vspace{0.2in}

\section{Comparing LDMM, reweighted LDMM, and ALOHA for inpainting}
\label{sec:comp-ldmm-rewe}

In the following tables, we compare the inpainting performance of LDMM and reweighted LDMM with ALOHA \cite{JY2015} on $9$ test images of size $256\times 256$. For ALOHA, we use hyperparameters provided by the authors\footnote{We use the code from one of the authors' website \texttt{http://bispl.weebly.com/aloha-inpainting.html}.}; for LDMM and reweighted LDMM, we use the same parameters as described in Section~\ref{sec:numerical}. For columns of LDMM and reweighted LDMM, the row titled \textbf{symm. GL} indicates whether or not the graph diffusion Laplacian used for generating nonlocal basis in convolution framelets is symmetrized. The difference between symmetrized and un-symmetrized graph Laplacians are as follows: Recall from Section~\ref{sec:local-nonlocal-approximation} of the main text that $L=I-D^{-1/2}WD^{-1/2}$ where
$$W_{ij}=\exp \left( -\left\| F_i-F_j \right\|^2/\epsilon \right).$$
In the original implementation of LDMM, the value $\epsilon$ is chosen adaptively with respect to the rows of $W$ by setting
$$W_{ij}=\exp \left( -\left\| F_i-F_j \right\|^2/\epsilon_i \right)$$
where $\epsilon_i$ is the distance between $F_i$ and its $20$th nearest neighbor; the resulting weight matrix $W$ is generally non-symmetric. In all numerical experiments we also compare the results between this non-symmetric construction with its symmetrized counterpart setting
$$W_{ij}=\exp \left( -\left\| F_i-F_j \right\|^2/\sqrt{\epsilon_i\epsilon_j} \right)$$
where $\epsilon_i,\epsilon_j$ are the distances between $F_i,F_j$ and their $20$th nearest neighbors, respectively. This additional symmetrization step is well-known for practitioners of diffusion maps \cite{CoifmanLafon2006,CoifmanLafonLMNWZ2005PNAS1,CoifmanLafonLMNWZ2005PNAS2} and spectral clustering \cite{RCY2011,SWY2015,JY2016} algorithms. We include results using both symmetrized and un-symmetrized graph diffusion Laplacians since numerical experiments illustrates that in some circumstances the symmetrized Laplacian leads to improvement both visually and in PSNR.

\subsection{Final PSNR in LDMM and rw-LDMM}
\label{sec:final-psnr-ldmm}

In \cref{table:psnr-2-percent} through \cref{table:psnr-50-percent}, the PSNR in the columns of LDMM and reweighted LDMM are computed from the final image obtained in the $100$th iteration, as opposed to the maximum PSNR among the first $100$ iterations in Section~\ref{sec:maximum-psnr-ldmm}.

\begin{table}[htbp]
\centering
\caption{PSNR of images reconstructed from $\mathbf{2\%}$ subsamples using various inpainting algorithms. The highest PSNR for each image are highlighted with boldface. In this table the PSNR for LDMM and rw-LDMM are computed for images obtained after the $100$th iteration.}
\label{table:psnr-2-percent}
\begin{tabular}{|l|c|c|c|c|c|c|c|}
\hline
\textbf{Method}		&	ALOHA	&	\multicolumn{2}{c|}{LDMM}	&	\multicolumn{2}{c|}{rw-LDMM SVD}	&	\multicolumn{2}{c|}{rw-LDMM DCT}	\\	\hline
\textbf{symm. GL}	&	NA	&	No	&	Yes	&	No	&	Yes	&	No	&	Yes	 \\ \hline
{\sc Barbara}		&	18.68	&	18.91	&	19.43	&	20.05	&	16.10	&	\textbf{20.12} &	16.06	 \\ \hline
{\sc Boat}		&	18.83	&	19.09	&	19.35	&	\textbf{19.78}	&	16.76	&	19.76	&	16.78	 \\ \hline
{\sc Checkerboard}	&	\textbf{7.81}	&	6.24	&	5.44	&	6.34	&	6.75	&	6.37	&	6.45	 \\ \hline
{\sc Couple}		&	18.37	&	18.61	&	18.68	&	19.74	&	17.01	&	\textbf{19.77}	&	16.75	 \\ \hline
{\sc Fingerprint}	&	\textbf{15.76}	&	14.82	&	13.68	&	15.20	&	14.04	&	14.93	&	14.10	 \\ \hline
{\sc Hill}		&	22.71	&	23.31	&	23.85	&	25.14	&	18.03	&	\textbf{25.16}	&	17.89	 \\ \hline
{\sc House}		&	20.62	&	21.12	&	\textbf{21.80}	&	21.66	&	15.47	&	21.65	&	15.46	 \\ \hline
{\sc Man}		&	18.62	&	19.07	&	19.82	&	20.48	&	18.91	&	\textbf{20.55}	&	19.00	 \\ \hline
{\sc Swirl}		&	14.55	&	14.83	&	13.72	&	15.05	&	14.27	&	\textbf{15.14}	&	14.52	 \\ \hline
\end{tabular}
\end{table}

\begin{table}[htbp]
\centering
\caption{PSNR of images reconstructed from $\mathbf{5\%}$ subsamples using various inpainting algorithms. The highest PSNR for each image are highlighted with boldface. In this table the PSNR for LDMM and rw-LDMM are computed for images obtained after the $100$th iteration.}
\label{table:psnr-5-percent}
\begin{tabular}{|l|c|c|c|c|c|c|c|}
\hline
\textbf{Method}		&	ALOHA	&	\multicolumn{2}{c|}{LDMM}	&	\multicolumn{2}{c|}{rw-LDMM SVD}	&	\multicolumn{2}{c|}{rw-LDMM DCT}	\\	\hline
\textbf{symm. GL}	&	NA	&	No	&	Yes	&	No	&	Yes	&	No	&	Yes	 \\ \hline
{\sc Barbara}		&	\textbf{22.86}	&	21.70	&	21.85	&	22.01	&	21.47	&	21.97	&	21.40	 \\ \hline
{\sc Boat}		&	20.90	&	21.20	&	21.39	&	\textbf{21.81}	&	21.23	&	21.80	&	21.22	 \\ \hline
{\sc Checkerboard}	&	\textbf{11.57}	&	5.71	&	7.59	&	8.23	&	8.78	&	8.12	&	8.78	 \\ \hline
{\sc Couple}		&	20.99	&	21.59	&	21.44	&	\textbf{22.10}	&	20.32	&	22.08	&	20.36	 \\ \hline
{\sc Fingerprint}	&	\textbf{20.59}	&	16.50	&	19.72	&	20.57	&	20.58	&	20.23	&	20.45	 \\ \hline
{\sc Hill}		&	26.34	&	27.05	&	26.70	&	\textbf{27.57}	&	26.63	&	27.52	&	26.04	 \\ \hline
{\sc House}		&	24.50	&	25.11	&	25.07	&	\textbf{26.03}	&	21.85	&	\textbf{26.03}	&	22.07	 \\ \hline
{\sc Man}		&	21.07	&	22.07	&	21.66	&	\textbf{22.48}	&	22.03	&	\textbf{22.48}	&	21.98	 \\ \hline
{\sc Swirl}		&	17.78	&	15.95	&	16.07	&	\textbf{16.52}	&	16.42	&	16.51	&	16.52	 \\ \hline
\end{tabular}
\end{table}

\begin{table}[htbp]
\centering
\caption{PSNR of images reconstructed from $\mathbf{10\%}$ subsamples using various inpainting algorithms. The highest PSNR for each image are highlighted with boldface. In this table the PSNR for LDMM and rw-LDMM are computed for images obtained after the $100$th iteration.}
\label{table:psnr-10-percent}
\begin{tabular}{|l|c|c|c|c|c|c|c|}
\hline
\textbf{Method}		&	ALOHA	&	\multicolumn{2}{c|}{LDMM}	&	\multicolumn{2}{c|}{rw-LDMM SVD}	&	\multicolumn{2}{c|}{rw-LDMM DCT}	\\	\hline
\textbf{symm. GL}	&	NA	&	No	&	Yes	&	No	&	Yes	&	No	&	Yes	 \\ \hline
{\sc Barbara}		&	\textbf{26.25}	&	24.75	&	24.78	&	25.61	&	25.34	&	25.71	&	25.59	 \\ \hline
{\sc Boat}		&	\textbf{23.75}	&	23.21	&	23.08	&	23.66	&	23.33	&	23.63	&	23.31	 \\ \hline
{\sc Checkerboard}	&	13.67	&	12.18	&	12.37	&	13.74	&	14.38	&	13.75	&	\textbf{14.29}	 \\ \hline
{\sc Couple}		&	23.40	&	23.78	&	23.04	&	24.24	&	23.62	&	\textbf{24.27}	&	23.65	 \\ \hline
{\sc Fingerprint}	&	\textbf{23.26}	&	21.93	&	21.79	&	22.60	&	22.30	&	22.52	&	22.20	 \\ \hline
{\sc Hill}		&	28.62	&	28.71	&	28.26	&	29.06	&	28.54	&	\textbf{29.07}	&	28.37	 \\ \hline
{\sc House}		&	28.90	&	29.30	&	28.19	&	29.83	&	28.83	&	\textbf{29.93}	&	28.93	 \\ \hline
{\sc Man}		&	23.59	&	24.21	&	23.84	&	24.61	&	23.98	&	\textbf{24.62}	&	23.87	 \\ \hline
{\sc Swirl}		&	21.24	&	19.02	&	19.31	&	\textbf{21.60}	&	20.83	&	21.00	&	21.07	 \\ \hline
\end{tabular}
\end{table}

\begin{table}[htbp]
\centering
\caption{PSNR of images reconstructed from $\mathbf{15\%}$ subsamples using various inpainting algorithms. The highest PSNR for each image are highlighted with boldface. In this table the PSNR for LDMM and rw-LDMM are computed for images obtained after the $100$th iteration.}
\label{table:psnr-15-percent}
\begin{tabular}{|l|c|c|c|c|c|c|c|}
\hline
\textbf{Method}		&	ALOHA	&	\multicolumn{2}{c|}{LDMM}	&	\multicolumn{2}{c|}{rw-LDMM SVD}	&	\multicolumn{2}{c|}{rw-LDMM DCT}	\\	\hline
\textbf{symm. GL}	&	NA	&	No	&	Yes	&	No	&	Yes	&	No	&	Yes	 \\ \hline
{\sc Barbara}		&	\textbf{28.41}	&	26.37	&	26.38	&	26.88	&	26.46	&	26.88	&	26.31	 \\ \hline
{\sc Boat}		&	\textbf{25.11}	&	24.63	&	24.10	&	24.83	&	24.57	&	24.86	&	24.62	 \\ \hline
{\sc Checkerboard}	&	14.88	&	16.03	&	\textbf{16.74}	&	16.23	&	16.64	&	16.31	&	16.39	 \\ \hline
{\sc Couple}		&	25.12	&	25.51	&	24.49	&	\textbf{25.65}	&	25.25	&	25.59	&	25.22	 \\ \hline
{\sc Fingerprint}	&	\textbf{24.87}	&	23.27	&	23.14	&	23.57	&	23.35	&	23.52	&	23.69	 \\ \hline
{\sc Hill}		&	30.09	&	29.98	&	29.06	&	30.14	&	29.26	&	\textbf{30.15}	&	29.24	 \\ \hline
{\sc House}		&	31.07	&	31.29	&	30.53	&	31.38	&	30.41	&	\textbf{31.39}	&	30.50	 \\ \hline
{\sc Man}		&	24.75	&	25.72	&	24.77	&	25.84	&	25.03	&	\textbf{25.92}	&	25.18	 \\ \hline
{\sc Swirl}		&	23.62	&	23.99	&	23.73	&	24.56	&	\textbf{25.12}	&	24.85	&	24.99	 \\ \hline
\end{tabular}
\end{table}

\begin{table}[htbp]
\centering
\caption{PSNR of images reconstructed from $\mathbf{20\%}$ subsamples using various inpainting algorithms. The highest PSNR for each image are highlighted with boldface. In this table the PSNR for LDMM and rw-LDMM are computed for images obtained after the $100$th iteration.}
\label{table:psnr-20-percent}
\begin{tabular}{|l|c|c|c|c|c|c|c|}
\hline
\textbf{Method}		&	ALOHA	&	\multicolumn{2}{c|}{LDMM}	&	\multicolumn{2}{c|}{rw-LDMM SVD}	&	\multicolumn{2}{c|}{rw-LDMM DCT}	\\	\hline
\textbf{symm. GL}	&	NA	&	No	&	Yes	&	No	&	Yes	&	No	&	Yes	 \\ \hline
{\sc Barbara}		&	\textbf{30.37}	&	28.75	&	27.89	&	29.32	&	28.44	&	29.36	&	28.27	 \\ \hline
{\sc Boat}		&	\textbf{26.71}	&	25.94	&	25.57	&	26.33	&	25.85	&	26.32	&	25.81	 \\ \hline
{\sc Checkerboard}	&	16.40	&	18.82	&	18.00	&	\textbf{19.03}	&	18.29	&	18.97	&	18.29	 \\ \hline
{\sc Couple}		&	26.83	&	26.90	&	25.92	&	27.04	&	26.05	&	\textbf{27.09}	&	26.23	 \\ \hline
{\sc Fingerprint}	&	\textbf{26.98}	&	24.76	&	24.23	&	24.87	&	24.82	&	25.00	&	24.75	 \\ \hline
{\sc Hill}		&	\textbf{31.50}	&	31.12	&	30.06	&	31.09	&	29.92	&	31.18	&	30.36	 \\ \hline
{\sc House}		&	\textbf{33.08}	&	32.34	&	31.13	&	32.99	&	30.90	&	33.02	&	30.33	 \\ \hline
{\sc Man}		&	26.22	&	27.03	&	26.16	&	\textbf{27.22}	&	26.17	&	\textbf{27.22}	&	26.14	 \\ \hline
{\sc Swirl}		&	25.34	&	27.35	&	26.60	&	\textbf{28.52}	&	27.62	&	28.40	&	27.65	 \\ \hline
\end{tabular}
\end{table}

\begin{table}[htbp]
\centering
\caption{PSNR of images reconstructed from $\mathbf{50\%}$ subsamples using various inpainting algorithms. The highest PSNR for each image are highlighted with boldface. In this table the PSNR for LDMM and rw-LDMM are computed for images obtained after the $100$th iteration.}
\label{table:psnr-50-percent}
\begin{tabular}{|l|c|c|c|c|c|c|c|}
\hline
\textbf{Method}		&	ALOHA	&	\multicolumn{2}{c|}{LDMM}	&	\multicolumn{2}{c|}{rw-LDMM SVD}	&	\multicolumn{2}{c|}{rw-LDMM DCT}	\\	\hline
\textbf{symm. GL}	&	NA	&	No	&	Yes	&	No	&	Yes	&	No	&	Yes	 \\ \hline
{\sc Barbara}		&	\textbf{38.13}	&	36.10	&	28.62	&	36.14	&	27.01	&	36.20	&	29.16	 \\ \hline
{\sc Boat}		&	\textbf{33.17}	&	32.78	&	28.59	&	32.77	&	28.22	&	32.75	&	28.93	 \\ \hline
{\sc Checkerboard}	&	22.56	&	\textbf{28.11}	&	27.41	&	27.45	&	26.80	&	27.46	&	26.80	 \\ \hline
{\sc Couple}		&	32.12	&	32.29	&	28.41	&	32.19	&	28.51	&	\textbf{32.32}	&	28.84	 \\ \hline
{\sc Fingerprint}	&	\textbf{34.71}	&	27.84	&	27.10	&	27.94	&	27.04	&	27.75	&	26.90	 \\ \hline
{\sc Hill}		&	\textbf{35.70}	&	35.06	&	32.95	&	35.14	&	32.94	&	35.18	&	32.78	 \\ \hline
{\sc House}		&	\textbf{39.94}	&	38.69	&	35.96	&	38.67	&	34.57	&	38.72	&	34.74	 \\ \hline
{\sc Man}		&	30.98	&	32.30	&	29.61	&	\textbf{32.34}	&	29.87	&	32.31	&	29.51	 \\ \hline
{\sc Swirl}		&	33.77	&	\textbf{35.39}	&	32.75	&	35.35	&	32.85	&	\textbf{35.39}	&	32.71	 \\ \hline
\end{tabular}
\end{table}

\newpage

\subsection{Maximum PSNR in LDMM and rw-LDMM}
\label{sec:maximum-psnr-ldmm}

In \cref{table:max-psnr-2-percent} through \cref{table:max-psnr-50-percent}, the PSNR in the columns of LDMM and reweighted LDMM are the maximum PSNR that occurred before the $100$th iteration. Since in practice the ground truth image is not given, no criterion is readily available for us to terminate the algorithm before reaching the $100$th iteration, nor is there a rule for picking one reconstructed image from all reconstructions resulted from the first $100$ iterations; the comparisons presented in this subsection, rather than performance evaluations of practical inpainting algorithms as in Section~\ref{sec:final-psnr-ldmm}, are only proof-of-concepts to illustrate that important information are captured in LDMM and reweighted LDMM that could potentially be used to further improve the performance of these inpainting algorithms.

\begin{table}[htbp]
\centering
\caption{PSNR of images reconstructed from $\mathbf{2\%}$ subsamples using various inpainting algorithms. The highest PSNR for each image are highlighted with boldface. In this table the PSNR for LDMM and rw-LDMM are the maximum PSNR within the first $100$th iteration.}
\label{table:max-psnr-2-percent}
\begin{tabular}{|l|c|c|c|c|c|c|c|}
\hline
\textbf{Method}		&	ALOHA	&	\multicolumn{2}{c|}{LDMM}	&	\multicolumn{2}{c|}{rw-LDMM SVD}	&	\multicolumn{2}{c|}{rw-LDMM DCT}	\\	\hline
\textbf{symm. GL}	&	NA	&	No	&	Yes	&	No	&	Yes	&	No	&	Yes	 \\ \hline
{\sc Barbara}		&	18.68	&	19.50	&	19.43	&	20.05	&	16.10	&	\textbf{20.12}	&	16.06	 \\ \hline
{\sc Boat}		&	18.83	&	19.09	&	19.43	&	\textbf{19.79}	&	16.76	&	\textbf{19.79}	&	16.78	 \\ \hline
{\sc Checkerboard}	&	7.81	&	7.69	&	7.69	&	\textbf{7.84}	&	7.79	&	\textbf{7.84}	&	7.79	 \\ \hline
{\sc Couple}		&	18.37	&	19.16	&	18.68	&	\textbf{19.78}	&	17.01	&	\textbf{19.78}	&	16.75	 \\ \hline
{\sc Fingerprint}	&	\textbf{15.76}	&	14.87	&	14.76	&	15.20	&	14.04	&	15.03	&	14.10	 \\ \hline
{\sc Hill}		&	22.71	&	23.36	&	23.85	&	25.14	&	18.03	&	\textbf{25.16}	&	17.89	 \\ \hline
{\sc House}		&	20.62	&	21.12	&	\textbf{21.80}	&	21.66	&	15.47	&	21.66	&	15.46	 \\ \hline
{\sc Man}		&	18.62	&	19.69	&	19.85	&	20.60	&	18.91	&	\textbf{20.61}	&	19.00	 \\ \hline
{\sc Swirl}		&	14.55	&	14.91	&	13.72	&	15.08	&	14.27	&	\textbf{15.14}	&	14.52	 \\ \hline
\end{tabular}
\end{table}

\begin{table}[htbp]
\centering
\caption{PSNR of images reconstructed from $\mathbf{5\%}$ subsamples using various inpainting algorithms. The highest PSNR for each image are highlighted with boldface. In this table the PSNR for LDMM and rw-LDMM are the maximum PSNR within the first $100$th iteration.}
\label{table:max-psnr-5-percent}
\begin{tabular}{|l|c|c|c|c|c|c|c|}
\hline
\textbf{Method}		&	ALOHA	&	\multicolumn{2}{c|}{LDMM}	&	\multicolumn{2}{c|}{rw-LDMM SVD}	&	\multicolumn{2}{c|}{rw-LDMM DCT}	\\	\hline
\textbf{symm. GL}	&	NA	&	No	&	Yes	&	No	&	Yes	&	No	&	Yes	 \\ \hline
{\sc Barbara}		&	\textbf{22.86}	&	21.70	&	22.14	&	22.04	&	21.47	&	22.05	&	21.44	 \\ \hline
{\sc Boat}		&	20.90	&	21.40	&	21.62	&	\textbf{21.91}	&	21.25	&	21.87	&	21.30	 \\ \hline
{\sc Checkerboard}	&	\textbf{11.57}	&	7.88	&	7.88	&	8.88	&	9.11	&	8.87	&	9.12	 \\ \hline
{\sc Couple}		&	20.99	&	21.75	&	21.71	&	22.18	&	20.32	&	\textbf{22.19}	&	20.36	 \\ \hline
{\sc Fingerprint}	&	20.59	&	17.27	&	20.16	&	20.57	&	\textbf{20.92}	&	20.36	&	20.88	 \\ \hline
{\sc Hill}		&	26.34	&	27.08	&	26.93	&	\textbf{27.71}	&	26.74	&	27.61	&	26.15	 \\ \hline
{\sc House}		&	24.50	&	25.11	&	25.55	&	\textbf{26.09}	&	21.85	&	26.03	&	22.07	 \\ \hline
{\sc Man}		&	21.07	&	22.18	&	21.90	&	22.58	&	22.21	&	\textbf{22.63}	&	22.24	 \\ \hline
{\sc Swirl}		&	\textbf{17.78}	&	16.32	&	16.07	&	16.65	&	16.60	&	16.58	&	16.62	 \\ \hline
\end{tabular}
\end{table}

\begin{table}[htbp]
\centering
\caption{PSNR of images reconstructed from $\mathbf{10\%}$ subsamples using various inpainting algorithms. The highest PSNR for each image are highlighted with boldface. In this table the PSNR for LDMM and rw-LDMM are the maximum PSNR within the first $100$th iteration.}
\label{table:max-psnr-10-percent}
\begin{tabular}{|l|c|c|c|c|c|c|c|}
\hline
\textbf{Method}		&	ALOHA	&	\multicolumn{2}{c|}{LDMM}	&	\multicolumn{2}{c|}{rw-LDMM SVD}	&	\multicolumn{2}{c|}{rw-LDMM DCT}	\\	\hline
\textbf{symm. GL}	&	NA	&	No	&	Yes	&	No	&	Yes	&	No	&	Yes	 \\ \hline
{\sc Barbara}		&	\textbf{26.25}	&	24.75	&	25.04	&	25.61	&	25.77	&	25.71	&	25.83	 \\ \hline
{\sc Boat}		&	\textbf{23.75}	&	23.31	&	23.40	&	23.66	&	23.49	&	23.65	&	23.55	 \\ \hline
{\sc Checkerboard}	&	13.67	&	12.61	&	12.41	&	14.07	&	\textbf{14.59}	&	14.08	&	14.49	 \\ \hline
{\sc Couple}		&	23.40	&	24.05	&	24.06	&	24.26	&	23.65	&	\textbf{24.29}	&	23.65	 \\ \hline
{\sc Fingerprint}	&	\textbf{23.26}	&	22.09	&	22.31	&	22.91	&	23.10	&	22.77	&	23.07	 \\ \hline
{\sc Hill}		&	28.62	&	29.01	&	28.87	&	29.39	&	28.85	&	\textbf{29.40}	&	28.93	 \\ \hline
{\sc House}		&	28.90	&	29.63	&	29.10	&	30.09	&	29.46	&	\textbf{30.10}	&	29.45	 \\ \hline
{\sc Man}		&	23.59	&	24.41	&	24.40	&	\textbf{24.66}	&	24.71	&	24.65	&	24.70	 \\ \hline
{\sc Swirl}		&	21.24	&	20.24	&	19.81	&	\textbf{21.60}	&	21.17	&	21.00	&	21.10	 \\ \hline
\end{tabular}
\end{table}

\begin{table}[htbp]
\centering
\caption{PSNR of images reconstructed from $\mathbf{15\%}$ subsamples using various inpainting algorithms. The highest PSNR for each image are highlighted with boldface. In this table the PSNR for LDMM and rw-LDMM are the maximum PSNR within the first $100$th iteration.}
\label{table:max-psnr-15-percent}
\begin{tabular}{|l|c|c|c|c|c|c|c|}
\hline
\textbf{Method}		&	ALOHA	&	\multicolumn{2}{c|}{LDMM}	&	\multicolumn{2}{c|}{rw-LDMM SVD}	&	\multicolumn{2}{c|}{rw-LDMM DCT}	\\	\hline
\textbf{symm. GL}	&	NA	&	No	&	Yes	&	No	&	Yes	&	No	&	Yes	 \\ \hline
{\sc Barbara}		&	\textbf{28.41}	&	26.40	&	27.21	&	26.88	&	27.25	&	26.88	&	27.27	 \\ \hline
{\sc Boat}		&	\textbf{25.11}	&	24.76	&	24.76	&	24.97	&	25.02	&	24.96	&	24.97	 \\ \hline
{\sc Checkerboard}	&	14.88	&	16.37	&	16.92	&	16.96	&	17.16	&	17.00	&	\textbf{17.22}	 \\ \hline
{\sc Couple}		&	25.12	&	25.57	&	25.68	&	25.97	&	25.58	&	\textbf{25.99}	&	25.49	 \\ \hline
{\sc Fingerprint}	&	\textbf{24.87}	&	23.98	&	24.24	&	24.59	&	24.60	&	24.61	&	24.58	 \\ \hline
{\sc Hill}		&	30.09	&	30.34	&	30.16	&	\textbf{30.67}	&	29.79	&	\textbf{30.67}	&	30.02	 \\ \hline
{\sc House}		&	31.07	&	31.33	&	31.67	&	31.61	&	31.85	&	31.61	&	\textbf{31.87}	 \\ \hline
{\sc Man}		&	24.75	&	25.89	&	25.70	&	\textbf{26.11}	&	26.03	&	26.10	&	26.05	 \\ \hline
{\sc Swirl}		&	23.62	&	24.16	&	24.37	&	24.63	&	25.31	&	24.95	&	\textbf{25.34}	 \\ \hline
\end{tabular}
\end{table}

\begin{table}[htbp]
\centering
\caption{PSNR of images reconstructed from $\mathbf{20\%}$ subsamples using various inpainting algorithms. The highest PSNR for each image are highlighted with boldface. In this table the PSNR for LDMM and rw-LDMM are the maximum PSNR within the first $100$th iteration.}
\label{table:max-psnr-20-percent}
\begin{tabular}{|l|c|c|c|c|c|c|c|}
\hline
\textbf{Method}		&	ALOHA	&	\multicolumn{2}{c|}{LDMM}	&	\multicolumn{2}{c|}{rw-LDMM SVD}	&	\multicolumn{2}{c|}{rw-LDMM DCT}	\\	\hline
\textbf{symm. GL}	&	NA	&	No	&	Yes	&	No	&	Yes	&	No	&	Yes	 \\ \hline
{\sc Barbara}		&	\textbf{30.37}	&	28.93	&	29.60	&	29.45	&	29.84	&	29.40	&	29.78	 \\ \hline
{\sc Boat}		&	26.71	&	26.17	&	26.56	&	26.49	&	26.76	&	26.47	&	\textbf{26.79}	 \\ \hline
{\sc Checkerboard}	&	16.40	&	19.94	&	20.49	&	20.58	&	20.73	&	20.55	&	\textbf{20.76}	 \\ \hline
{\sc Couple}		&	26.83	&	27.24	&	27.26	&	\textbf{27.60}	&	26.76	&	\textbf{27.60}	&	26.76	 \\ \hline
{\sc Fingerprint}	&	\textbf{26.98}	&	25.82	&	25.83	&	26.35	&	26.33	&	26.37	&	26.33	 \\ \hline
{\sc Hill}		&	31.50	&	32.03	&	31.84	&	\textbf{32.28}	&	31.45	&	32.25	&	31.58	 \\ \hline
{\sc House}		&	33.08	&	33.20	&	33.41	&	33.73	&	33.73	&	33.69	&	\textbf{33.80}	 \\ \hline
{\sc Man}		&	26.22	&	27.35	&	27.06	&	\textbf{27.56}	&	27.28	&	\textbf{27.56}	&	27.27	 \\ \hline
{\sc Swirl}		&	25.34	&	27.44	&	27.18	&	\textbf{28.66}	&	28.31	&	28.46	&	28.24	 \\ \hline
\end{tabular}
\end{table}

\begin{table}[htbp]
\centering
\caption{PSNR of images reconstructed from $\mathbf{50\%}$ subsamples using various inpainting algorithms. The highest PSNR for each image are highlighted with boldface. In this table the PSNR for LDMM and rw-LDMM are the maximum PSNR within the first $100$th iteration.}
\label{table:max-psnr-50-percent}
\begin{tabular}{|l|c|c|c|c|c|c|c|}
\hline
\textbf{Method}		&	ALOHA	&	\multicolumn{2}{c|}{LDMM}	&	\multicolumn{2}{c|}{rw-LDMM SVD}	&	\multicolumn{2}{c|}{rw-LDMM DCT}	\\	\hline
\textbf{symm. GL}	&	NA	&	No	&	Yes	&	No	&	Yes	&	No	&	Yes	 \\ \hline
{\sc Barbara}		&	\textbf{38.13}	&	37.28	&	37.37	&	37.38	&	37.45	&	37.37	&	37.47	 \\ \hline
{\sc Boat}		&	\textbf{33.17}	&	33.06	&	33.04	&	33.10	&	33.07	&	33.09	&	33.07	 \\ \hline
{\sc Checkerboard}	&	22.56	&	29.11	&	\textbf{29.14}	&	29.13	&	29.13	&	29.11	&	\textbf{29.14}	 \\ \hline
{\sc Couple}		&	32.12	&	32.95	&	32.95	&	\textbf{33.00}	&	32.99	&	32.98	&	32.99	 \\ \hline
{\sc Fingerprint}	&	\textbf{34.71}	&	33.00	&	32.99	&	33.04	&	33.02	&	33.04	&	33.02	 \\ \hline
{\sc Hill}		&	35.70	&	36.61	&	36.59	&	\textbf{36.62}	&	36.61	&	36.61	&	36.61	 \\ \hline
{\sc House}		&	39.94	&	39.81	&	40.17	&	39.88	&	40.23	&	39.86	&	\textbf{40.26}	 \\ \hline
{\sc Man}		&	30.98	&	32.74	&	32.71	&	\textbf{32.77}	&	32.71	&	32.75	&	32.72	 \\ \hline
{\sc Swirl}		&	33.77	&	36.36	&	36.32	&	36.41	&	\textbf{36.42}	&	36.38	&	36.40	 \\ \hline
\end{tabular}
\end{table}

\newpage

\subsection{Comparing Inpainting Results of LDMM, rw-LDMM, and ALOHA}
\label{sec:figures-comparisons}

In this section, we compare the final reconstructed images (after the $100$th iteration) of LDMM, rw-LDMM, and ALOHA. Since all PSNR results are already summarized in Section~\ref{sec:comp-ldmm-rewe}, we only present a subset of the reconstruction results for images {\sc Couple} (\cref{fig:couple_10percent_collage}, \cref{fig:couple_15percent_collage}, \cref{fig:couple_20percent_collage}, \cref{fig:couple_50percent_collage}), {\sc Hill} (\cref{fig:hill_10percent_collage}, \cref{fig:hill_15percent_collage}, \cref{fig:hill_20percent_collage}, \cref{fig:hill_50percent_collage}), and {\sc Swirl} (\cref{fig:swirl_10percent_collage}, \cref{fig:swirl_15percent_collage}, \cref{fig:swirl_20percent_collage}, \cref{fig:swirl_50percent_collage}) under subsample rate $10\%$, $15\%$, $20\%$, and $50\%$, to save space.

\begin{figure}[htbp]
\includegraphics[width = \textwidth]{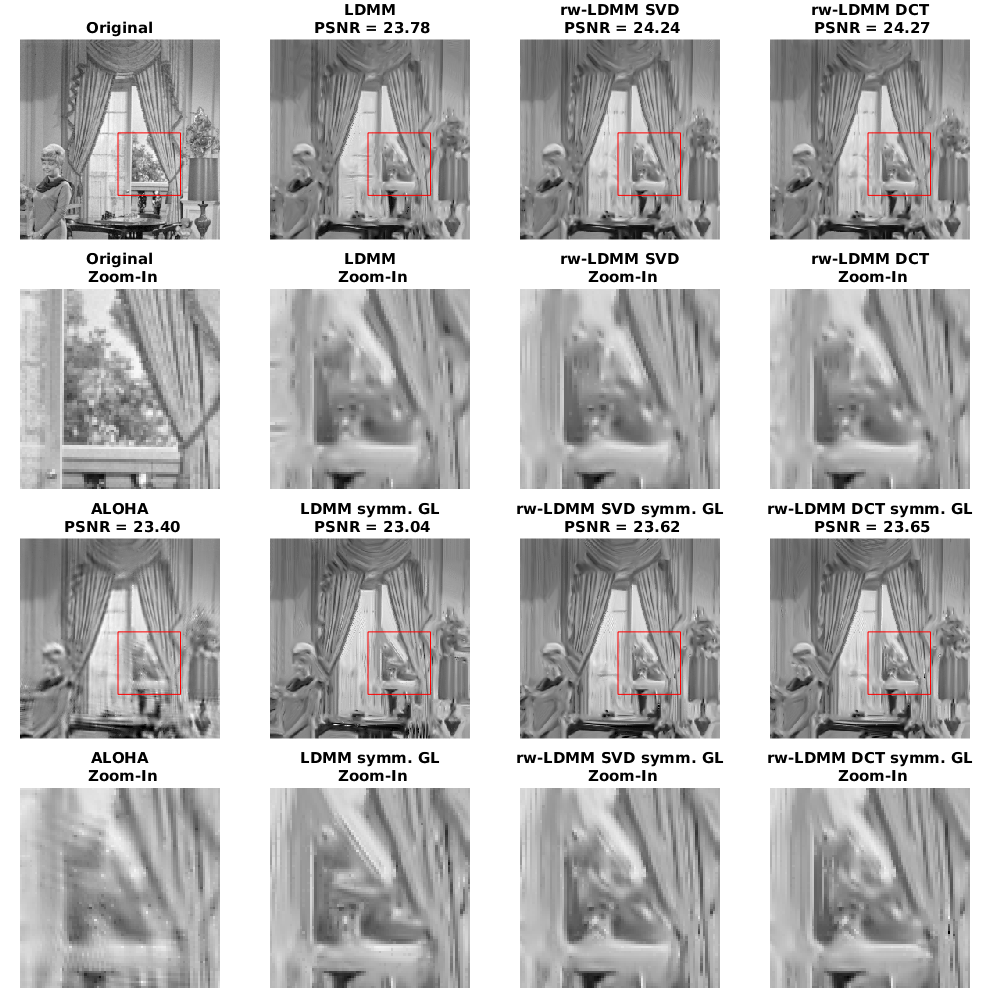}
\caption{Comparing reconstructed $256\times 256$ {\sc couple} images from $10\%$ subsamples using LDMM, rw-LDMM, and ALOHA.}
\label{fig:couple_10percent_collage}
\end{figure}

\begin{figure}[htbp]
\includegraphics[width = \textwidth]{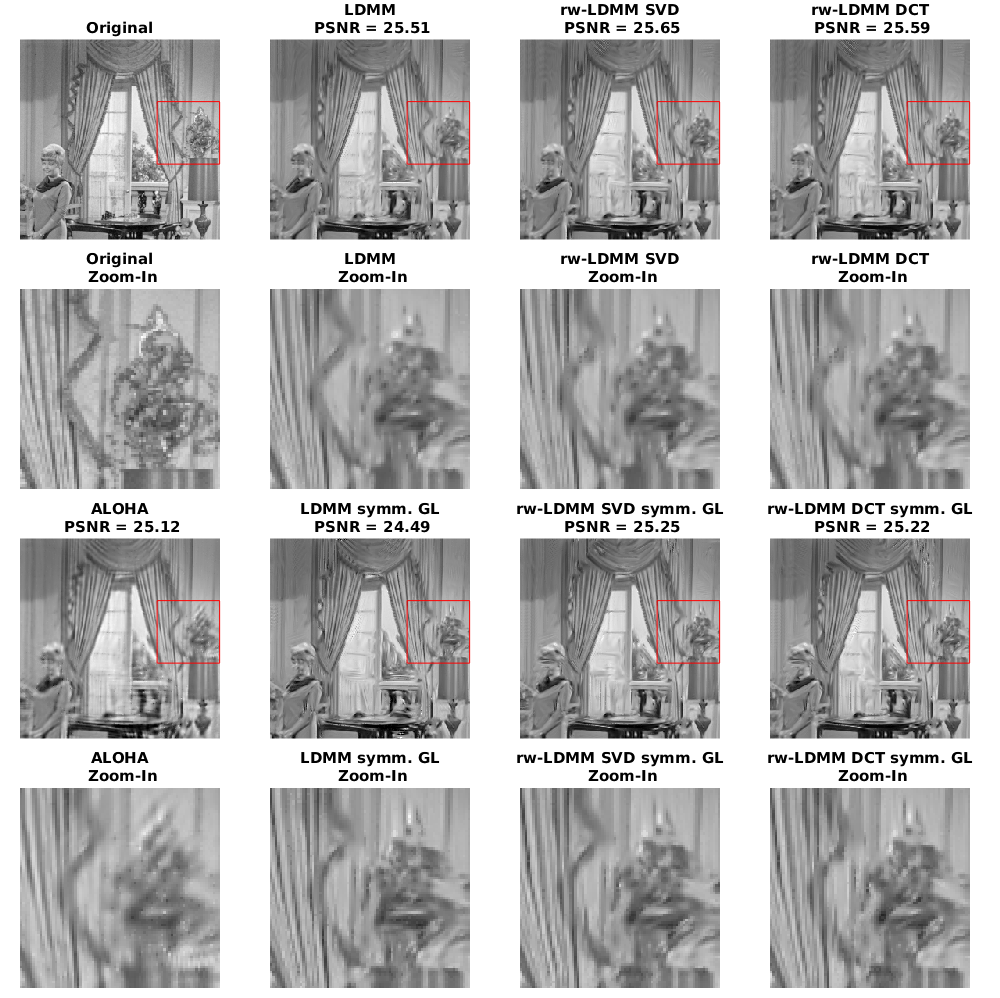}
\caption{Comparing reconstructed $256\times 256$ {\sc couple} images from $15\%$ subsamples using LDMM, rw-LDMM, and ALOHA.}
\label{fig:couple_15percent_collage}
\end{figure}

\begin{figure}[htbp]
\includegraphics[width = \textwidth]{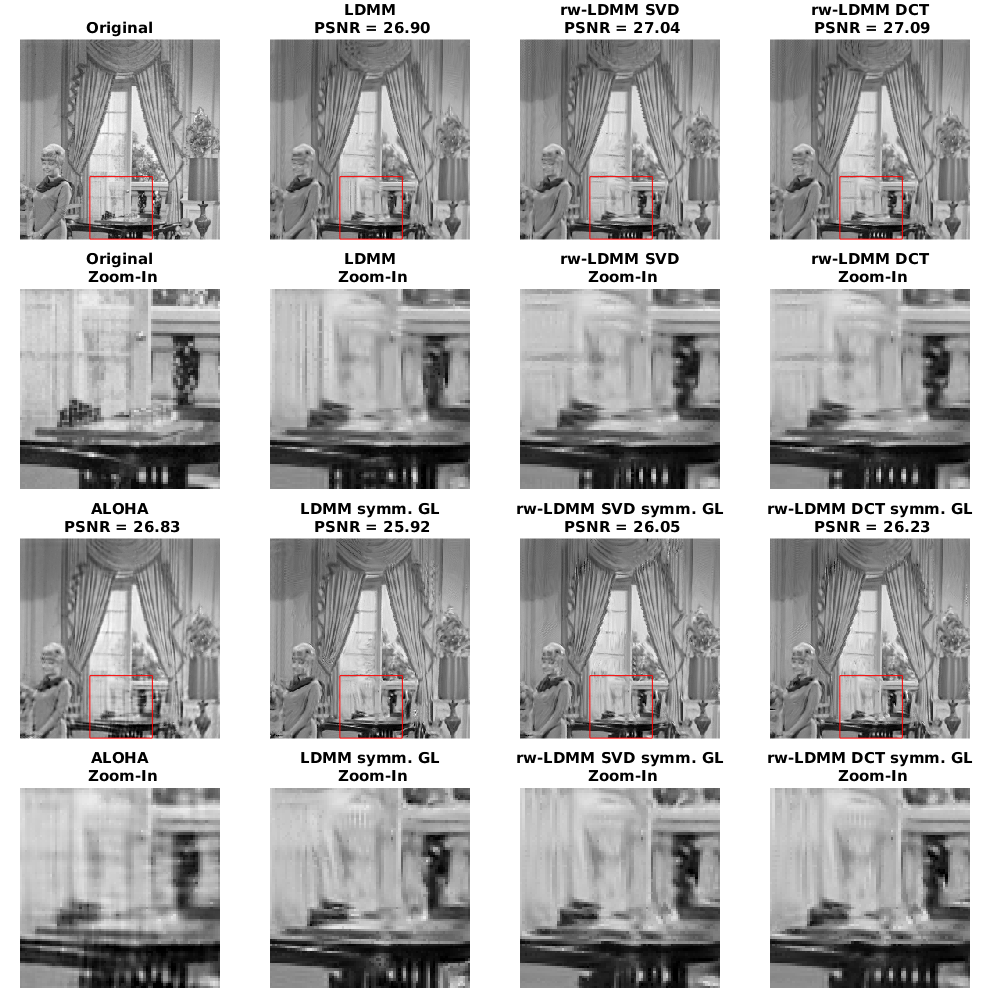}
\caption{Comparing reconstructed $256\times 256$ {\sc couple} images from $20\%$ subsamples using LDMM, rw-LDMM, and ALOHA.}
\label{fig:couple_20percent_collage}
\end{figure}

\begin{figure}[htbp]
\includegraphics[width = \textwidth]{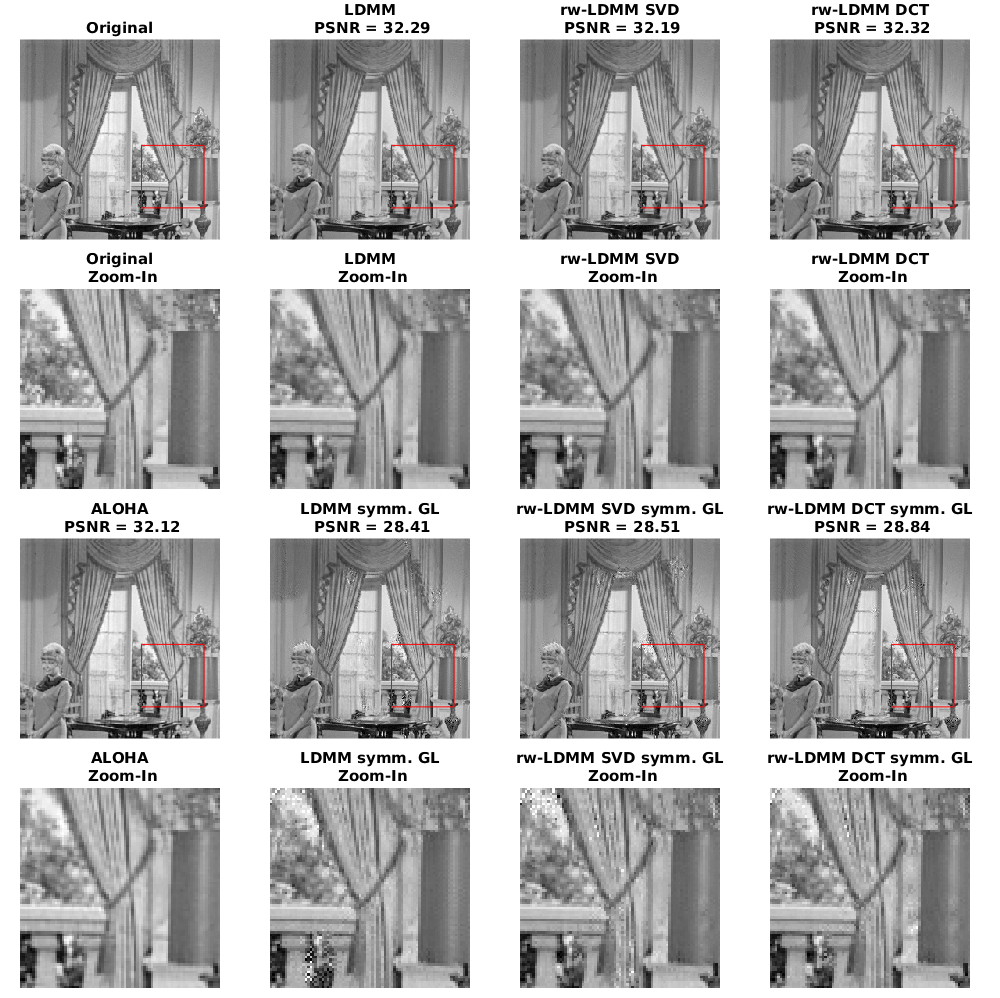}
\caption{Comparing reconstructed $256\times 256$ {\sc couple} images from $50\%$ subsamples using LDMM, rw-LDMM, and ALOHA.}
\label{fig:couple_50percent_collage}
\end{figure}

\begin{figure}[htbp]
\includegraphics[width = \textwidth]{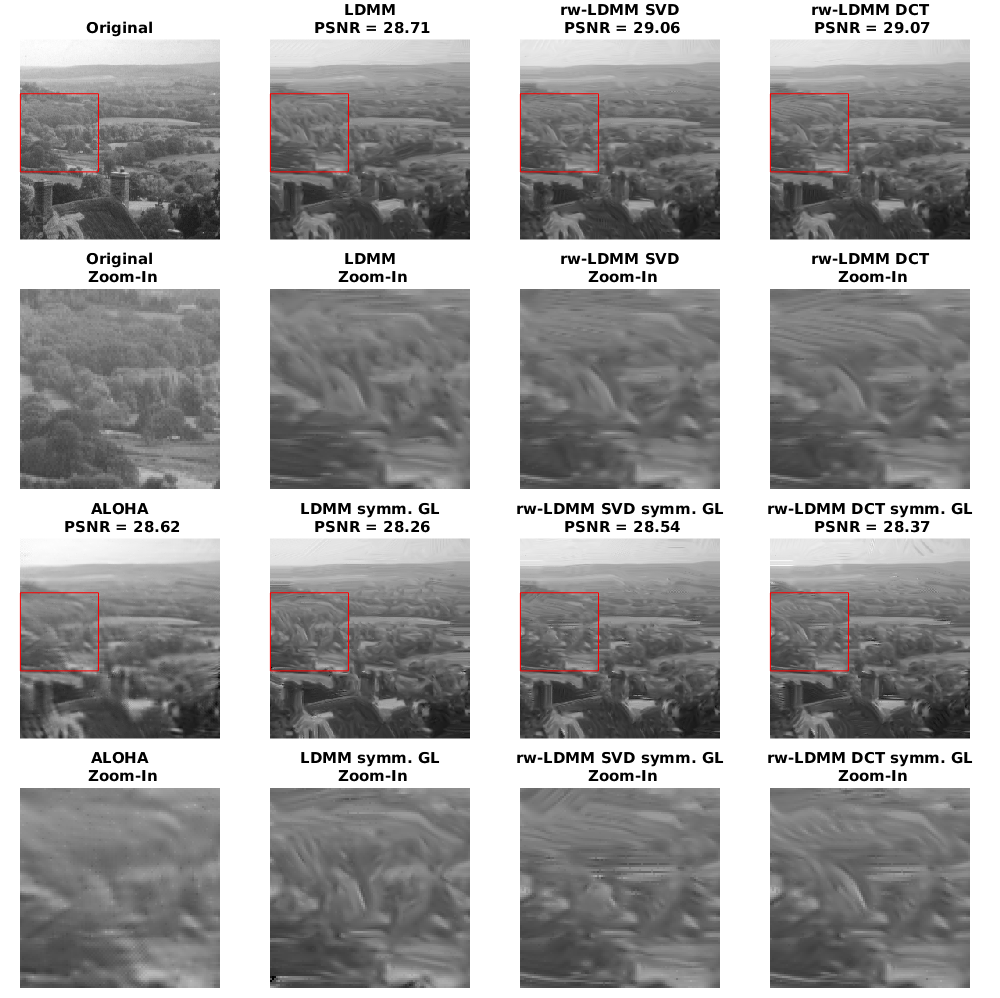}
\caption{Comparing reconstructed $256\times 256$ {\sc hill} images from $10\%$ subsamples using LDMM, rw-LDMM, and ALOHA.}
\label{fig:hill_10percent_collage}
\end{figure}

\begin{figure}[htbp]
\includegraphics[width = \textwidth]{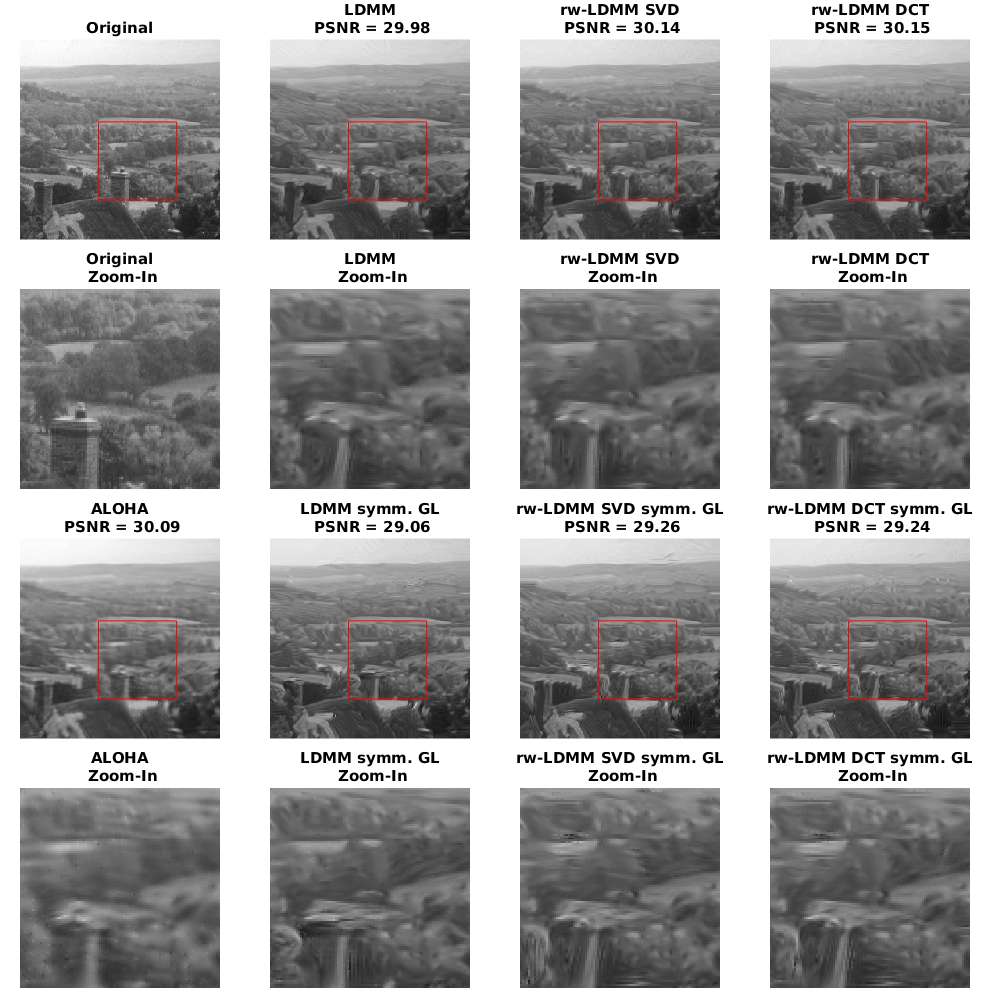}
\caption{Comparing reconstructed $256\times 256$ {\sc hill} images from $15\%$ subsamples using LDMM, rw-LDMM, and ALOHA.}
\label{fig:hill_15percent_collage}
\end{figure}

\begin{figure}[htbp]
\includegraphics[width = \textwidth]{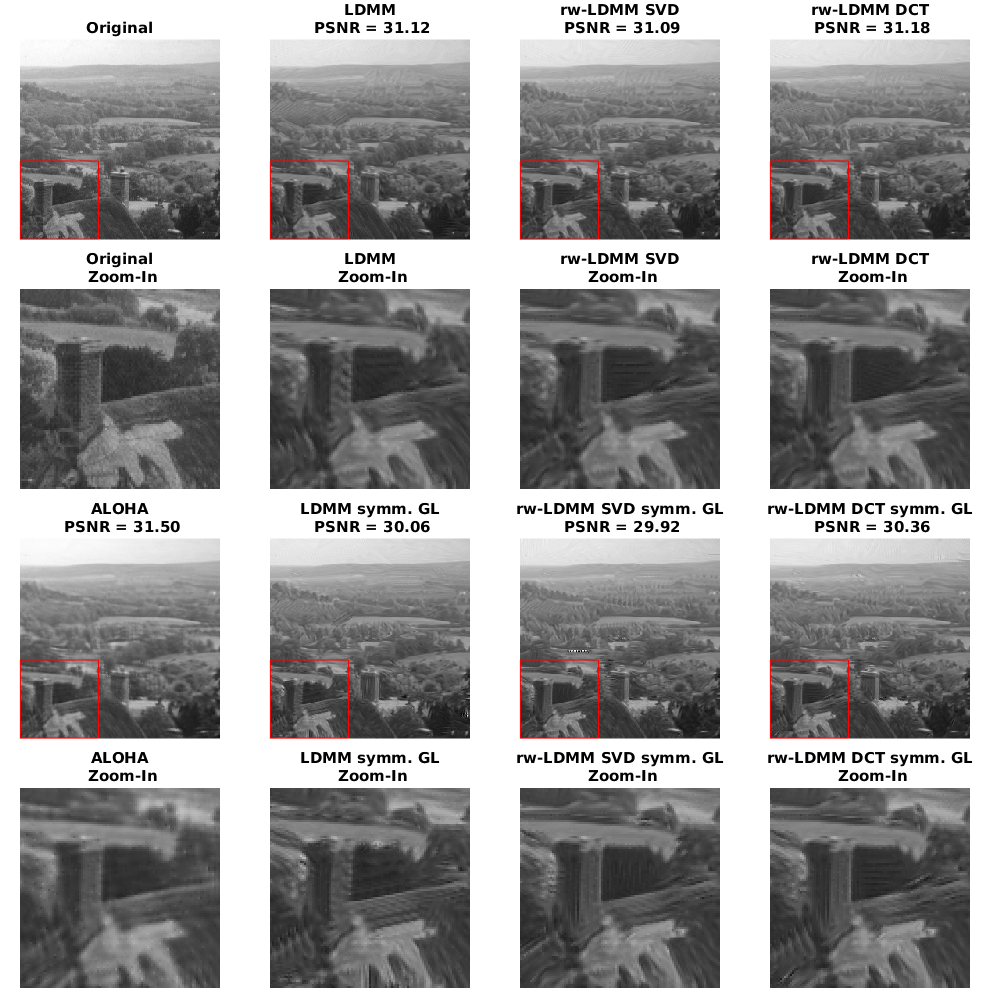}
\caption{Comparing reconstructed $256\times 256$ {\sc hill} images from $20\%$ subsamples using LDMM, rw-LDMM, and ALOHA.}
\label{fig:hill_20percent_collage}
\end{figure}

\begin{figure}[htbp]
\includegraphics[width = \textwidth]{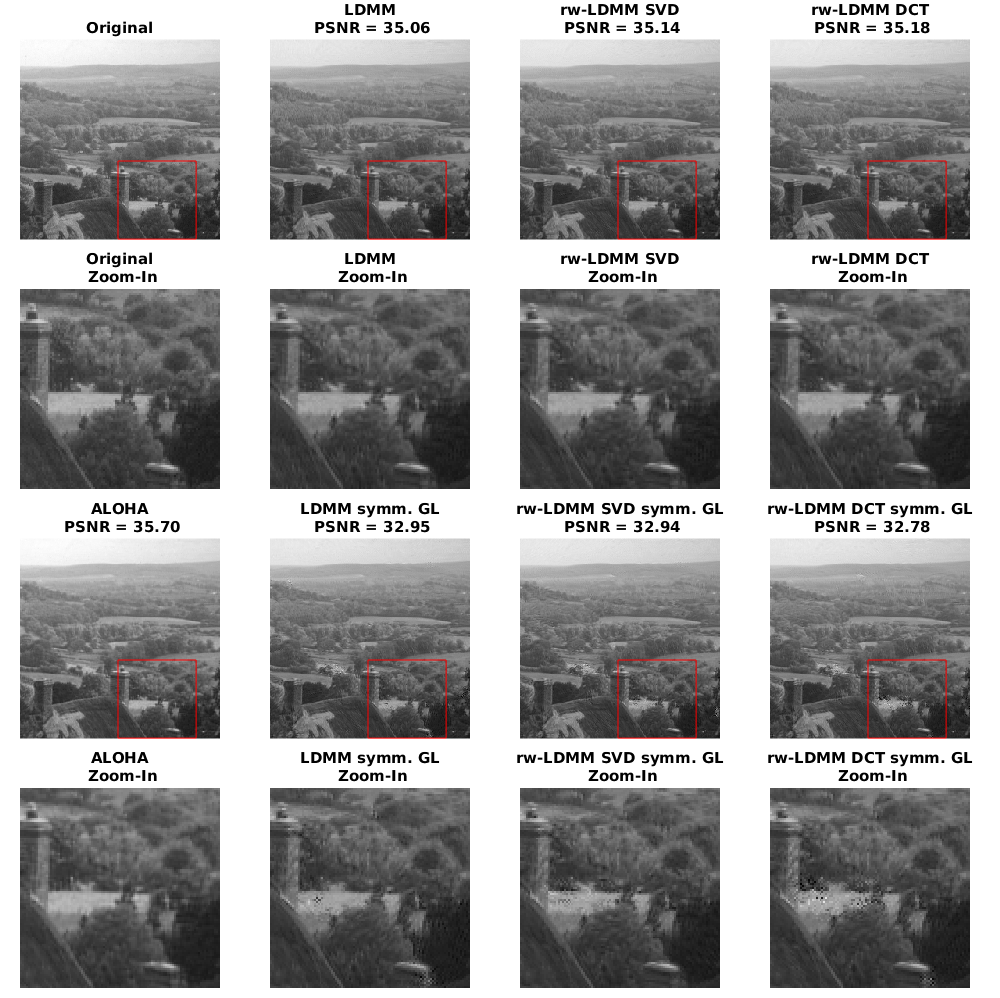}
\caption{Comparing reconstructed $256\times 256$ {\sc hill} images from $50\%$ subsamples using LDMM, rw-LDMM, and ALOHA.}
\label{fig:hill_50percent_collage}
\end{figure}

\begin{figure}[htbp]
\includegraphics[width = \textwidth]{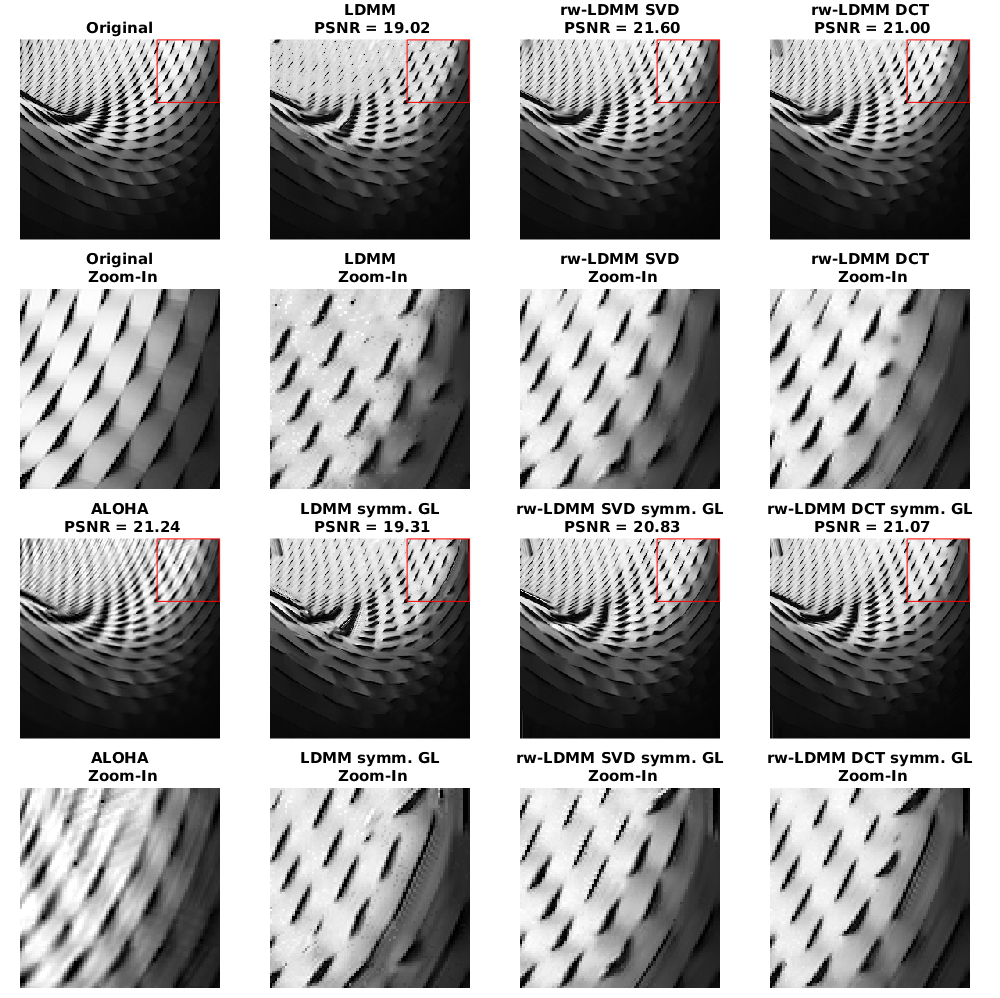}
\caption{Comparing reconstructed $256\times 256$ {\sc Swirl} images from $10\%$ subsamples using LDMM, rw-LDMM, and ALOHA.}
\label{fig:swirl_10percent_collage}
\end{figure}

\begin{figure}[htbp]
\includegraphics[width = \textwidth]{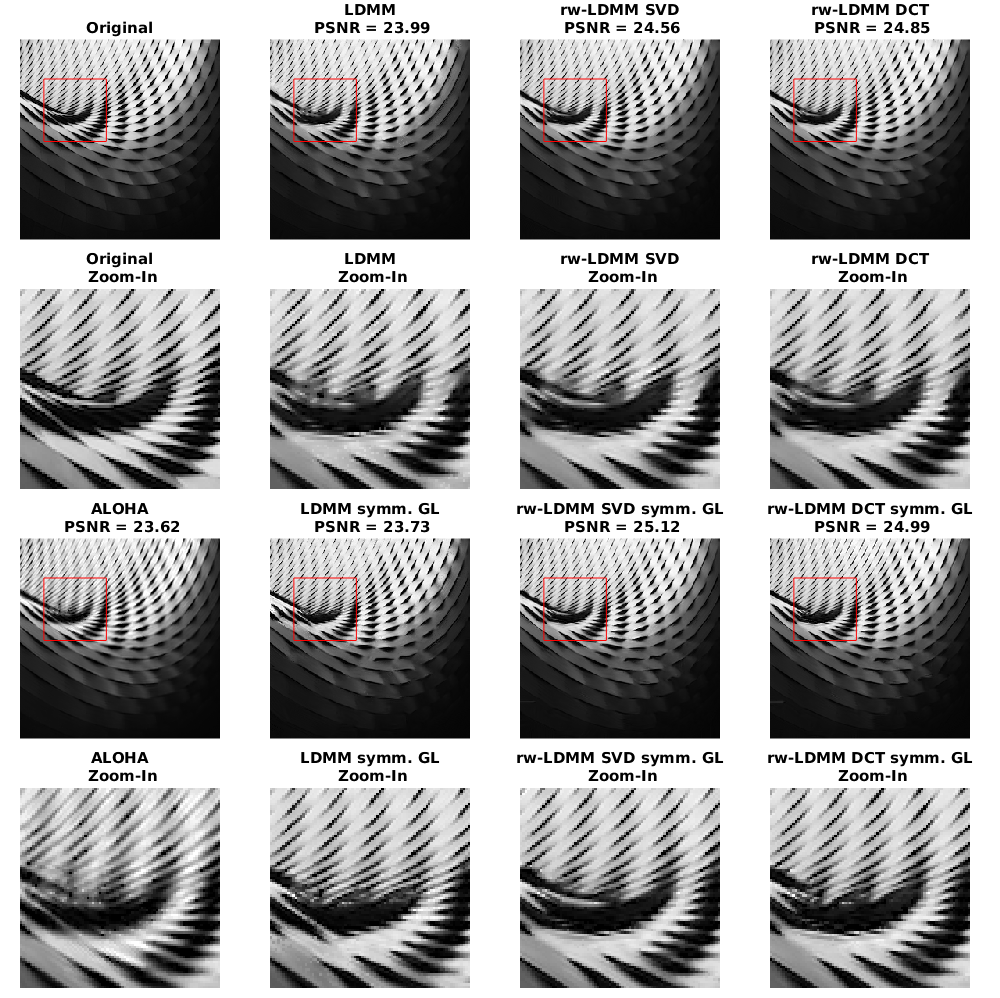}
\caption{Comparing reconstructed $256\times 256$ {\sc Swirl} images from $15\%$ subsamples using LDMM, rw-LDMM, and ALOHA.}
\label{fig:swirl_15percent_collage}
\end{figure}

\begin{figure}[htbp]
\includegraphics[width = \textwidth]{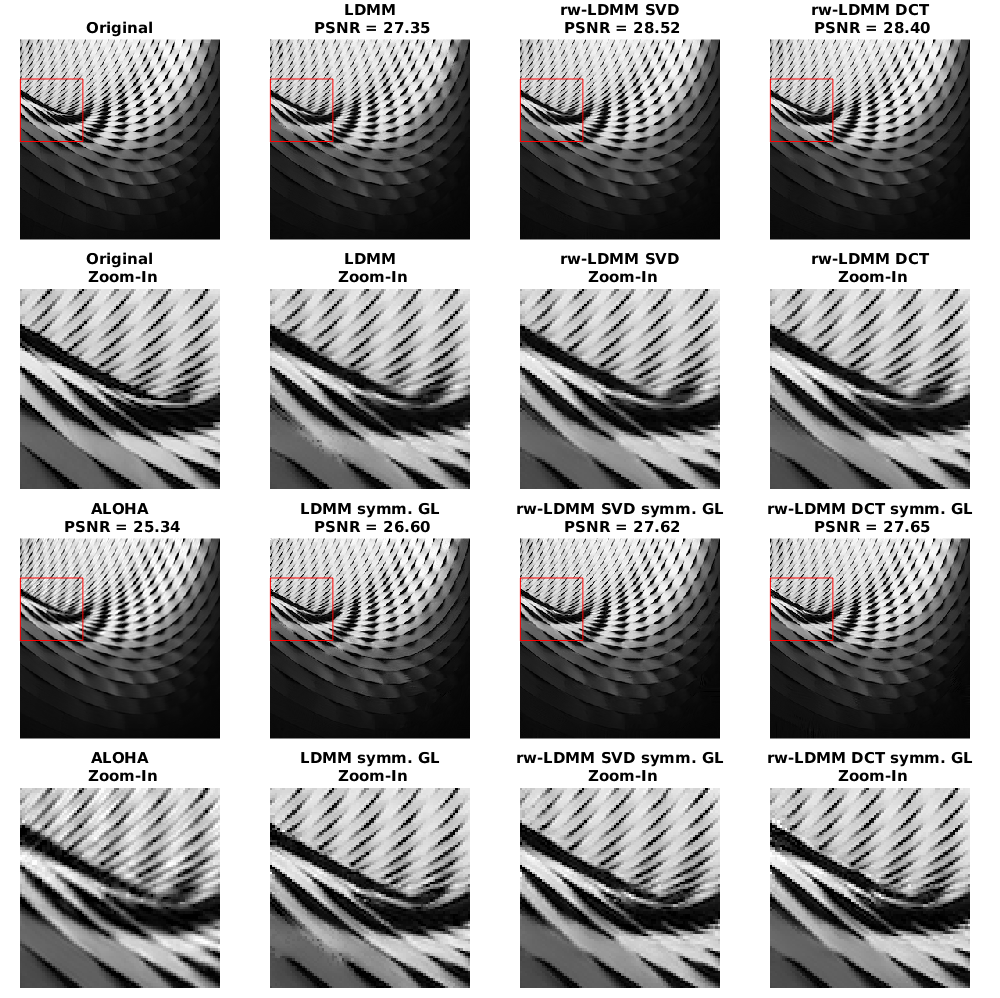}
\caption{Comparing reconstructed $256\times 256$ {\sc Swirl} images from $20\%$ subsamples using LDMM, rw-LDMM, and ALOHA.}
\label{fig:swirl_20percent_collage}
\end{figure}

\begin{figure}[htbp]
\includegraphics[width = \textwidth]{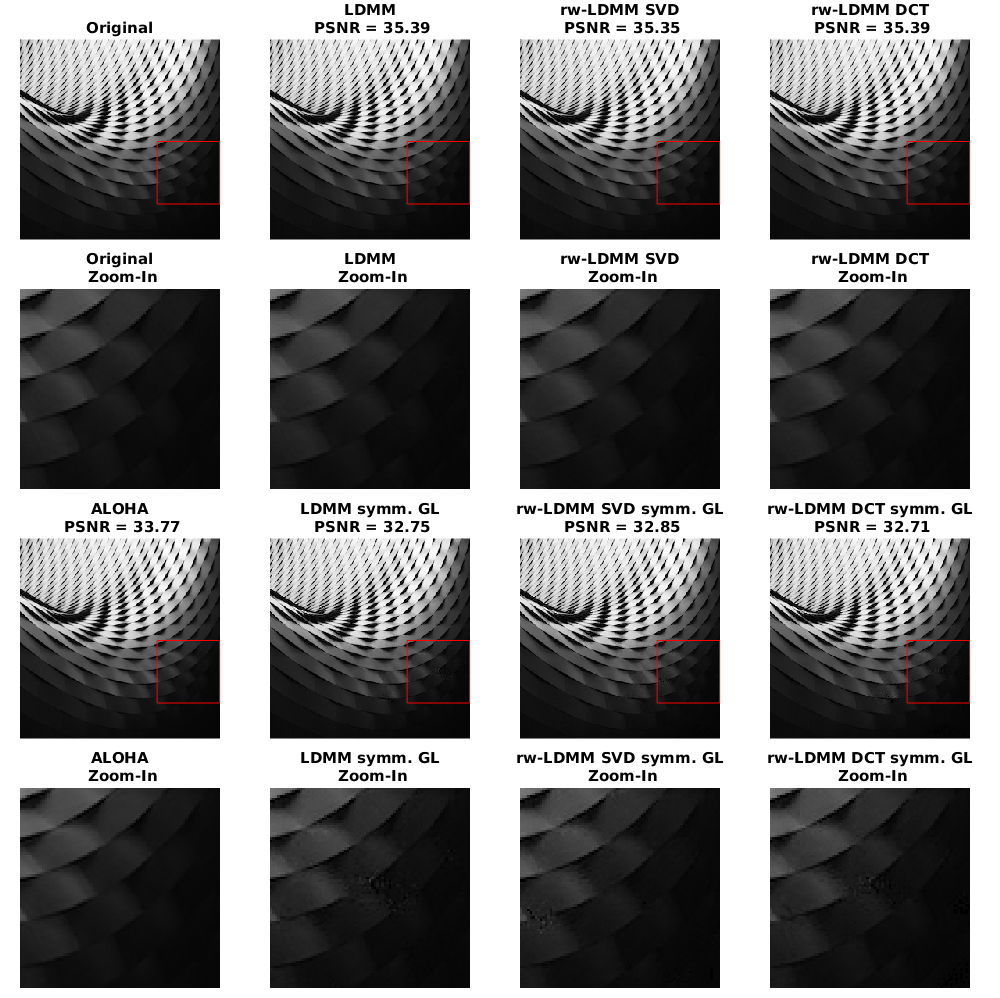}
\caption{Comparing reconstructed $256\times 256$ {\sc Swirl} images from $50\%$ subsamples using LDMM, rw-LDMM, and ALOHA.}
\label{fig:swirl_50percent_collage}
\end{figure}


\section{PSNR vs. Number of Iterations in LDMM and rw-LDMM}
\label{sec:psnr-vs-num_of_iterations}

In this section, we plot the PSNR values as a function of the number of iterations in LDMM and rw-LDMM (SVD and DCT) for $8$ test images with subsample rate $10\%$. Both symmetrized and un-symmetrized graph diffusion Laplacians are used. The goal of this experiment is to validate whether the proposed rw-LDMM algorithm outperforms the original LDMM in \cite{LDMM} consistently for different numbers of total iterations. The numerical results suggest that both rw-LDMM SVD and rw-LDMM DCT often yield higher PSNR than the original LDMM within the first $100$ iterations for the images tested. Note that the PSNR values shown in \cref{fig:psnr_curves} may be slightly different from the results in Section~\ref{sec:comp-ldmm-rewe} because they come from a different random initialization.

\begin{figure}[htbp]
\includegraphics[width = 0.5\textwidth]{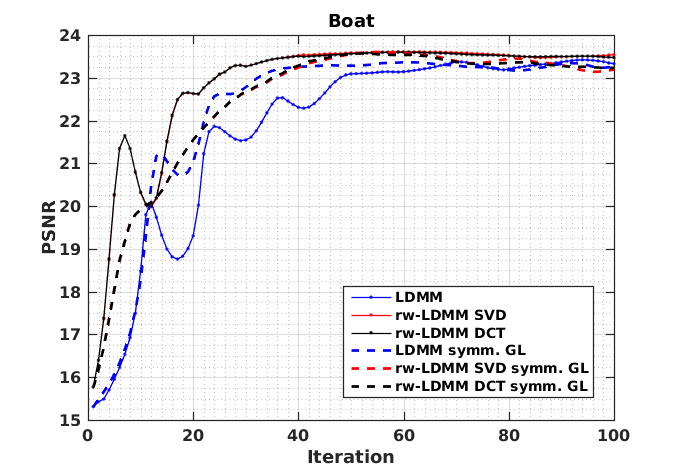}
\includegraphics[width = 0.5\textwidth]{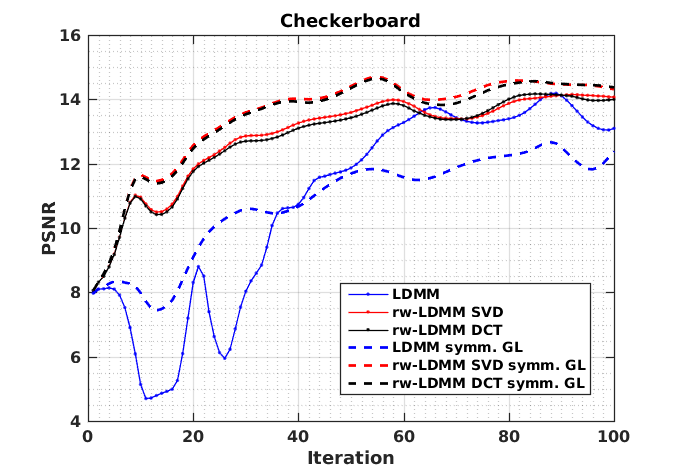}
\includegraphics[width = 0.5\textwidth]{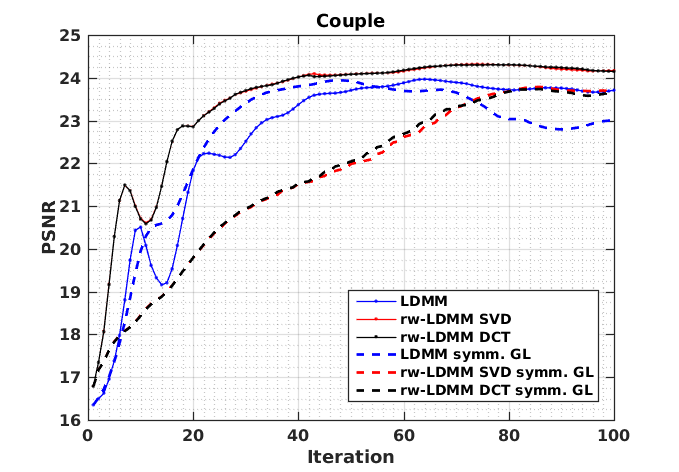}
\includegraphics[width = 0.5\textwidth]{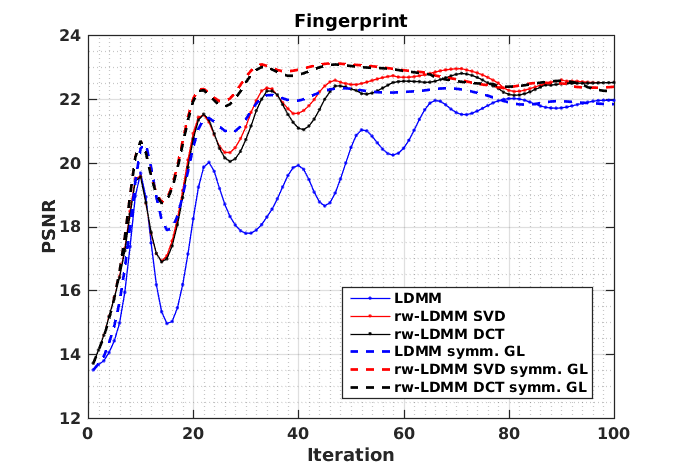}
\includegraphics[width = 0.5\textwidth]{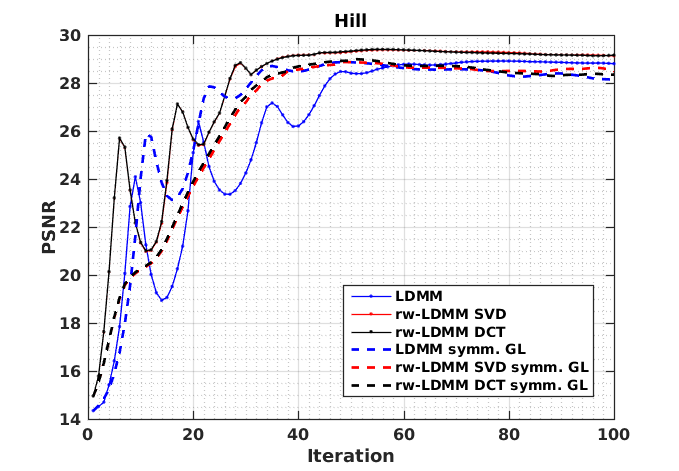}
\includegraphics[width = 0.5\textwidth]{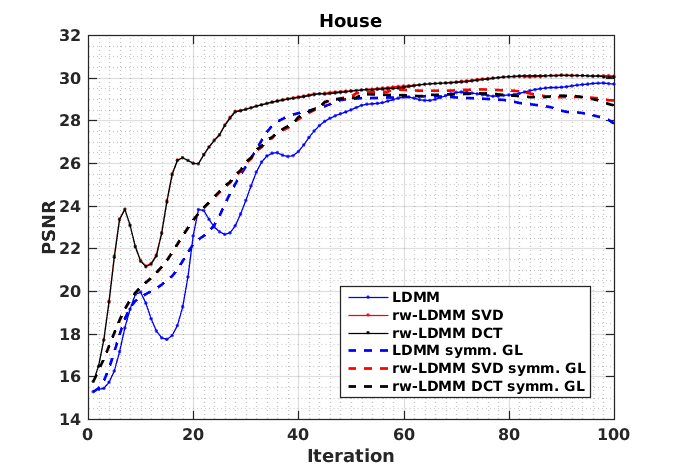}
\includegraphics[width = 0.5\textwidth]{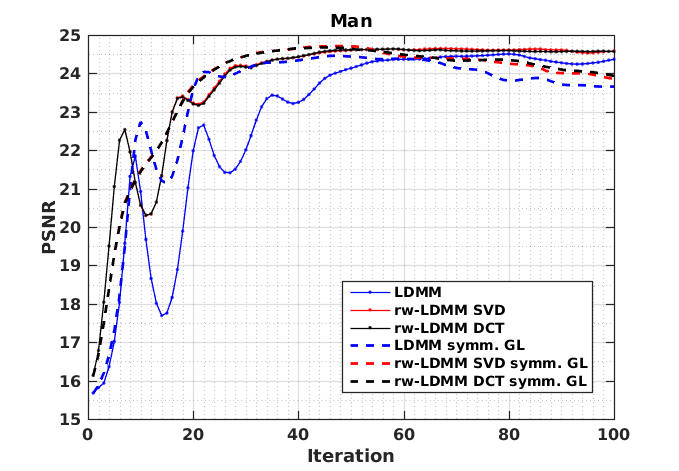}
\includegraphics[width = 0.5\textwidth]{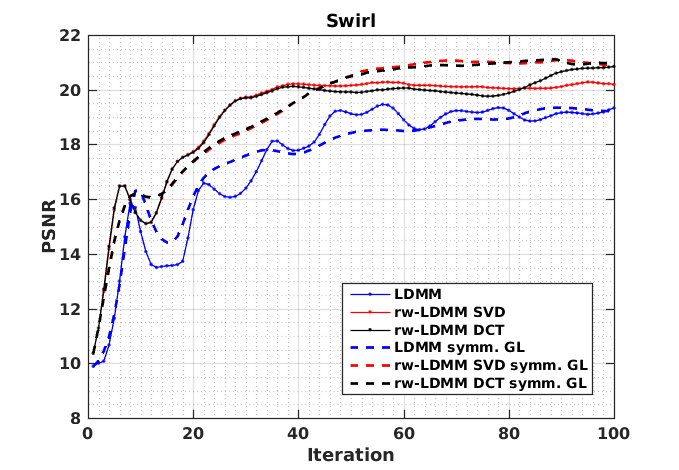}
\caption{Curves depicting the change of PSNR values as the number of iterations increases for $8$ test images. The initial subsample rates are all set to $10\%$. It is worth noting from these figures that rw-LDMM algorithms outperform the original LDMM (both in terms of final and maximum PSNR)  consistently for a wide range of numbers of iterations. Also, in some of these figures the PSNR curves for rw-LDMM SVD (solid blue lines with dots) and rw-LDMM DCT (solid red lines with dots) stay very closely to each other, indicating that the methodology of rw-LDMM is robust to these choices of local basis.}
\label{fig:psnr_curves}
\end{figure}

\end{document}